\documentclass{article}

\PassOptionsToPackage{numbers, compress}{natbib}

\usepackage[final]{neurips_2024}

\usepackage[utf8]{inputenc} %
\usepackage[T1]{fontenc}    %
\usepackage{hyperref}       %
\usepackage{url}            %
\usepackage{booktabs}       %
\usepackage{amsfonts}       %
\usepackage{nicefrac}       %
\usepackage{microtype}      %
\usepackage{xcolor}         %

\usepackage{amsmath}
\usepackage{amssymb}
\usepackage{mathtools}
\usepackage{amsthm}

\usepackage{graphicx}

\usepackage{thmtools}
\usepackage{bbm}
\usepackage{dsfont}
\usepackage{caption}
\usepackage{subcaption}
\usepackage{multirow}
\usepackage{diagbox}
\usepackage{enumitem}
\usepackage{multicol}
\usepackage{booktabs}
\usepackage{dblfloatfix} 

\usepackage{tikz}
\usepackage{wrapfig}

\newcommand{\sknn}{\ensuremath{\text{s}k\text{NN}}}
\newcommand{\asym}{A^{\text{sym}}}
\newcommand{\lsym}{L^{\text{sym}}}

\theoremstyle{plain}
\newtheorem{theorem}{Theorem}[section]
\newtheorem{proposition}[theorem]{Proposition}
\newtheorem{lemma}[theorem]{Lemma}
\newtheorem{corollary}[theorem]{Corollary}
\theoremstyle{definition}
\newtheorem{definition}[theorem]{Definition}

\theoremstyle{remark}

\title{Persistent Homology for High-dimensional Data \\Based on Spectral Methods}

\author{Sebastian Damrich\footnotemark[2] \hspace{1cm}
Philipp Berens\footnotemark[2]\hspace{0.15cm}\footnotemark[3]\hspace{1cm}
Dmitry Kobak\footnotemark[2]\hspace{0.15cm}\footnotemark[4]\\
\footnotemark[2]\hspace{0.2cm}Hertie Institute for AI in Brain Health, University of T\"{u}bingen, Germany\\
\footnotemark[3]\hspace{0.2cm}T\"{u}bingen AI Center, Germany\\
\footnotemark[4]\hspace{0.2cm}IWR, Heidelberg University, Germany\\
\texttt{\{sebastian.damrich,philipp.berens,dmitry.kobak\}@uni-tuebingen.de}
}

\begin{document}

\maketitle

\begin{abstract}
    Persistent homology is a popular computational tool for analyzing the topology of point clouds, such as the presence of loops or voids. However, many real-world datasets with low intrinsic dimensionality reside in an ambient space of much higher dimensionality. We show that in this case traditional persistent homology becomes very sensitive to noise and fails to detect the correct topology. The same holds true for existing refinements of persistent homology. As a remedy, we find that spectral distances on the $k$-nearest-neighbor graph of the data, such as diffusion distance and effective resistance, allow to detect the correct topology even in the presence of high-dimensional noise. Moreover, we derive a novel closed-form formula for effective resistance, and describe its relation to diffusion distances. Finally, we apply these methods to high-dimensional single-cell RNA-sequencing data and show that spectral distances allow robust detection of cell cycle loops.
\end{abstract}

\section{Introduction}

Algebraic topology can describe the shape of a continuous manifold. In particular, it can detect if a manifold has holes, using its so-called homology groups~\citep{hatcher2002algebraic}. For example, a cup has a single one-dimensional hole, or \textit{loop} (its handle), whereas a football has a single two-dimensional hole, or \textit{void} (its hollow interior). These global topological properties are often helpful for understanding an object's overall structure. However, real-world datasets are typically given as point clouds, a discrete set of points sampled from an underlying manifold. In this setting, true homologies are trivial, as there is one connected component per point and no holes whatsoever; instead, \textit{persistent homology} can be used to find holes in point clouds and to assign an importance score called \textit{persistence} to each~\citep{edelsbrunner2002topological, zomorodian2004computing}. Holes with high persistence are indicative of holes in the underlying manifold. Persistence homology has been successfully applied in machine learning pipelines, for instance for gait recognition~\citep{lamar2012human}, instance segmentation~\citep{hu2019topology}, and protein binding~\citep{wang2020topology}, as well as for neural network analysis~\citep{rieck2018neural}.

Persistent homology works well for low-dimensional data~\citep{turkes2022effectiveness} but we find that it has difficulties in high dimensionality. If data points are sampled from a low-dimensional manifold embedded in a high-dimensional ambient space (\textit{manifold hypothesis}), then the measurement noise typically affects all ambient dimensions. In this setting, traditional persistent homology is not robust against even low levels of noise. On a dataset as simple as a circle in $\mathbb R^{50}$, persistent homology based on the Euclidean distance between noisy points can fail to identify the correct loop as a clear outlier in the persistence diagram (Figure~\ref{fig:1}). The aim of our work is to find alternatives to traditional persistent homology that can robustly detect the correct topology despite high-dimensional noise.

We were inspired by visualization methods $t$-SNE~\citep{van2008visualizing} and UMAP~\citep{mcinnes2018umap} that are able to depict the loop in the same noisy dataset (Figure~\ref{fig:1}d,e). They approximate the data manifold by the $k$-nearest-neighbor ($k$NN) graph~\citep{tenenbaum2000global, roweis2000nonlinear,belkin2001laplacian, hinton2002stochastic,moon2019visualizing}. Therefore, we suggest to use persistent homology with spectral distances on this $k$NN graph, such as the effective resistance~\citep{doyle1984random} and the diffusion distance~\citep{coifman2006diffusion}. Effective resistance successfully identified the correct loop in the above toy example (Figure~\ref{fig:1}). We also found spectral distances to outperform other distances in detecting the correct topology of several synthetic datasets as well as finding the cell cycles in single-cell RNA-sequencing data.

Our contributions are:

\begin{enumerate}[nosep, leftmargin=0.5\leftmargin]%
    \item an analysis of the failure modes of persistent homology for noisy high-dimensional data;
    \item a closed-form expression for effective resistance, explaining its relation to diffusion distances;
    \item a synthetic benchmark, with spectral distances outperforming state-of-the-art alternatives;
    \item an application to a range of single-cell RNA-sequencing datasets with ground-truth cycles.
\end{enumerate}

Our code is available at \url{https://github.com/berenslab/eff-ph/tree/neurips2024}.

\begin{figure}[tb]
    \centering
    \includegraphics[width=\textwidth]{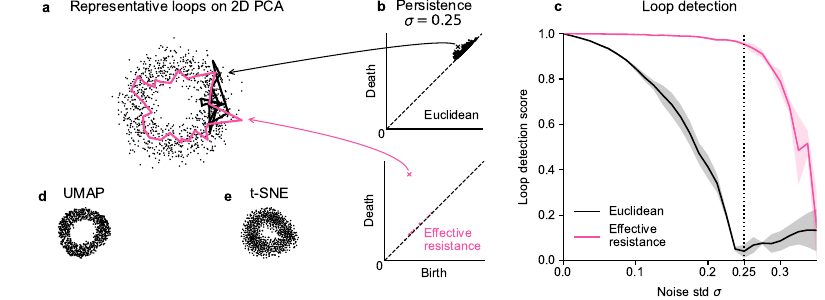}
    \caption{\textbf{a.} 2D PCA of a noisy circle ($\sigma=0.25$, radius 1) in $\mathbb R^{50}$. Overlaid are representative cycles of the most persistent loops. \textbf{b.} Persistence diagrams using Euclidean distance and the effective resistance.
    \textbf{c.} Loop detection scores of persistent homology using effective resistance and Euclidean distance. \textbf{d, e.} UMAP and $t$-SNE embeddings of the same data, showing the loop structure in 2D.}
    \label{fig:1}
\end{figure}

\section{Related work}

Persistent homology has long been known to be sensitive to outliers~\citep{chazal2011geometric} and several extensions have been proposed to make it more robust. One recurring idea is to replace the Euclidean distance with a different distance matrix, before running persistent homology. \citet{bendich2011improving} suggested to use diffusion distances~\citep{coifman2006diffusion}, but their empirical validation was limited to a single dataset in 2D. \citet{anai2020dtm} suggested to use the distance-to-measure (DTM)~\citep{chazal2011geometric} and \citet{fernandez2023intrinsic} proposed to use Fermat distances~\citep{groisman2022nonhomogeneous}. \citet{vishwanath2020robust} introduced persistent homology based on robust kernel density estimation, an approach that itself becomes challenging in high dimensionality. All of these works focused on low-dimensional datasets ($<$10D, mostly 2D or 3D), while our work specifically addresses the challenges of persistent homology in high dimensionality.

The concurrent work of~\citet{hiraoka2024curse} is the most relevant related work. Their treatment of the curse of dimensionality of persistent homology is mostly theoretical, while ours has an empirical focus. The two works are thus complementary to each other. \citeauthor{hiraoka2024curse}'s theoretical description of the curse of dimensionality is similar to ours (Appendix~\ref{app:dist_by_dim}) in that it analyses how distance concentration with high-dimensional noise impairs persistent homology, but is more general. Practically, \citeauthor{hiraoka2024curse} propose normalized PCA to mitigate the curse of dimensionality. However, this approach assumes the true dimensionality of the data to be known, which is not realistic in real-world applications, and performs worse than our suggestions~(Appendix~\ref{app:pca}).

Below, we recommend using effective resistance and diffusion distances for persistent homology in high-dimensional spaces. Both of these distances, as well as the shortest path distance, have been used in combination with persistent homology to analyze the topology of graph data~\citep{petri2013networks, hajij2018visual, aktas2019persistence,tran2019scale, carriere2020perslay, memoli2022persistent, davies2023persistent}. Shortest paths on the $k$NN graph were also used by~\citet{naitzat2020topology} and \citet{fernandez2023intrinsic}. Motivated by the performance of UMAP~\citep{mcinnes2018umap} for dimensionality reduction, ~\citet{gardner2022toroidal} and \citet{hermansen2024uncovering} used UMAP affinities to define distances for persistent homology.

Effective resistance is a well-established graph distance~\citep{doyle1984random, fouss2016algorithms}. A correction, more appropriate for large graphs, was suggested by~\citet{luxburg2010getting, von2014hitting}. When speaking of \textit{effective resistance}, we mean this corrected version, if not otherwise stated. It has not yet been combined with persistent homology. Conceptually similar diffusion distances~\citep{coifman2006diffusion} have been used in single-cell RNA-sequencing data analysis, for dimensionality reduction~\citep{moon2019visualizing}, trajectory inference~\citep{haghverdi2015diffusion}, feature extraction~\citep{chew2022manifold}, and hierarchical clustering, similar to 0D persistent homology~\citep{brugnone2019coarse, kuchroo2023single}. 

Persistent homology has been applied to single-cell RNA-sequencing data, but only the concurrent work of~\citet{flores2023unraveling} applies it directly to the high-dimensional data. \citet{wang2023cannot} used a Witness complex on a PCA of the data. Other works applied persistent homology to a derived graph, e.g., a gene regulator network~\citep{masoomy2021topological} or a Mapper graph~\citep{singh2007topological, rizvi2017single}. In other biological contexts, persistent homology has also been applied to a low-dimensional representation of the data: 3D projection of cytometry data~\citep{mukherjee2022determining}, 6D PCA of hippocampal spiking data~\citep{gardner2022toroidal}, and 3D PHATE embedding of calcium signaling~\citep{moore2023cell}. Several recent applications of persistent homology only computed 0D features (i.e. clusters)~\citep{hajij2018visual, jia2022single, petenkaya2022identifying}, which amounts to doing single linkage clustering~\citep{gower1969minimum}. Here we only investigate the detection of higher-dimensional (1D and 2D) holes with persistent homology. The dimensionality of the data itself, however, is typically much higher.

\begin{wrapfigure}{R}{0.6\textwidth}
    \centering
    \vspace*{-0.75cm}
    \includegraphics[width=\linewidth]{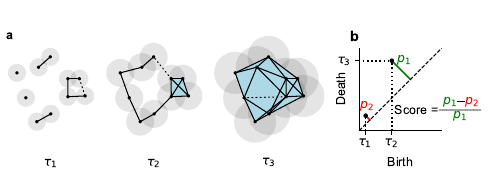}
    \caption{\textbf{a.} Persistent homology applied to a noisy circle ($n=10$) in 2D tracks appearing and disappearing holes as balls grow around each datapoint. 
    Dotted lines show the graph edges that lead to the birth / death of two loops (Section~\ref{sec:PH}). \textbf{b.} The corresponding persistence diagram with two detected 1D holes (loops). Our \textit{hole detection score} measures the gap in persistence between the first and the second detected holes (Section~\ref{para:performance_metric}).}
    \label{fig:ph}
\end{wrapfigure}

\section{Background: persistent homology}
\label{sec:PH}

Persistent homology computes topological invariants of a space at different scales. For point clouds, the different scales are typically given by growing a ball around each point (Figure~\ref{fig:ph}a), and letting the radius $\tau$ grow from $0$ to infinity. For each value of $\tau$, homology groups of the union of all balls are computed to find the holes, and holes that \textit{persist} for longer time periods are considered more prominent. At $\tau\approx 0$, there are no holes as the balls are non-overlapping, while at $\tau\to\infty$ there are no holes as all the balls merge together.

To keep the computation tractable, instead of the union of growing balls, persistent homology operates on a so-called \textit{filtered simplicial complex} (Figure~\ref{fig:ph}a). A simplicial complex is a hypergraph containing points as nodes, edges between nodes, triangles bounded by edges, and so forth. These building blocks are called \textit{simplices}. At each time $\tau$, the complex encodes all intersections between the balls and suffices to find the holes. The complexes at smaller $\tau$ values are nested within the complexes at larger $\tau$ values, and together form a filtered simplicial complex, with $\tau$ being the filtration time. In this work, we only use the Vietoris--Rips complex, which includes an $n$-simplex $(v_0, v_1, \dots, v_n)$ at filtration time $\tau$ if the distances between all pairs $v_i, v_j$ are at most $\tau$. Therefore, to build a Vietoris--Rips complex, it suffices to provide pairwise distances between all pairs of points. We compute persistent homology via the \texttt{ripser} package~\citep{bauer2021ripser} to which we pass a distance matrix.

Persistent homology consists of a set of holes for each dimension. We limit ourselves to loops and voids. Each hole has associated birth and death times $(\tau_b, \tau_d)$, i.e., the first and last filtration value $\tau$ at which that hole exists. Their difference $p = \tau_d - \tau_b$ is called the \textit{persistence} of the hole and quantifies its prominence. The birth and death times can be visualized as a scatter plot (Figure~\ref{fig:ph}b), known as the \textit{persistence diagram}. Points far from the diagonal have high persistence. This process is illustrated in Figure~\ref{fig:ph} for a noisy sample of $n=10$ points from a circle $S^1 \subset \mathbb R^2$.  At $\tau_1$, a small spurious loop is formed thanks to the inclusion of the dotted edge, but it dies soon afterwards.  The ground-truth loop is formed at $\tau_2$ and dies at $\tau_3$, once the hole is completely filled in by triangles. Both loops (one-dimensional holes) found in this dataset are shown in the persistence diagram.

\section{The curse of dimensionality for persistent homology
}
\label{sec:curse}

\begin{figure*}[tb]
    \centering\includegraphics[width=5.5in]{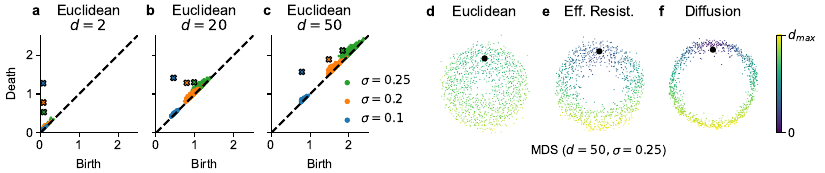}
    \caption{\textbf{a\,--\,c.} Persistence diagrams of a noisy circle in different ambient dimensionality and with different amount of noise. Ideally, there should be one feature (point) with high persistence, corresponding to the circle. But for high noise and dimensionality that feature vanishes into the noise cloud near the diagonal. \textbf{d\,--\,f.} Multidimensional scaling of Euclidean, effective resistance, and  diffusion distances for a noisy circle in $\mathbb R^{50}$. Color indicates the distance to the highlighted point.
    }
    \label{fig:curse}
\end{figure*}

While persistent homology is robust to small changes in point positions~\citep{cohen2005stability}, the curse of dimensionality can still severely hurt its performance. To illustrate, we consider the same toy setting as in Figure~\ref{fig:1}: we sample points from $S^1 \subset \mathbb R^d$, and add Gaussian noise of standard deviation $\sigma$ to each ambient coordinate. When $d=2$, higher noise does not affect the birth times but leads to lower death times (Figure~\ref{fig:curse}a), because some points get distorted to the middle of the circle and the hole fills up at earlier $\tau$. When we increase the ambient dimensionality to $d=20$, higher noise leads to later birth times (Figure~\ref{fig:curse}b) because in higher dimensionality distances get dominated by the noise dimensions rather than by the circular structure. Indeed, in Corollary~\ref{cor:noise_dominates_n_pts} we prove that for any two points $x_i, x_j\in \mathbb{R}^d$ and two isotropic, multivariate normal noise vectors $\varepsilon_1, \varepsilon_2\sim\mathcal{N}(\mathbf{0}, \sigma^2\mathbf{I}_d)$ the ratio $\|\varepsilon_1 - \varepsilon_2\| / \|x_i+ \varepsilon_1 - (x_2 +\varepsilon_2)\| \to 1$ in probability as $d\to \infty$. Finally, for $d=50$ both the birth \textit{and} the death times increase with $\sigma$ (Figure~\ref{fig:curse}c, Corollary~\ref{cor:noise_dominates_ph}) and the ground-truth hole disappears in the cloud of spurious holes. Applying MDS to the Euclidean distances obtained with $d=50$ and $\sigma=0.25$ yields a 2D layout with almost no visible hole, because all distances have become similar (Figure~\ref{fig:curse}d). 
See the concurrent work of~\citet{hiraoka2024curse} for a more detailed treatment.

Therefore, the failure modes of persistent homology differ between low- and  high-dimensional spaces. While in low dimensions, persistent homology is susceptible to outlier points in the middle of the circle, in high dimensions, there are no points in the middle of the circle; instead, all distances become too similar, hiding the true loops. See Appendix~\ref{sec:outliers} for more details on the effect of outliers.

\section{Spectral distances are more robust}
\label{sec:background-spectral}

Many modern manifold learning and dimensionality reduction methods rely on the $k$-nearest-neighbor ($k$NN) graph of the data. This works well because, although distances become increasingly similar in high-dimensional spaces, nearest neighbors still carry information about the data manifold. To make persistent homology overcome high-dimensional noise, we therefore suggest to rely on the symmetric $k$NN graph, which contains edge $ij$ if node $i$ is among the $k$ nearest neighbors of $j$ or vice versa. A natural choice is to use its geodesics, but, as we show below, this does not work well, likely because a single graph edge across a circle can destroy the corresponding feature too early. Instead, we propose to use spectral methods, such as the effective resistance or diffusion distance. Both methods rely on random walks and thus incorporate information from all edges. 

For a connected graph $G$ with $n$ nodes, e.g., the symmetric $k$NN graph, let $A$ be its symmetric, $n\times n$ adjacency matrix with elements $a_{ij}=1$ if edge $ij$ exits in $G$ and $a_{ij}=0$ otherwise. The degree matrix $D$ is defined by $D = \text{diag}\{d_i\}$, where $d_i = \sum_{j=1}^n a_{ij}$ are the node degrees. We define $\text{vol}(G) = \sum_{i=1}^n d_i$. Let $H_{ij}$ be the \textit{hitting time} from node $i$ to $j$, i.e., the average number of edges it takes a random walker, that starts at node $i$ randomly moving along edges, to reach node $j$. The naive effective resistance is defined as $\tilde{d}^\text{eff}_{ij} = (H_{ij} + H_{ji}) / \text{vol}(G)$. This version is known to be unsuitable for large graphs (Figure~\ref{fig:eff_res_comparison}) because it reduces to $\tilde{d}^\text{eff}_{ij} \approx 1/d_i + 1/d_j$ ~\citep{luxburg2010getting}. Therefore, we used~\citet{luxburg2010getting}'s corrected version
\begin{equation}
    d^\text{eff}_{ij} = \tilde{d}^\text{eff}_{ij} - 1/d_i - 1/d_j + 2 a_{ij}/(d_id_j) - a_{ii}/d_i^2 - a_{jj}/d_j^2.
\end{equation}

Diffusion distances also rely on random walks. The random walk transition matrix is given by $P = D^{-1}A$. Then $P^t_{i, :}$, the $i$-th row of $P^t$, holds the probability distribution over nodes after $t$ steps of a random walker starting at node $i$. The diffusion distance is then defined as 
\begin{equation}
    d_{ij}(t) = \sqrt{\text{vol}(G)}\| (P^t_{i, :} - P^t_{j,:})D^{-\frac{1}{2}}\|. \label{eq:diff}
\end{equation}
There are many possible random walks between nodes $i$ and $j$ if they both reside in the same densely connected region of the graph, while it is unlikely for a random walker to cross between sparsely connected regions. As a result, both effective resistance and diffusion distance are small between parts of the graph that are densely connected and are robust against single stray edges (Figure~\ref{fig:eff_res_geo_intuition}). This makes spectral distances on the $k$NN graph the ideal input to persistent homology for detecting the topology of data in high-dimensional spaces. Indeed, the MDS embedding of the effective resistance and of the diffusion distance of the circle in ambient $\mathbb{R}^{50}$ both clearly show the circular structure (Figure~\ref{fig:curse}e,f).

\begin{figure}[tb]
    \centering
    \includegraphics[width=\textwidth]{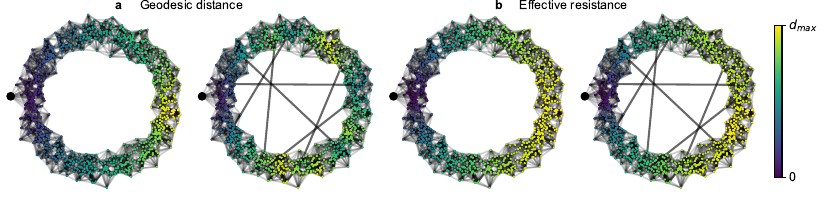}
    \caption{Robustness of effective resistance. We sampled $n=1\,000$ points from a noisy circle in 2D with Gaussian noise of standard deviation $\sigma=0.1$, constructed the unweighted symmetric $15$-NN graph, and optionally added 10 random edges (thick lines). Node colors indicate the graph distance from the fat black dot. \textbf{a.} The geodesic distance is severely affected by the random edges. \textbf{b.}~The effective resistance distance is robust to them.}
    \label{fig:eff_res_geo_intuition}
\end{figure}

\section{Relation between spectral distances}
\label{sec:spectral}

We show in Section~\ref{sec:benchmarks} that spectral methods excel as input distances to persistent homology for high-dimensional data. But first, we explain the relationships between them. 
Laplacian Eigenmaps distance and diffusion distance can be written as Euclidean distances in data representations given by appropriately scaled eigenvectors of the graph Laplacian. In this section, we derive a similar closed-form formula for effective resistance and show that effective resistance aggregates all but the most local diffusion distances.

Let $\asym = D^{-\frac{1}{2}} A D^{-\frac{1}{2}}$ and $\lsym = I - \asym$ be the symmetrically normalized adjacency and Laplacian matrix. We denote the eigenvectors of $\lsym$ by $u_1, \dots, u_n$ and their eigenvalues by $\mu_1, \dots, \mu_n$ in increasing order. 
For a connected graph, $\mu_1 = 0$ and $u_1 = D^{\frac{1}{2}} (1, \ldots, 1)^\top / \sqrt{\text{vol}(G)}$.

The $\tilde{d}$-dimensional Laplacian Eigenmaps embedding is given by the first $\tilde{d}$ nontrivial eigenvectors: 
\begin{align}
    d_{ij}^\text{LE}(\tilde{d}) 
    &= \| e_i^{\text{LE}}(\tilde{d}) - e_j^\text{LE}(\tilde{d})\|,\text{ where } 
    e^{\text{LE}}_i(\tilde{d}) = (u_{2,i},\ldots,u_{(\tilde{d}+1), i}).
\end{align}
The diffusion distance after $t$ diffusion steps is given by~\citep{coifman2006diffusion}
\begin{align}\label{eq:spectral_diff}
    d_{ij}^{\text{diff}}(t) = \sqrt{\text{vol}(G)}\| e_i^{\text{diff}}(t) - e_j^{\text{diff}}(t)\|, \text{ where }
    e_i^{\text{diff}}(t) = \frac{\big((1-\mu_2)^t u_{2,i}, \ldots, (1-\mu_n)^t u_{n,i}\big)}{\sqrt{d_i}}.
\end{align}
The original uncorrected version of effective resistance is given by~\citep{lovasz1993random}
\begin{align}
    \tilde{d}^{\text{eff}}_{ij} = \| \tilde{e}_i^{\text{eff}} - \tilde{e}_j^{\text{eff}}\|^2, \text{ where }
    \tilde{e}_i^{\text{eff}} = \big(u_{2,i} /\sqrt{\mu_2},\ldots, u_{n,i}/\sqrt{\mu_n}\big)/\sqrt{d_i}.\label{eq:eff_res_naive}
\end{align}

\begin{figure}[t]
    \centering
    \includegraphics[width=5.5in]{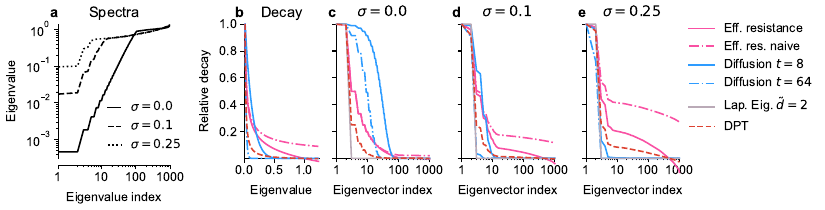}
    \caption{\textbf{a.} Eigenvalue spectra of the $k$NN graph Laplacian for the noisy circle in ambient $\mathbb{R}^{50}$ for noise levels \mbox{$\sigma=\{0.0, 0.1, 0.25\}$.} \textbf{b.} Decay of eigenvector contribution based on the eigenvalue for effective resistance, diffusion distances and DPT. \textbf{c\,--\,e.} Relative contribution of each eigenvector for eff. resistance, diffusion distance, Laplacian Eigenmaps, and DPT for various noise levels (Section~\ref{sec:spectral}).}
    \label{fig:eigenvalues}
\end{figure}

In Appendix~\ref{app:eff_res} we prove that the corrected effective resistance \citep{luxburg2010getting} can also be written in this form:
\begin{proposition} \label{prop:corr_eff_res_main}
The corrected effective resistance distance can be computed by 
\begin{align}
    d_{ij}^\textnormal{eff}= \|e_i^\textnormal{eff} -e_j^\textnormal{eff}\|^2, \text{ where } 
    e_i^\textnormal{eff} = \left(\frac{1-\mu_2}{\sqrt{\mu_2} }u_{2,i}, \ldots, \frac{1-\mu_n}{\sqrt{\mu_n}}u_{n,i}\right) \Big/ \sqrt{d_i}. \label{eq:eff_res_corr}
\end{align}
\end{proposition}
It has been known~\cite{memoli} that the  uncorrected effective resistance can be written in terms of diffusion distances as \mbox{$\tilde{d}^{\textnormal{eff}}_{ij} = \sum_{t=0}^\infty d^{\textnormal{diff}}_{ij}(t/2)^2 / \text{vol}(G)$,} see Proposition~\ref{prop:diff_naive_eff_res}.
Here, based on Proposition~\ref{prop:corr_eff_res_main}, we derive a similar result for the corrected effective resistance (proof in Appendix~\ref{app:diff_eff_res}): 
\begin{corollary}\label{cor:diff_corr_eff_res_short}
    If $G$ is connected and not bipartite, we have
    \begin{align}
        d^{\textnormal{eff}}_{ij} = 
        \sum_{t=2}^\infty d^{\textnormal{diff}}_{ij}(t/2)^2 / \textnormal{vol}(G) \label{eq:diff_corr_eff_res_main} 
        \quad \text{ and hence } \quad       \tilde{d}^{\textnormal{eff}}_{ij} - d^{\textnormal{eff}}_{ij} = \Big(d^{\textnormal{diff}}_{ij}(0)^2 + d^{\textnormal{diff}}_{ij}(1/2)^2\Big)/\textnormal{vol}(G).
    \end{align}
\end{corollary}
In words, the corrected effective resistance combines all diffusion distances, save for those with the shortest diffusion time. These most local diffusion distances form exactly the correction from naive to corrected effective resistance. While the effective resistance is a \textit{squared} Euclidean distance, omitting the square amounts to taking the square root of all birth and death times, maintaining the loop detection performance of effective resistance~(Figure~\ref{fig:eff_res_comparison}). Therefore, the main difference between the spectral methods is in to how they decay eigenvectors based on the corresponding eigenvalues.

The naive effective resistance decays the eigenvectors with $1/\sqrt{\mu_i}$, which is much slower than diffusion distances' $(1-\mu_i)^t$ for \mbox{$t\in[8,64]$}. Corrected effective resistance shows intermediate behavior (Figure~\ref{fig:eigenvalues}b). When represented as a sum over diffusion distances, it contains all diffusion distances with $t\ge 1$, making it decay slower than diffusion distances with $t=8$ or $64$, but does not contain the non-decaying $t=0$ term, so it decays faster than its naive version. The correction matters little for $S^1\subset \mathbb R^{50}$ in the absence of noise, when the first eigenvalues are much smaller than the rest and dominate the embedding (Figure~\ref{fig:eigenvalues}a,c) but becomes important as the noise and consequently the low eigenvalues increase (Figure~\ref{fig:eigenvalues}a,d,e). As the noise increases, the decay for diffusion distances gets closer to a step function preserving only the first two non-constant eigenvectors, sufficient for the circular structure. In contrast, Laplacian Eigenmaps needs the number of components as input (Figure~\ref{fig:eigenvalues}c\,--\,e).\footnote{Diffusion pseudotime (DPT)~\cite{haghverdi2016diffusion} has a very similar expression as corrected effective resistance, using the scaling $(1-\mu_i)/\mu_i$, see Appendix~\ref{app:dpt}. This means that DPT decays eigenvalues faster than both versions of effective resistance (Figure~\ref{fig:eigenvalues}). We prove an analogous statement to Corollary~\ref{cor:diff_corr_eff_res_short} for DPT in Proposition ~\ref{prop:dpt_diff}.}

\section{Spectral distances find holes in high-dimensional spaces}
\label{sec:benchmarks}

High-dimensional data is ubiquitous, but traditional persistent homology can fail to detect its topology. Here, we benchmark the performance of various distances as input to persistent homology.

\begin{figure}[tb]
    \centering
    \includegraphics[width=5.5in]{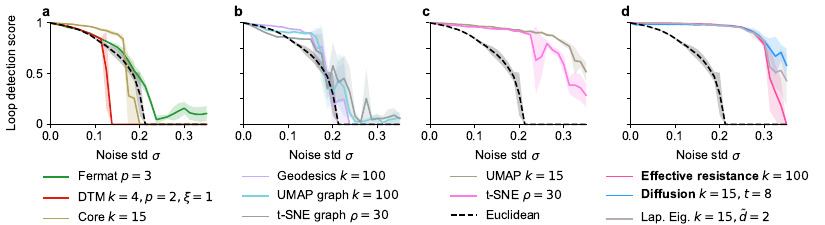}
    \caption{Loop detection score for persistent homology with various distances on a noisy circle in~$\mathbb R^{50}$. The best hyperparameter setting for each distance is shown. Methods are grouped into panels for visual clarity. Recommended methods in \textbf{bold}.}
    \label{fig:circle-benchmark}
\end{figure}

\paragraph{Distance measures}

We examined twelve distances as input to persistent homology, beyond the Euclidean distance. Full definitions are given in Appendix~\ref{app:dist}. First, there are some state-of-the-art approaches for persistent homology in the presence of noise and outliers. Fermat distances~\citep{fernandez2023intrinsic} aim to exaggerate large over small distances to incorporate the density of the data. Distance-to-measure (DTM)~\citep{anai2020dtm} aims for outlier robustness by combining the Euclidean distance with the distances from each point to its $k$ nearest neighbors, which are high for outliers. Similarly, the core distance used in the HDBSCAN algorithm~\citep{campello2015hierarchical, damm2022core} raises each Euclidean distance at least to the distance between incident points and their $k$-th nearest neighbors. We evaluate these methods here with respect to Gaussian noise in high-dimensional ambient space, a different noise model than the one for which these methods were designed. Second, we consider some non-spectral graph distances. The geodesic distance on the $k$NN graph was popularized by Isomap~\citep{tenenbaum2000global} and used for persistent homology by~\citet{naitzat2020topology}. Following~\citet{gardner2022toroidal} we used distances based on UMAP affinities, and also experimented with $t$-SNE affinities. Third, we computed $t$-SNE and UMAP embeddings and used distances in the 2D embedding space. Finally, we explored methods using the spectral decomposition of the $k$NN graph Laplacian, see Section~\ref{sec:spectral}: effective resistance, diffusion distance, and the distance in Laplacian Eigenmaps' embedding space. 

All methods come with hyperparameters. We report the results for the best hyperparameter setting on each dataset (Appendix~\ref{sec:hyperparameters}) but found spectral methods to be robust to these choices (Appendix~\ref{sec:hyperparam_sensitivity}).

\paragraph{Performance score}\label{para:performance_metric}

The output of persistent homology is a persistence diagram showing birth and death times for all detected holes. It may be difficult to decide whether this procedure has actually detected a hole in the data, or not. Ideally, for a dataset with $m$ ground-truth holes, the persistence diagram should have $m$ points with high persistence while all other points should have low persistence and lie close to the diagonal. Therefore, for $m$ ground-truth features, our \textit{hole detection score} $s_m\in[0,1]$ is the relative gap between the persistences $p_m$ and $p_{m+1}$ of the $m$-th and $(m+1)$-th most persistent features: $s_m  = (p_m-p_{m+1})/p_m$. This corresponds to the visual gap between them in the persistence diagram (Figure~\ref{fig:ph}b). \citet{rieck2017agreement} as well as \citet{smith2021skeletonisation} used similar quantities to find important features. We prove a continuity property of $s_m$ in Appendix~\ref{app:score_cont} and consider alternative scores in Appendix~\ref{app:alt_scores}.

In addition, we set $s_m = 0$ if all features in the persistence diagram have very low death-to-birth ratios $\tau_d / \tau_b < 1.25$. This handles situations with very few detected holes that die very quickly after being born, which otherwise can have spuriously high $s_m$ values. This was done everywhere apart from the qualitative Figures~\ref{fig:1}, \ref{fig:malaria} and in Figure~\ref{fig:circle_many_methods_d_50_no_filter}. We call this heuristic \textit{thresholding}.

Note that the number of ground-truth topological features was used only for evaluation. We report the mean over three random seeds; shading and error bars indicate the standard deviation.

\subsection{Synthetic benchmark}
\label{sec:benchmark-synthetic}

\begin{figure}[tb]
    \centering
    \includegraphics[width=5.5in]{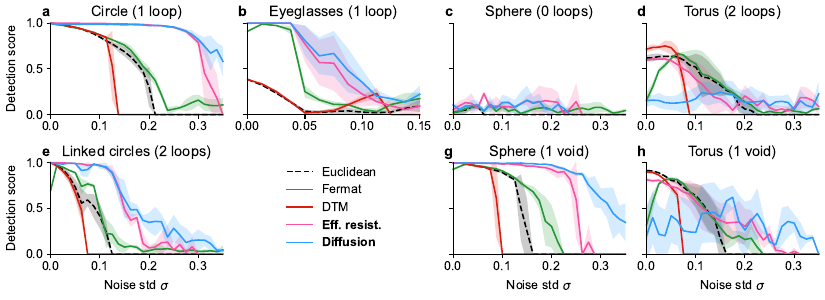}
    \caption{Loop detection score for selected methods on synthetic datasets in ambient $\mathbb{R}^{50}$. More experimental results can be found in Figures~\ref{fig:circle_many_methods_d_50}\,--\,\ref{fig:torus_many_methods_d_50_2D}. Recommended methods in \textbf{bold}.
    }
    \label{fig:datasets-benchmark}
\end{figure}

\paragraph{Benchmark setup}
\label{subsec:benchmark_setup}

In our synthetic benchmark, we evaluated the performance of various distance measures in conjunction with persistent homology on five manifolds: a circle, a pair of linked circles, the eyeglasses dataset (a circle squeezed nearly to a figure eight)~\citep{fernandez2023intrinsic}, the sphere, and the torus. The radii of the circles, the sphere, and the torus' tube were set to $1$, the bottleneck of the eyeglasses was $0.7$, and the torus' tube followed a circle of radius $2$. In each case, we uniformly sampled $n=1\,000$ points from the manifold, mapped them isometrically to $\mathbb{R}^d$ for $d\in[2, 50]$, and then added isotropic Gaussian noise sampled from $\mathcal{N}(\mathbf 0, \sigma^2 \mathbf I_d)$ for $\sigma \in [0, 0.35]$. More details can be found in Appendix~\ref{app:datasets}. For each resulting dataset, we computed persistent homology for loops and, for the sphere and the torus, also for voids. We never computed holes of dimension $3$ or higher.

\paragraph{Results on synthetic data}

On the circle dataset in $\mathbb R^{50}$, persistent homology with all distance metrics found the correct hole when the noise level $\sigma$ was very low (Figure~\ref{fig:circle-benchmark}). However, as the amount of noise increased, the performance of Euclidean distance quickly deteriorated, reaching zero score at $\sigma\approx0.2$. Most other distances outperformed the Euclidean distance, at least in the low noise regime. Fermat distance did not have any effect, and neither did DTM distance, which collapsed at $\sigma\approx0.15$ due to our thresholding (Figure~\ref{fig:circle-benchmark}a). Geodesics, UMAP/$t$-SNE graph, and core distance offered only a modest improvement over Euclidean (Figure~\ref{fig:circle-benchmark}b) highlighting that many $k$NN-graph-based distances cannot handle high-dimensional noise. In contrast, embedding-based distances performed very well on the circle (Figure~\ref{fig:circle-benchmark}c), but have obvious limitations: for example, a 2D embedding cannot possibly have a void. UMAP with higher embedding dimension struggled with loop detection on surfaces and the torus' void (Appendix~\ref{sec:umap_high_embd_dim}). Finally, all spectral methods (effective resistance, diffusion, and Laplacian Eigenmaps) showed similarly excellent performance (Figure~\ref{fig:circle-benchmark}d).

In line with these results, spectral methods outperformed other methods across most synthetic datasets in $\mathbb R^{50}$ (Figure~\ref{fig:datasets-benchmark}).
DTM collapsed earlier than Euclidean but detected loops on the torus for low noise levels best by a small margin. Fermat distance typically had little effect and provided a benefit over Euclidean only on the eyeglasses and the sphere. Spectral distances outperformed all other methods on all datasets apart from the torus, where effective resistance was on par with Euclidean but diffusion performed poorly. On a more densely sampled torus, all methods performed better and the spectral methods again outperformed the others (Figure~\ref{fig:torus_vary_n}).
On all other datasets diffusion distance slightly outperformed  effective resistance for large $\sigma$. Reassuringly, all methods passed the negative control and did not find any persistent loops on the sphere (Figure~\ref{fig:datasets-benchmark}c).

As discussed in Section~\ref{sec:curse}, persistent homology with Euclidean distances deteriorates with increasing ambient dimensionality. Using the circle data in $\mathbb R^d$, we found that if the noise level was fixed at $\sigma=0.25$, no persistent loop was found using Euclidean distances for $d\gtrsim 30$ (Figure~\ref{fig:circle-dims}). In the same setting, DTM deteriorated even more quickly than Euclidean distances. In contrast, effective resistance and diffusion distance were robust against both the high noise level and the large ambient dimension (Figure~\ref{fig:circle-dims}a,c\,--\,e). See Figure~\ref{fig:circle-high-dims} for an extended analysis.

\begin{figure*}[b]
    \includegraphics[width=\linewidth]{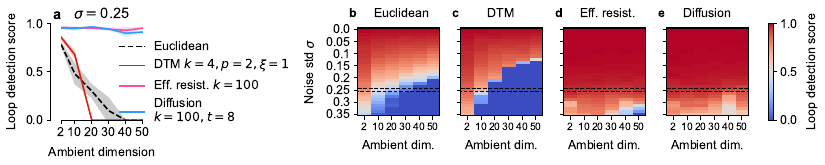}
    \caption{\textbf{a.} Loop detection score of various methods on a noisy circle depending on the ambient dimensionality. Noise $\sigma=0.25$. 
    \textbf{b\,--\,e.} Heat maps for $\sigma\in[0, 0.35]$ and $d\in[2,50]$.}
    \label{fig:circle-dims} 
\end{figure*}

\subsection{Detecting cycles in single-cell data}
\label{sec:benchmark-scRNAseq}

We applied our methods to six single-cell RNA-sequencing datasets: Malaria~\citep{howick2019malaria}, Neurosphere and Hippocampus from~\citep{zheng2022universal}, HeLa2~\citep{schwabe2020transcriptome}, Neural IPCs~\citep{braun2023comprehensive}, and Pancreas~\citep{bastidas2019comprehensive}. Single-cell RNA-sequencing data consists of expression levels for thousands of genes in individual cells, so the data is high-dimensional and notoriously noisy. Importantly, all selected datasets are known to contain circular structures, usually corresponding to the cell division cycle during which gene expression levels cyclically change. As a result, we know how many loops to expect in each dataset and can therefore use them as a real-world benchmark of various distances for persistent homology.
In each case, we followed preprocessing pipelines from prior publications 
leading to representations with $10$ to $5\,156$ dimensions. We downsampled datasets with more than $4\,000$ cells  to $n=1\,000$ (Appendix~\ref{app:datasets}). 

\begin{wrapfigure}{R}{0.6\linewidth}
    \centering
    \includegraphics[width=\linewidth]{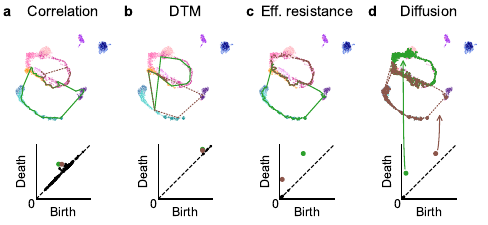}
    \caption{Malaria dataset. \textbf{a\,--\,d.} Representatives of the two most persistent loops overlaid on UMAP embedding (top) and persistence diagrams (bottom) using four methods. Biology dictates that there should be two loops (in warm colors and in cold colors) connected as in a figure eight.
    }
    \label{fig:malaria}
\end{wrapfigure}

The Malaria dataset is expected to contain two cycles: the parasite replication cycle in red blood cells, and the parasite transmission cycle between human and mosquito hosts. Following~\citet{howick2019malaria}, we based all computations for this dataset (and all derived distances) on the correlation distance instead of the Euclidean distance. Persistent homology based on the correlation distance itself failed to correctly identify the two ground-truth cycles and DTM produced representatives that only roughly approximate the two ground truth cycles (Figure~\ref{fig:malaria}a,b). Both effective resistance and diffusion distance successfully uncovered both cycles with $s_2 > 0.9$ (Figure~\ref{fig:malaria}c,d).

Across all six datasets, the detection scores were higher for spectral methods than for their competitors (Figure~\ref{fig:scrnaseq}). Furthermore, we manually investigated representative loops for all considered methods on all datasets and found several cases where the most persistent loop(s) was/were likely not correct (hatched bars in Figure~\ref{fig:scrnaseq}). Overall, we found that the spectral methods, and in particular effective resistance, could reliably find the correct loops with high detection score. Persistent homology based on the $t$-SNE and UMAP embeddings worked on average better than traditional persistent homology, Fermat distances, and DTM, but worse than the spectral methods.

\begin{figure}[b]
    \includegraphics[width=\linewidth]{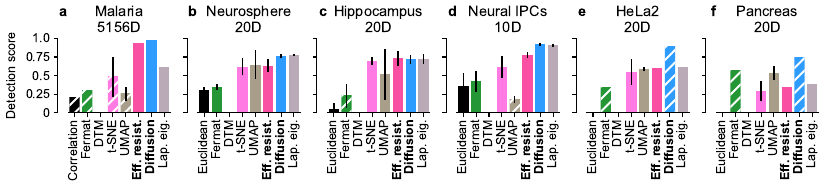}
    \caption{Loop detection scores on six high-dimensional scRNA-seq datasets. Hatched bars indicate implausible representatives. See Figure~\ref{fig:scrnaseq_all} for detection scores for different hyperparameter values. Recommended methods in \textbf{bold}.
    }
    \label{fig:scrnaseq}
\end{figure}

\section{Limitations and future work}\label{sec:limitations}

In the real-world applications, it was important to look at representatives of detected holes as some holes were persistent, but arguably incorrect. That said, each homology class has many different representative cycles, making interpretation difficult. Given ground-truth cycles, an automatic procedure for evaluating cycle correctness remains an interesting research question.

Persistent homology can only detect topology, which is often a useful global level of abstraction. However, it may therefore fail to distinguish some non-isomorphic point clouds~\citep{smith2024generic}. There exist dedicated measures for detecting isometry~\citep{boutin2004reconstructing, widdowson2023recognizing, kurlin2024polynomial}.

Dimensionality reduction methods are designed to handle high-dimensional data. \mbox{$t$-SNE} and UMAP indeed performed well on many datasets, but on average worse than spectral distances on the real data. Moreover, UMAP struggled with the surface of the 3D toy datasets and the torus' void (Appendix~\ref{sec:umap_high_embd_dim}). Finally, they require the choice of an embedding dimension and are known to produce artifacts~\citep{landweber2016fiber,chari2023specious, wang2023cannot}, e.g., leading to poor scores in the noiseless setting in Figures~\ref{fig:highD_umap},~\ref{fig:circle_tsne_embd}. In contrast, spectral distances on the symmetric $k$NN graph worked well without a low-dimensional embedding (Section~\ref{sec:benchmarks}).

Using effective resistance or diffusion distances is easy in practice as their computation time $O(n^3$) is dwarfed by that of the persistent homology  (Table~\ref{tab:run_times}), which scales as $\mathcal{O}(n^{3(\delta+1)})$ for $n$ points and topological holes of dimension $\delta$~\citep{myers2023persistent}. This high complexity of persistent homology aggravates other problems of high-dimensional datasets as dense sampling in high-dimensional space would require a prohibitively large sample size (recall that spectral methods needed a high sampling density for good performance on some of our datasets such as the torus). Combining persistent homology with non-Euclidean distance measures could mitigate this problem via the approach of~\citet{bendich2011improving}, who performed subsampling after computation of the distance matrix. This is a particularly attractive avenue for future research. 

Both effective resistance and diffusion distances require the choice of hyperparameters. However, effective resistance only needs a single hyperparameter: the number of $k$NN neighbors. For this reason and due to its greater outlier resistance (Appendix~\ref{sec:outliers}), we tend to recommend effective resistance over diffusion distances, but a principled criterion when to use which of the two is still missing.

Moreover, we do not have a theoretical proof that spectral distances mitigate the curse of dimensionality. Such a proof may be achieved in the future taking inspiration from the stability results in~\citep{carriere2020perslay, hiraoka2024curse, tran2019scale} and, more generally, spectral perturbation theory.

Our empirical results focus on benchmarking which distances identify the correct topology in the presence of high-dimensional noise. Therefore, we only considered datasets with known ground-truth topology. The next step will be to use spectral distances to detect non-trivial topology in real-world exploratory contexts.

High-dimensional data and thus application areas for our improved topology detection pipeline are becoming ubiquitous. Within biology, we see possible applications for our method in other single-cell omics modalities, population genomics, or neural activity data \citep{gardner2022toroidal, hermansen2024uncovering}. Beyond biology, we believe that our approach can improve the topological analysis of artificial neural network activations~\citep{naitzat2020topology}, and in general be used to detect topology of any high-dimensional data, e.g. in the climate sciences, in astronomical measurements, or wearable sensor data.

\section{Conclusion}

In this work we asked how to use persistent homology on high-dimensional noisy datasets which are very common in real-world applications even if the intrinsic data dimensionality is low. We found spectral methods to be  the optimal approach.
We demonstrated that, as the dimensionality of the data increases, the main problem for persistent homology shifts from handling outliers to handling noise dimensions (Section~\ref{sec:curse}). We used a synthetic benchmark to show that traditional persistent homology and many of its existing extensions struggle to find the correct topology in this setting. Our main finding is that spectral methods based on the $k$NN graph, such as the effective resistance and diffusion distances, still work well (Section~\ref{sec:benchmark-synthetic}). Furthermore, we view it as an advantage that we found existing methods that are able to handle the important problem of high-dimensional noise. We derived an expression for effective resistance based on the eigendecomposition of the graph Laplacian, and demonstrated that it combines all but the most local diffusion distances (Section~\ref{sec:spectral}).  Finally, we showed that spectral distances outperform all competitors on single-cell data (Section~\ref{sec:benchmark-scRNAseq}).

\clearpage

\begin{ack}
We thank Enrique Fita Sanmartin and Ulrike von Luxburg for productive discussions on effective resistance and persistent homology, and Benjamin Dunn, Erik Hermansen, and David Klindt for helpful discussions on combining persistent homology with other dissimilarities than Euclidean distance. Moreover, we thank Sten Linnarsson and Miri Danan Gotthold for sharing the scVI representation of their pallium data. Final thanks go to Bastian Rieck for feedback on the writing.

This work was funded by the Deutsche Forschungsgemeinschaft (DFG, German Research Foundation) via Germany’s Excellence Strategy (Excellence cluster 2064 ``Machine Learning --- New Perspectives for Science'', EXC 390727645; Excellence cluster 2181 ``STRUCTURES'', EXC 390900948), the German Ministry of Science and Education (BMBF) via the T\"{u}bingen AI Center (01IS18039A), the Gemeinn\"{u}tzige Hertie-Stiftung, and the National Institutes of Health (UM1MH130981). The content is solely the responsibility of the authors and does not necessarily represent the official views of the National Institutes of Health. 
\end{ack}

\bibliography{references}
\bibliographystyle{abbrvnat}

\clearpage
\appendix

\renewcommand{\thefigure}{S\arabic{figure}}
\setcounter{figure}{0}

\renewcommand{\thetable}{S\arabic{table}}
\setcounter{table}{0}

\section{Broader impact}\label{app:impact}

The goal of our paper is to advance the field of machine learning. Our comprehensive benchmark required a lot of compute. However, we expect that the lessons learned will save compute for researchers applying persistent homology to their high-dimensional data. Beyond this, we do not see any potential societal consequences of our work that would be worth specifically highlighting here.

\section{Noise in high-dimensional spaces eventually dominates structure}
\label{app:dist_by_dim}

In this section, we rigorously prove the intuitive result that the Euclidean distances~---~and thus the traditional persistence diagrams~---~eventually get dominated by noise as the number of ambient dimensions grows (assuming homoscedastic noise). The main results are Corollary~\ref{cor:noise_dominates_n_pts} and Corollary~\ref{cor:noise_dominates_ph}. Our analysis is similar to the concurrent work of~\citet{hiraoka2024curse}. Their analysis extends ours in that they consider an arbitrary noise distribution, the \v{C}ech complex, and make more precise statements on how persistent homology gets impaired by high-dimensional noise.

We begin with some helpful lemmata.

\begin{lemma}\label{lem:multivariate_normal}
    Let $\{\varepsilon_{1, d}\}$ and $\{\varepsilon_{2, d}\}$ be two sequences of $d$-dimensional multivariate normally distributed random variables $\varepsilon_{i, d} \sim \mathcal{N}(\mathbf{0}, \sigma^2\mathbf{I}_d)$. For each fixed $d$, we have $\mathbb{E}(\| \varepsilon_{1, d}- \varepsilon_{2, d}\|^2) = 2 \sigma^2 d$. Moreover, the sequence $\{\|\varepsilon_{1, d} - \varepsilon_{2,d}\| /\sqrt{d}\}$ converges to $\sqrt{2}\sigma$ almost surely.
\end{lemma}
\begin{proof}
    All entries of $\varepsilon_{1, d}- \varepsilon_{2, d}$ are i.i.d. Gaussian random variables with zero mean and variance $2\sigma^2$. The squared Euclidean distance is the sum of the squared entries of the vector, which implies the first statement. By the strong law of large numbers, $\|\varepsilon_{1, d} - \varepsilon_{2,d}\|^2 / d$ converges almost surely to $2\sigma^2$. By the continuous mapping theorem, this implies almost sure convergence of $\|\varepsilon_{1, d} - \varepsilon_{2,d}\| /\sqrt{d}$ to $\sqrt{2}\sigma$.
\end{proof}

\begin{lemma}\label{lem:exp_var_sq_mean}
    Let $\{X_d\}$ be a sequence of random variables with finite means and variances. If $\mathbb{E}(X_d)$ converges to a constant $c$ and $\textnormal{Var}(X_d) \to 0$ as $d\to \infty$, then $\{X_d\}$ converges to $c$ in squared mean, i.e., $\mathbb{E}\big((X_d- c)^2\big) \to 0$.
\end{lemma}
\begin{proof}  
    We have 
    \begin{align}
        \mathbb{E}\big((X_d - c)^2\big) &= \mathbb{E}(X_d^2) - 2c\mathbb{E}(X_d) + c^2  \notag\\
        &= \mathbb{E}(X_d^2) - \mathbb{E}(X_d)^2 + \mathbb{E}(X_d)^2 - 2c\mathbb{E}(X_d) + c^2 \notag & \\
        &= \text{Var}(X_d) + \mathbb{E}(X_d)^2 - 2c\mathbb{E}(X_d) + c^2 \notag \\
        &\to 0 + c^2 - 2c^2 + c^2  = 0.
    \end{align}
\end{proof}

\begin{lemma}\label{lem:conv_sq_mean_prob}
Let $\{X_d\}$ be a sequence of random variables which converges in squared mean to a random variable $X$. Then it also converges to $X$ in probability.
\end{lemma}
\begin{proof}
    This is an application of Markov's inequality. Let $\varepsilon >0$. Then 
    $$\mathbb{P}\big(|X_d - X| > \varepsilon \big) = \mathbb{P}\big((X_d-X)^2 > \varepsilon^2\big) \leq \mathbb{E}\big((X_d-X)^2\big) / \varepsilon^2 \to 0$$
    for $d\to \infty$ by convergence in squared mean.
\end{proof}

Now we show that for sufficiently many noise dimensions the distance between any two points gets dominated by the noise. This follows from the Pythagorean theorem, because most noise dimensions are orthogonal to the difference of the noise-free points. A similar result can be found in~\citet{hall2005geometric}.

\begin{proposition}\label{prop:pyth}
    Let $x_1$ and $x_2$ be two points in $\mathbb{R}^{d'}$. Let $\{\iota_d\}$ be a sequence of isometries \mbox{$\iota_d: \mathbb{R}^{d'} \to \mathbb{R}^{d}$} for $d\geq d'$. Let $\{\varepsilon_{1,d}\}$ and $\{\varepsilon_{2,d}\}$ be two sequences of $d$-dimensional multivariate normally distributed random variables $\varepsilon_{i, d} \sim \mathcal{N}(\mathbf{0}, \sigma^2\mathbf{I}_d)$ starting at $d=d'$. Let $\{\Delta_d\}$ be the sequence of Euclidean distances between $\iota_d(x_1)+\varepsilon_{1,d}$ and $\iota_d(x_2) + \varepsilon_{2,d}$. Then $\{\Delta_d^2 / d\}$ converges to $2\sigma^2$ in squared mean.
\end{proposition}
\begin{proof}
    Denote the distance of the embedded points by $\delta = \|x_1 -x_2\| = \|\iota_d(x_1)-\iota_d(x_2)\|$. We have 
    $$\Delta_d = \| \iota_d(x_1) + \varepsilon_{1,d} - \big(\iota_d(x_2) + \varepsilon_{2,d}\big)\| = \| \big(\iota_d(x_1)+ \varepsilon_{1,d}-  \varepsilon_{2,d}\big) - \iota_d(x_2)\|.$$
    So instead of the setting where both points get noised, we can just consider the case where only the first point gets noised by a $d$-dimensional multivariate normally distributed variable of double variance $\varepsilon_d \sim \mathcal{N}(\mathbf{0}, 2\sigma^2 \mathbf{I}_d)$. We denote $\|\varepsilon_d\|$ by $l$. We will drop the use of $\iota_d$ now and just write $x_1, x_2 \in \mathbb{R}^{d}$ in slight abuse of notation.
    
    Let $\alpha$ be the angle that $\varepsilon_d$ makes with the vector $x_2 -x_1$. Then, by the law of cosines, 
    \begin{align}
        \Delta_d^2 
        = l^2 + \delta^2 - 2\delta \cos(\alpha) l.
    \end{align}
    Since $\varepsilon_d$ is isometrically distributed, we can assume that $x_2 -x_1$ is parallel to $(1, 0, \dots, 0)$ without loss of generality. Then $\cos(\alpha)l = \varepsilon'$, where $\varepsilon'  \sim \mathcal{N}(0, 2\sigma^2)$ is the first entry of $\varepsilon_d$. Thus, $\mathbb{E}(\Delta_d^2) = 2\sigma^2 d + \delta^2$ by Lemma~\ref{lem:multivariate_normal} and $\mathbb{E}(\Delta_d^2 / d) \to 2\sigma^2$ for $d\to \infty$. 
    
    For the variance, we obtain
    $$\text{Var}(\Delta_d^2) = \text{Var}(l^2+\delta^2-2\delta \varepsilon') = (d-1)\text{Var}(\varepsilon'^2) + \text{Var}(\varepsilon'^2 - 2\delta \varepsilon').$$
    The first term contains the variances of all the directions orthogonal to $x_2-x_1$. Since $\varepsilon_d$ is isotropic, all directions are independent and we can simply add up the individual variances. The second term contains the variance in direction $x_2-x_1$. We have 
    \begin{align}
        \text{Var}(\varepsilon'^2)  &= 3\cdot 4\sigma^4 - 4\sigma^4 \notag\\
        &= 8\sigma^4\\
        \text{Var}(\varepsilon'^2 - 2\delta \varepsilon') &= \mathbb{E}(\varepsilon'^4 - 4\delta \varepsilon'^3 + 4\delta^2 \varepsilon'^2) - \mathbb{E}(\varepsilon'^2 - 2\delta \varepsilon')^2 \notag\\
        &= 12\sigma^4 +0+ 8\delta^2\sigma^2 - 4\sigma^2 +0\notag \\
        &= 8 \sigma^4+ 8\delta^2\sigma^2.
    \end{align}    
    Together, we have 
    \begin{align}
        \text{Var}(\Delta_d^2 / d) &= \big((d-1) 8\sigma^4 + 8\sigma^4 + 8\delta^2\sigma^2\big) / d^2 \notag \\
        &=8\sigma^2(d \sigma^2 +\delta^2) / d^2\notag \\
        &\to 0
    \end{align}
    as $d\to \infty$.
    By Lemma~\ref{lem:exp_var_sq_mean}, we conclude that $\{\Delta_d^2 / d\}$ converges to $2\sigma^2$ in squared mean.
\end{proof}

The next corollary generalizes Proposition~\ref{prop:pyth} to an arbitrary arrangement of $n$ points. Independent of their structure, noise will dominate all distances for sufficiently high dimensionality.
\begin{corollary}\label{cor:noise_dominates_n_pts}
    Let $x_1, \dots, x_n$ be $n$ pairwise distinct points in $\mathbb{R}^{d'}$. Let $\{\iota_d\}$ be a sequence of isometries $\iota_d: \mathbb{R}^{d'} \to \mathbb{R}^{d}$ for $d\geq d'$. Let further $\{\varepsilon_{1,d}\}, \dots, \{\varepsilon_{n, d}\}$ be $n$ sequences of multivariate normally distributed random variables $\varepsilon_{i, d} \sim \mathcal{N}(\mathbf{0}, \sigma^2\mathbf{I}_d)$ in $d$ dimensions. Let $\Delta_{i,j, d}$ be the sequence of random variables of Euclidean distances between $\iota_d(x_i) + \varepsilon_{i, d}$ and $\iota_d(x_j) + \varepsilon_{j, d}$. Then each sequence $\{\Delta_{i,j, d}/\sqrt{d}\}$ converges to $\sqrt{2}\sigma$ in probability as $p\to \infty$. The sequence $\{\|\varepsilon_{i, d} - \varepsilon_{j, d}\| / \Delta_{i,j, d}\}$ converges to $1$ in probability and thus the joint vector $(\|\varepsilon_{i, d} - \varepsilon_{j, d}\| / \Delta_{i,j, d})_{i,j}$ with entries for each pair $i\neq j$ converges to the vector of all ones in probability.
\end{corollary}
\begin{proof}
    Proposition~\ref{prop:pyth}, Lemma~\ref{lem:conv_sq_mean_prob}, and the continuous mapping theorem imply that $\Delta_{i,j, d} / \sqrt{d}$ converges to $\sqrt{2}\sigma$ in probability as $p\to \infty$ and that $\sqrt{d} / \Delta_{i,j, d}$, which is well-defined up to a set of measure zero,  converges to $(\sqrt{2}\sigma)^{-1}$ in probability. By Lemma~\ref{lem:multivariate_normal} and since almost sure convergence implies convergence in probability, we have that $\|\varepsilon_{i, d} - \varepsilon_{j, d}\| /\sqrt{d}$ converges to $\sqrt{2}\sigma$ in probability as well. Since convergence in probability is preserved under multiplication, we obtain that $\|\varepsilon_{i, d} - \varepsilon_{j, d}\| /  \Delta_{i,j, d}$ converges to $1$ in probability. Finally, convergence in probability of two sequences implies joint convergence in probability.
\end{proof}

\begin{corollary}\label{cor:min_max_conv}
    Consider the setting of Corollary~\ref{cor:noise_dominates_n_pts} and choose some $\varepsilon >0$. Then we have that 
    \begin{align}
        \mathbb{P}(\min_{i,j} \Delta_{i,j,d}/\sqrt{d} &> \sqrt{2}\sigma -\varepsilon), \label{term:min_low} \\
        \mathbb{P}(\min_{i,j} \Delta_{i,j,d}/\sqrt{d} &< \sqrt{2}\sigma +\varepsilon), \label{term:min_high} \\
        \mathbb{P}(\max_{i,j} \Delta_{i,j,d}/\sqrt{d} &> \sqrt{2}\sigma -\varepsilon), \textnormal{ and }\label{term:max_low} \\
        \mathbb{P}(\max_{i,j} \Delta_{i,j,d}/\sqrt{d} &< \sqrt{2}\sigma +\varepsilon) \label{term:max_high}
    \end{align}
    all go to one as $d\to \infty$.
\end{corollary}
\begin{proof}
    Since $\max$ and $\min$ are continuous maps, the continuous mapping theorem and Corollary~\ref{cor:noise_dominates_n_pts} imply convergence of the sequences $\{\max_{i,j} \Delta_{i,j,d}/\sqrt{d}\}$ and $\{\min_{i,j} \Delta_{i,j,d}/\sqrt{d}\}$ to $\sqrt{2}\sigma$ in probability.
    The claim follows since the terms~\eqref{term:min_low}, \eqref{term:min_high} and \eqref{term:max_low}, \eqref{term:max_high} are lower bounded by 
    \begin{equation}
        \mathbb{P}(|\min_{i,j}\Delta_{i,j,d}/\sqrt{d} - \sqrt{2}\sigma| > \varepsilon) \text{ and }\mathbb{P}(|\max_{i,j} \Delta_{i,j,d}/\sqrt{d} - \sqrt{2}\sigma| > \varepsilon),
    \end{equation}
    respectively.
\end{proof}

The next corollary shows that for sufficiently high ambient dimension $d$, increasing the noise level will increase birth and death times and thus drive persistent homology with high probability.

\begin{corollary}\label{cor:noise_dominates_ph}
    Consider the setting of Corollary~\ref{cor:noise_dominates_n_pts} but for different noise levels $\sigma > \sigma' > 0$. Let $\beta_{d,\sigma}$ and $\delta_{d,\sigma}$ be the minimal birth and death times among all at least one-dimensional homological features at that noise level and ambient dimensionality. Similarly, let $B_{d,\sigma}$ and $D_{d, \sigma}$ be the maximal birth and death times among all at least one-dimensional homological features. Then $\mathbb{P}(\beta_{d,\sigma} > D_{d, \sigma'})$ and hence also $\mathbb{P}(\beta_{d,\sigma} > B_{d, \sigma'}), \mathbb{P}(\delta_{d,\sigma} > D_{d, \sigma'})$ converge to one as $d\to \infty$.  
\end{corollary}
\begin{proof}
    Denote the distance between the noised $d$-dimensional points $i, j$ with noise level $\sigma$ by $\Delta_{i,j,d,\sigma}$.
    The main idea is that \mbox{$\beta_{d,\sigma}\geq \min_{i,j} \Delta_{i,j,d, \sigma}$} and $D_{d,\sigma'} \leq \max_{i,j} \Delta_{i,j,d, \sigma}$. We prove $\mathbb{P}(\beta_{d,\sigma} > D_{d, \sigma'}) \to 1$. The other statements follow from $\beta_{d, \sigma} < B_{d, \sigma}$ and $\delta_{d, \sigma} > \beta_{d, \sigma}$.

    Choose any  $0<\varepsilon < (\sigma -\sigma')/\sqrt{2}$. Then $\sqrt{2}\sigma -\varepsilon > \sqrt{2}\sigma' + \varepsilon$. Furthermore, we have  
    $$\mathbb{P}(b_{d,\sigma} > D_{d, \sigma'}) \geq \mathbb{P}\Big(b_{d,\sigma} > (\sqrt{2}\sigma -\varepsilon)\sqrt{d} \text{ and } (\sqrt{2}\sigma' + \varepsilon)\sqrt{d} > D_{d, \sigma'}\Big).$$ 
    By Corollary~\ref{cor:min_max_conv}, we have for $d\to \infty$ that
    $$\mathbb{P}\big(b_{d,\sigma} > (\sqrt{2}\sigma -\varepsilon)\sqrt{d}\big) \geq \mathbb{P}\big(\min_{i,j} \Delta_{i,j,d, \sigma} > (\sqrt{2}\sigma -\varepsilon)\sqrt{d}\big) \to 1.$$
    Similarly,
    $$\mathbb{P}\big(D_{d,\sigma} < (\sqrt{2}\sigma' +\varepsilon)\sqrt{d}\big) \geq \mathbb{P}\big(\max_{i,j} \Delta_{i,j,d, \sigma} < (\sqrt{2}\sigma' +\varepsilon)\sqrt{d}\big) \to 1.$$
    As a result,  
    $$1\geq\mathbb{P}\big(b_{d,\sigma} > (\sqrt{2}\sigma -\varepsilon)\sqrt{d} \text{ and } (\sqrt{2}\sigma' + \varepsilon)\sqrt{d} > D_{d, \sigma'}\big) \to 1,$$
    and hence $ \mathbb{P}(b_{d,\sigma} > D_{d, \sigma'}) \to 1$ for $d\to \infty$.
\end{proof}

These statements show that noise will eventually dominate the Euclidean distances and drive traditional persistent homology if the ambient dimension is large enough. In Figure~\ref{fig:circle-dims} we saw empirically that tens of ambient dimensions already severely impact traditional persistent homology, while spectral distance offer more noise robustness. This also holds true in hundreds to thousands of ambient dimensions. Compared to the Euclidean distance, spectral distance can detect the correct topology in more than an order of magnitude higher dimensionality, before they eventually fail too (Figure~\ref{fig:circle-high-dims}).

\begin{figure}
    \centering
    \includegraphics[width=\linewidth]{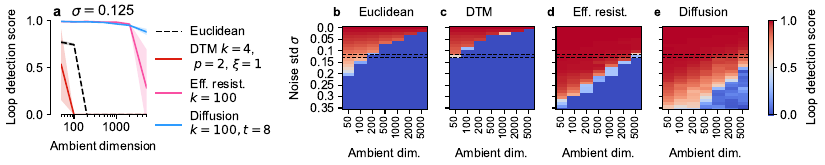}
    \caption{Extension of Figure~\ref{fig:circle-dims} to higher ambient dimensionalities. \textbf{a.} Loop detection scores of various methods on a noisy circle depending on the ambient dimensionality. Due to the higher dimensionalities, we use here the noise with standard deviation $\sigma=0.125$, one half compared to Figure~\ref{fig:circle-dims}a. \textbf{b\,--\,e.} Heat maps for $\sigma\in [0, 0.35]$ and $d\in [50, 5\,000]$. Spectral methods are much more noise robust, but eventually also fail to detect the correct topology.}
    \label{fig:circle-high-dims}
\end{figure}

\section{Persistent homology with PCA}\label{app:pca}

\citet{hiraoka2024curse} recommend combating the curse of dimensionality by performing a normalized PCA before applying persistent homology. We explore this approach here in our setting. 

Let $X = (x_1, \dots, x_n)^T\in \mathbb{R}^{n\times d}$ be the centered data matrix with each data point as a row. Let $X = U S V^T$ be the singular value decomposition of $X$ with $U$ and $V$ orthogonal matrices and $S\in \mathbb{R}^{n\times d}$ diagonal with decreasing non-negative values on the diagonal. PCA reduces the dimensionality to $k\leq d$ dimensions by considering the first $k$ columns of $US$, while normalized PCA instead considers the first $k$ columns of $U$.

We ran experiments on the 1D datasets for both versions with varying number of principal components (PCs) (Figure~\ref{fig:pcas}). Normalization provided no consistent benefit. The performance was worse for higher number of PCs as more noise remained after PCA. Using fewer PCs than the dataset's true dimensionality (Figure~\ref{fig:pcas}f, j) also led to poor performance. Knowing the true dimensionality of the data is therefore crucial for this approach, but true dimensionality is typically not available in real-world applications. Effective resistance outperformed PCA preprocessing on the linked circles and eyeglasses dataset and showed similar performance on the circle.

\begin{figure}[b]
    \centering
    \includegraphics[width=\linewidth]{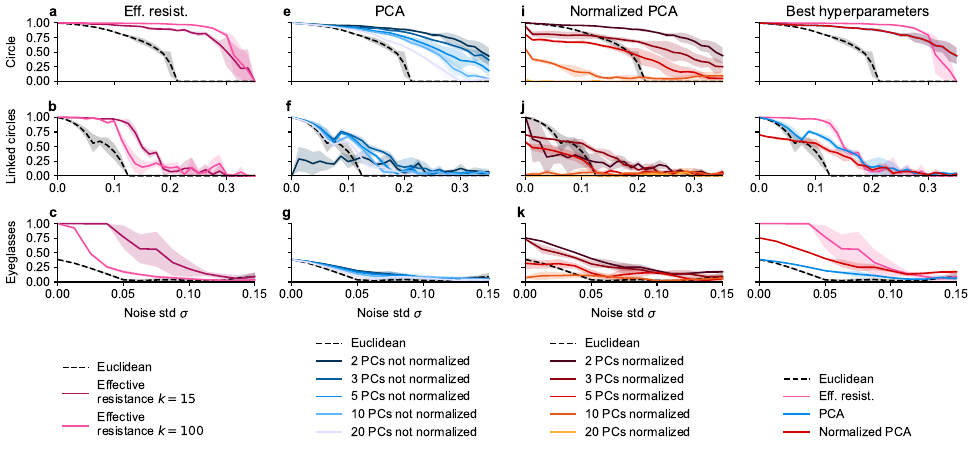}
    \caption{Loop detection score for effective resistance and PCA on three 1D toy datasets.}
    \label{fig:pcas}
\end{figure}

\section{Continuity of the hole detection score}\label{app:score_cont}

In this section we prove that our hole detection score is a continuous map, as long as the persistence diagrams have sufficiently many points. This means that small changes in the persistence diagram result in small changes of the hole detection score, a desirable property for a performance measure.

To state our result formally, we recall some definitions from~\citet{cohen2005stability}:

\begin{definition}[Multiset]
    A multiset is a set in which each element has a multiplicity in $\mathbb{N}_{>0}\cup {\infty}$.
\end{definition}

\begin{definition}[Persistence diagram]
    A persistence diagram is a multiset with elements from
    $\{(b, d) \in (\mathbb{R} \cup \{-\infty, \infty\})^2 \mid b<d\}$ counted with multiplicities, together with the diagonal $\{(b,b) \in (\mathbb{R} \cup \{-\infty, \infty\})^2\}$ counted with infinite multiplicity.
\end{definition}

Since we build persistence diagrams only for distances-based Vietoris--Rips complexes and do not consider $0$-dimensional homologies in this paper, no birth times will be $-\infty$. Moreover, since we do not consider infinite distances (Appendix~\ref{app:dist}), no death times will be infinite either. Therefore, in the following, we will treat persistence diagrams as multisets with elements in $\mathbb{R}^2$.

\begin{definition}[Bottleneck distance]
    Let $D$ and $D'$ be persistence diagrams. Then the bottleneck distance between them is given by 
    $$
    d_B(D, D')= \inf_{\eta: D\to D'} \sup_{x\in D} \|x-\eta(x)\|_\infty,
    $$
    where the infimum is over the bijections between the multisets $D$ and $D'$.
\end{definition}

A bijection $\eta$ realizing the bottleneck distance between persistence diagrams has three types of paired points: Pairs with both points off the diagonal, pairs with exactly one point off the diagonal, and pairs with both points equal and on the diagonal. Intuitively, the bottleneck distance matches the off-diagonal points between each other. However, the second type of pairs is needed in case the persistence diagrams do not have the same number of points off the diagonal. This is why the full diagonal with infinite multiplicity is considered to belong to each persistence diagram. The third type of pairs does not contribute to the bottleneck distance.

We can now proceed with our proofs. Let $\mathcal{D}$ be the space of all persistence diagrams with finitely many points off the diagonal, endowed with the topology induced by the bottleneck distance $d_B$. 

\begin{lemma}\label{lem:pm}
    For any $m \in \mathbb{N}$, the map $p_m:\mathcal{D} \to \mathbb{R}_{\geq 0}$ assigning to a diagram its $m$-th largest persistence value is continuous.
\end{lemma}
\begin{proof}
    For a persistence diagram $D \in \mathcal D$ and a feature $f=(\tau_b, \tau_d) \in D$ we denote by $p(f)=\tau_d - \tau_b$ the persistence of $f$. Note that for two features $f,f'$ in a persistence diagram, we have $|p(f)- p(f')| \leq \sqrt{2}\|f -f'\|_\infty$.
    
    We need to show that for any $\varepsilon > 0$ and any $D \in \mathcal{D}$, we can choose some $\delta>0$ such that $|p_m(D') -p_m(D)| < \varepsilon$ for all $D'$ whose bottleneck distance from $D$ is at most $\delta$, i.e., $D'\in B_{\delta}(D)$. Recall that in this setting, by definition of the bottleneck distance, there is a bijection (of multisets) $\eta: D \to D'$ such that for all $f\in D$ we have $\|f-\eta(f)\|_{\infty} < \delta$. 

    Continuity of $p_m$ is easier to demonstrate at diagrams $D\in \mathcal{D}$ for which all off-diagonal points have different persistences. We discuss this case first. Assume $D$ has $l$ off-diagonal points. Denote by $f_k$ the feature with $k$-th largest persistence. 

    Let $\tilde{\delta} = \min\{p(f_k)-p(f_{k+1}) \mid k\leq l\}$ and $\delta = \min\big(\tilde{\delta} / (2\sqrt{2}), \varepsilon/(2\sqrt{2})\big)$. Consider $D'\in B_{\delta}(D)$.  By the choice of $\delta$, we have that $|p(\eta(f_k)) - p(f_k)| < \tilde{\delta}/2$, so that the values $p(\eta(f_1)),\dots, p(\eta(f_l)), 0$ are all distinct and strictly decreasing. As a result, for all $m\leq l$ we have $p_m(D') = p(\eta(f_m))$ and thus
    $$|p_m(D') - p_m(D)| = |p(\eta(f_m)) - p(f_m)| < \sqrt{2} \delta < \varepsilon.$$

    For $m>l$, we have $p_m(D)=0$ and any $f'\in D'$ with $p_m(D')=p(f')$ is paired with a point on the diagonal of $D$ under $\eta$. Hence
    $$|p_m(D') - p_m(D)| = |p(f') - 0|  < \sqrt{2} \delta < \varepsilon.$$

    The case where $D$ might contain off-diagonal points with identical persistence follows the same overall argument, but requires more careful bookkeeping. Assume $D$ has $l$ off-diagonal points. Let $\tilde{\delta} = \min \big(\{p_k(D) - p_{k+1}(D) \mid k\leq l\}\backslash \{0\}\big)$. Let $$\delta = \min\big(\tilde{\delta}/(2\sqrt{2}), \varepsilon /(2\sqrt{2})\big).$$
    Set $$F_k = \{f\in D \mid p(f)= p_k(D)\}$$ for $k \leq l$ and $F_{l+1}$ equal to the diagonal. These multisets all contain at least one element. We have $p_k(D) = p_{k+1}(D)$ if and only if $F_k = F_{k+1}$.

    As a result, for all $k\leq l$, we have $|\bigcup_{i=1}^{k} F_i |= k$ if and only if $F_k \neq F_{k+1}$ which in turn holds if and only if $p_k(D) \neq p_{k+1}(D)$.

    For any $f, \tilde{f} \in D$ with $p(f) > p(\tilde{f})$ we have 
    \begin{align}
        p(\eta(f))- p(\eta(\tilde{f})) &= p(\eta(f))-p(f) + p(f) - p(\tilde{f}) + p(\tilde{f}) - p(\eta(\tilde{f})) \nonumber \\
        &\geq p(f) - p(\tilde{f}) - | p(\eta(f))-p(f)| - |p(\tilde{f}) - p(\eta(\tilde{f}))| \nonumber \\
        &> p(f) - p(\tilde{f}) - \tilde{\delta}\nonumber \\
        &\geq 0,
    \end{align}
    so $p(\eta(f)) > p(\eta(\tilde{f}))$. We will now show that if $m\leq l$  and $f'\in D'$ such that $p_m(D') = p(f')$, then $\eta^{-1}(f') \in F_m$.

    Let $a$ be the highest number below $m$ such that $F_a\neq F_m$. Then 
    $$
    \left|\bigcup_{i=1}^a \eta(F_i)\right| = \left|\bigcup_{i=1}^a F_i\right| = a.
    $$
    and $\bigcup_{i=1}^a \eta(F_i)$ contains the $a<m$ features of $D'$ with smallest persistence. In particular, \mbox{$\eta^{-1}(f') \notin \bigcup_{i=1}^a F_i$.}

    Let $b$ be the largest number such that $F_h = F_m$. By a similar argument, $\bigcup_{i=1}^b \eta(F_i)$ contains the $b\geq m$ features in $D'$ with largest persistence, including $f'$. Thus,  $\eta^{-1}(f')\in \bigcup_{i=a+1}^{b} F_i$. But by choice of $a$ and $b$, we have $F_{a+1} = \dots = F_m=\dots =F_b$, so $\eta^{-1}(f')\in F_m$.

    Thus, 
    $$
    |p_m(D') - p_m(D)| = |p(f') - p(\eta^{-1}(f'))| \leq \sqrt{2} \|f'-\eta^{-1}(f')\|_{\infty} < \varepsilon.
    $$

    For $m> l$, we have $p_m(D)=0$ and $F_m$ is the set of points on the diagonal of $D$. Moreover, 
    $$\left|\bigcup_{i=1}^l F_i\right| = l = \left|\bigcup_{i=1}^l \eta(F_i)\right|,$$
    so that any $f'\in D'$ with $p_m(D') = p(f')$ must be in the image of the diagonal under $\eta$ and we can proceed as in the case of unique positive persistences.
\end{proof}

Our hole detection score is a continuous score if there are at least $m$ points in the persistence diagram, and not continuous otherwise:

\begin{proposition}
    Let $\mathcal{D}$ be the space of all persistence diagrams with finitely many points off the diagonal, endowed with the topology induced by the bottleneck distance $d_B$. Let $m \in \mathbb{N}$ and define $\mathcal{D}'$ as the subset of $\mathcal{D}$ with at least $m$ points off the diagonal. Then the map $s_m: \mathcal{D} \to \mathbb{R}_{\geq 0}$ is continuous at all diagrams in $\mathcal{D'}$ but is discontinuous at all diagrams in $\mathcal{D}\backslash \mathcal{D'}$.
\end{proposition}
\begin{proof}
    The discontinuity of  $s_m$ at diagrams with less than $m$ points happens when adding the $m$-th point to a diagram. Let $D$ be a persistence diagram with $l<m$ points off the diagonal. Let $\varepsilon >0$ and define $D'$ as the persistence diagram obtained by adding $m-l$ copies of $(0,\varepsilon/2)$. Then $s_m(D) = 0$ and $s_m(D') = 1$, because $D'$ has exactly $m$ off-diagonal points, but $d_B(D, D')< \varepsilon$. Thus, $s_m$ is not continuous at $D$.

    For the continuity part, we write $s_m$ as the composition of two continuous maps. The map \mbox{$q: \mathbb{R}_{>0} \times \mathbb{R}_{\geq 0} \to \mathbb{R}_{\geq 0},$} \mbox{$(x, y) \mapsto (x-y)/x$} is continuous. Furthermore, since each diagram in $\mathcal{D}'$ has at least $m$ points off the diagonal, $p_m(D) \in \mathbb{R}_{>0}$ for all $D\in \mathcal{D}'$. 
    By lemma~\ref{lem:pm} the map \mbox{$p_m\times p_{m+1}:\mathcal{D}'\to \mathbb{R}_{\geq 0} \times \mathbb{R}_{\geq 0}$} is continuous.
    Finally, $s_m: \mathcal{D}' \to \mathbb{R}_{\geq 0}$ factors as $q\circ (p_m\times p_{m+1})$, showing its continuity.
\end{proof}

In other words, on persistence diagrams with sufficiently many points our hole detection score is continuous. This makes our score robust against small perturbations in the persistence diagram, a property expected for a reliable score.

In the typical case, where the persistence diagram contains a noise cloud of many points close to the diagonal and, possibly, some outliers far away from the diagonal, the requirement on the number of points is satisfied. Only in the case where we expect a large gap after the $m$-th feature, but the diagram does not even have $m$ features, can there be a sudden jump in our hole detection score.

\section{Alternative performance scores}\label{app:alt_scores}

While our hole detection score corresponds to an intuitive visual assessment of persistence diagrams and is continuous (Appendix~\ref{app:score_cont}), in this section we discuss two alternative performance scores. We argue that both of them are less suited for our benchmark.

\subsection{Widest gap score}
A natural way to interpret a persistence diagram is to deem all features above the widest gap in persistence values as true features and the rest as noise, i.e., infer that a diagram has $a$ true features if $a= \text{argmax}_{\alpha} (p_{\alpha} - p_{\alpha+1})$. 

We considered the binary score which is equal to 1 if the number of features above the widest gap equals the number of ground-truth features and 0 otherwise. We call it the \textit{widest gap score}.

Note that our hole detection score can be lower than 1 even if the persistence diagram clearly shows the correct number of outlier features (e.g. $s_m\approx 0.5$ for $\sigma = 0.2$ in Figure~\ref{fig:curse}b). In contrast, the widest gap score is equal to 1 in this case.

On the other hand, the binary nature of the widest gap score leads to instability in case of nearly equisized gaps. For instance in Figure \ref{fig:malaria}d, a small perturbation could make the gap after the most persistent feature the widest, dropping the widest gap score from 1 to 0. In contrast, our $s_2\approx 0.9$ correctly reflects the two clear outlier features.

We evaluated several distances on the three 1D datasets using the widest gap score in Figure~\ref{fig:widest_gap}. The binary nature of the widest gap score led to high-variance results even though we increased the number of random seeds from 3 to 10. Spectral methods still performed better than Euclidean, Fermat, and DTM distances. However, the widest gap score led to some false positives. For example, for the Fermat distance and $\sigma=0.3$ for the circle dataset, in 9 of 10 trials the widest gap appeared after the first feature. However, all 9 persistence diagrams looked like noise clouds and in 8 cases the most persistent feature's representative cycle was clearly a noise feature (Figure~\ref{fig:widest_gap}d\,--\,e). Conversely, for the effective resistance there was always a clear outlier in the diagram and the representative did follow the ground-truth loop (Figure~\ref{fig:widest_gap}f\,--\,g). Although the Fermat distance failed in this setting and the effective resistance performed well, their widest gap scores were both high.

For these reasons, we prefer our hole detection score to the widest gap score.

\begin{figure}
    \centering
    \includegraphics[width=\linewidth]{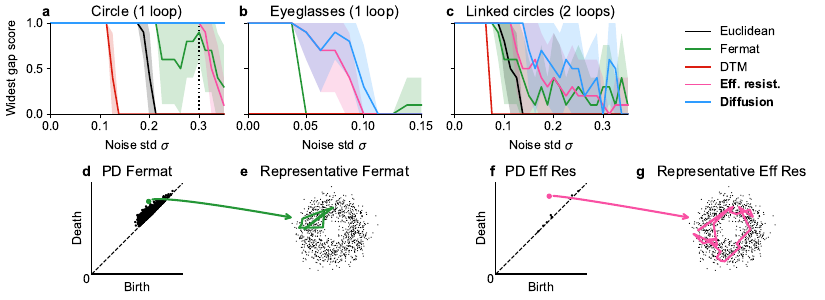}
    \caption{Top row: Widest gap score on the three 1D toy datasets. We used 10 random seeds here. Spectral methods outperformed Fermat, DTM, and Euclidean in this score. Bottom row: Example persistence diagrams and representatives for the toy circle with $\sigma=0.3$ with Fermat distance and effective resistance illustrate that despite the high score for both, Fermat distance actually failed.}
    \label{fig:widest_gap}
\end{figure}

\subsection{Cycle matching}

The need to distinguish between true signal and noise features in a persistence diagram motivated~\citet{reani2022cycle} to develop \textit{cycle matching}, a technique that quantifies the correspondence of features in two persistence diagrams by a prevalence score in $[0, 1]$. True signal corresponds to features that reappear  with high prevalence in variations of the data obtained, e.g., via resampling.

This approach to detecting the true topology of high-dimensional data is orthogonal to our exploration of different input distances. Indeed, cycle matching can be combined with any distance.

An important downside to cycle matching is the need for resampling. In an exploratory context, the true data distribution is not known. Instead, \citet{reani2022cycle} suggest to either bootstrap the existing finite dataset or sample from a kernel density estimate. Unfortunately, for high-dimensional data both approaches are problematic. Bootstrapping high-dimensional data is error-prone~\citep{clarte2024analysis} and persistent homology of bootstraps is biased because repeated points effectively decrease the sample size~\citep{roycraft2023bootstrapping}. Kernel density estimation requires a prohibitively large sample size in high dimensions~\citep[chapter 6.5]{wasserman2006all}.

Nevertheless, as a proof of concept, we explored cycle matching with the relatively fast \texttt{ripser}-based implementation of~\citet{garcia2024fast}. The authors report a runtime of about $90$ minutes for a dataset of sample size $n=1000$ (their Table 3) and in our experiments we experienced even longer run times. As the result, this implementation of cycle matching was about $450$ times slower than our approach (Table~\ref{tab:run_times}). For this reason, we only used three resamples.

\citet{garcia2024fast}'s implementation accepts datasets as point clouds and implicitly assumes the Euclidean distance. Therefore, in this experiment we only used Euclidean distances and diffusion distances, which can be realized as Euclidean distances (Eq.~\ref{eq:spectral_diff}). We computed prevalences of matched cycles on the noisy circle with $n=1000$ points in $\mathbb{R}^{50}$ for Gaussian noise of standard deviations $\sigma \in [0.1, 0.2, 0.25,0.3, 0.35]$.

\begin{figure}[tb]
    \centering
    \includegraphics[width=\linewidth]{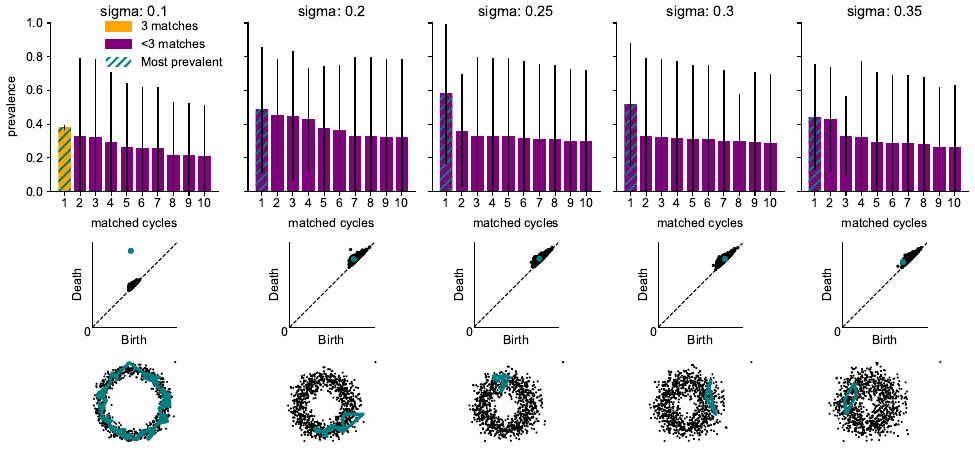}
    \caption{Results for cycle matching with the Euclidean distance on a noisy circle in $\mathbb{R}^{50}$ with noise level $\sigma$. Top row: Prevalences of the 10 most prevalent cycles. Means and standard deviation over three seeds. Color indicates whether the cycle was matched for all three random seeds or not. Second row: Persistence diagrams with the most prevalent features highlighted. Third row: Representative of the most prevalent feature overlaid on a 2D PCA of the data.}
    \label{fig:cycle_match_eucl}
\end{figure}

\begin{figure}[tb]
    \centering
    \includegraphics[width=\linewidth]{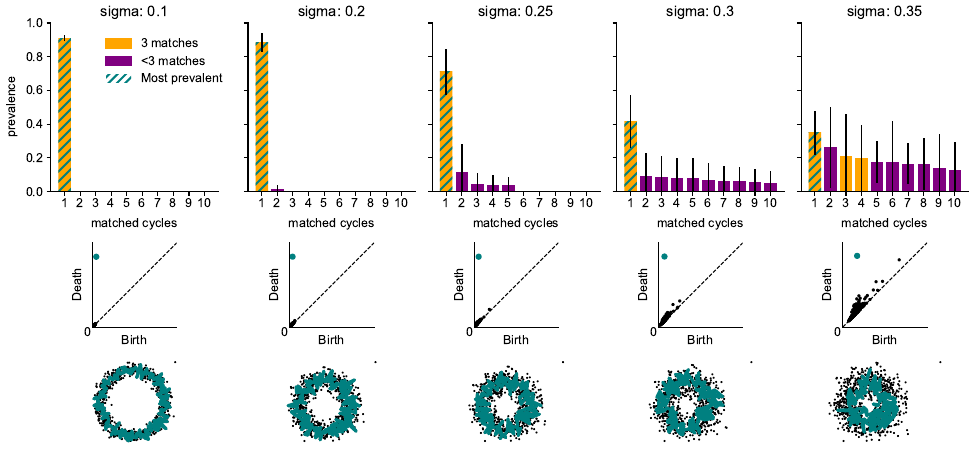}
    \caption{Results for cycle matching with the diffusion distance ($k=15, t=8$) on a noisy circle in $\mathbb{R}^{50}$ with noise level $\sigma$. Top row: Prevalences of the 10 most prevalent cycles. Means and standard deviation over three seeds. Color indicates whether the cycle was matched for all three random seeds or not. Second row: Persistence diagrams with the most prevalent features highlighted. Third row: Representative of the most prevalent feature overlaid on a 2D PCA of the data.}
    \label{fig:cycle_match_diff}
\end{figure}

Cycle matching with the Euclidean distance did not reliably identify the ground-truth feature for $sigma > 0.1$ (Figure~\ref{fig:cycle_match_eucl}). In all experiments with the Euclidean distance and $\sigma>0.1$, the feature with the highest prevalence did not correspond to the ground-truth loop. 

For $\sigma= 0.2, 0.25$, we also checked the prevalences of the ground truth cycle. For $\sigma=0.2$ the ground truth cycle was not matched for any of the three resamples, leading to a prevalence of $0$. For $\sigma=0.25$ the ground-truth cycle had the average prevalence of $0.1 \pm 0.1$, much lower than the cycle with the highest prevalence. 

One might hope that the cycles that do get matched for all resamples are true topological features, even if their prevalence values are low. Indeed, for $\sigma =0.1$ only the correct cycle was matched for all three resamples. However this heuristic did not work either because for $\sigma \in \{0.2, 0.25, 0.3\}$ no cycle was matched for all three resamples, while for $\sigma=0.35$ there was a cycle matched for all three resamples (not shown in Figure~\ref{fig:cycle_match_eucl} because it was not among the 10 most prevalent ones), but it did not encode the ground-truth feature (and its prevalence $0.1 \pm 0.1$ was very low). 

The fact that most cycles were not matched for all resamples led to high standard deviations of the prevalences. The reason for the high mean prevalences for $\sigma>0.1$ was that some noise cycles got matched with very high prevalence (sometimes $>0.9$) for one or two of the resamples. For the same reason, the prevalence of the ground truth cycle for $\sigma=0.1$ was not much higher than the second highest prevalence. Increasing the number of  resamples may help, but would make the procedure even more computationally expensive.

Overall, we conclude that cycle matching is not only very slow, but also does not alleviate the curse of dimensionality for persistent homology. In contrast, our spectral methods produced a near-perfect loop detection score until $\sigma=0.3$ (Figure~\ref{fig:circle-benchmark}).

That said, cycle matching can serve as an alternative performance score. We applied cycle matching to the diffusion distance (Figure~\ref{fig:cycle_match_diff}) and found that 
\begin{enumerate}[nosep, topsep=-0.5\parskip, leftmargin=0.5\leftmargin]
    \item the maximal prevalences were larger than for the Euclidean distance for $\sigma < 0.3$,
    \item the cycle with maximal prevalence always represented the ground truth loop,
    \item the prevalence of the correct cycle was a clear outlier among all prevalences for $\sigma \leq 0.3$,
    \item the cycle representing the ground truth feature was always matched for all three resamples and was the only one matched for all three resamples for $\sigma\leq 0.3$.
\end{enumerate}

This confirms the superior performance of spectral methods for detecting topology in high-dimensional settings.

\section{Effective resistance as a spectral method}
\label{app:eff_res}

In this section, we derive an explicit spectral embedding realizing the square root of \citet{luxburg2010getting}'s \textit{corrected} effective resistance. We also show that it is not necessarily a proper metric.

Let us recall the notation from Sections~\ref{sec:background-spectral} and~\ref{sec:spectral}, extended to weighted graphs. Let $G$ be a weighted connected graph with $n$ nodes. We denote by $A = (a_{ij})_{i, j = 1, \dots, n}$ the weighted adjacency matrix whose entries $a_{ij}=a_{ji}$ equal the edge weight of edge $ij$ if edge $ij$ is part of the graph and zero otherwise. We further denote by $D = \text{diag}(d_i)$ the degree matrix, where \mbox{$d_i = \sum_j a_{ij}$} are the node degrees. We define $\text{vol}(G) = \sum_i d_i$. Let further $\asym = D^{-\frac{1}{2}} A D^{-\frac{1}{2}}$  and $\lsym = I - \asym$ be the symmetrically normalized adjacency matrix and the symmetrically normalized graph Laplacian. Denote the eigenvectors of $\lsym$ by $u_1, \dots, u_n$ and their eigenvalues by $\mu_1, \dots, \mu_n$ in increasing order. The eigenvectors and eigenvalues of $\asym$ are $u_1,\dots, u_n$ and $1-\mu_1, \dots, 1-\mu_n$.

\begin{definition}
    The naive effective resistance distance between nodes $i$ and $j$ is defined as 
    \begin{equation}
        \tilde{d}^{\text{eff}}_{ij} = \frac{1}{\text{vol}(G)} (H_{ij} + H_{ji}), 
    \end{equation}
    where $H_{ij}$ is the hitting time from $i$ to $j$, i.e., the expected number of steps that a random walker starting at node $i$ takes to reach node $j$ for the first time.
    
    The corrected effective resistance distance between nodes $i$ and $j$ is defined as 
    \begin{equation}
        d^{\text{eff}}_{ij} = \tilde{d}^{\text{eff}}_{ij} - \frac{1}{d_i} - \frac{1}{d_j} + 2 \frac{a_{ij}}{d_id_j} - \frac{a_{ii}}{d_i^2} - \frac{a_{jj}}{d_j^2}.
    \end{equation}
\end{definition}

The following proposition refines Proposition 4 in \cite{luxburg2010getting}.
\begin{proposition}\label{prop:corr_eff_res}
The corrected effective resistance distance $d^\textnormal{eff}_{ij}$ between nodes $i$ and $j$ can be computed by \mbox{$d^{\textnormal{eff}}_{ij} = \| e^{\textnormal{eff}}_i -e^{\textnormal{eff}}_j\|^2$,} where 
$$e^{\textnormal{eff}}_i = \frac{1}{\sqrt{d_i}} \left(\frac{1-\mu_2}{\sqrt{\mu_2} }u_{2,i},\dots, \frac{1-\mu_n}{\sqrt{\mu_n} }u_{n,i}\right).$$
\end{proposition}
\begin{proof}
    By the fourth step of the large equation in the proof of Proposition 2 in \cite{von2010hitting}, we have
    \begin{equation}
        \frac{1}{\text{vol}(G) }H_{ij} = \frac{1}{d_j}+ \langle b_j, \asym (b_j-b_i)\rangle + \sum_{r=2}^n \frac{1}{\mu_r} \langle \asym b_j, u_r u_r^\top \asym (b_j -b_i)\rangle,
    \end{equation}
    where $b_i = \frac{1}{\sqrt{d_i}}\hat{e}_i = D^{-\frac{1}{2}}\hat{e}_i$ with $\hat{e}_i$ being the $i$-th standard basis vector.
    
    Adding this expression for $ij$ and $ji$ and using the definition of $\tilde{d}^\text{eff}_{ij}$, we get 
    \begin{align}
        \tilde{d}^\text{eff}_{ij} - \frac{1}{d_i} - \frac{1}{d_j}
        &= \frac{1}{\text{vol}(G)} (H_{ij} + H_{ji})- \frac{1}{d_i} - \frac{1}{d_j} \\
        &= \langle b_j, \asym (b_j -b_i)\rangle +  \langle b_i, \asym (b_i -b_j)\rangle \\
        & \quad +         \sum_{r=2}^n \frac{1}{\mu_r} \langle \asym b_j, u_ru_r^\top \asym (b_j -b_i)\rangle\\
        & \quad+ \sum_{r=2}^n \frac{1}{\mu_r} \langle \asym b_i, u_r u_r^\top \asym (b_i -b_j)\rangle \\
        &= \langle b_j-b_i, \asym (b_j -b_i)\rangle \\
        &\quad + \sum_{r=2}^n \frac{1}{\mu_r} \langle \asym (b_j - b_i), u_r u_r^\top \asym (b_j -b_i)\rangle
    \end{align}
    
    For ease of exposition, we treat the two terms separately. By unpacking the definitions and using the symmetry of $D$, we get
    \begin{align}
        \langle b_j-b_i, \asym (b_j -b_i)\rangle
        &= \langle D^{-\frac{1}{2}}(\hat{e}_j-\hat{e}_i), \asym D^{-\frac{1}{2}}(\hat{e}_j-\hat{e}_i)\rangle\\
        &= \langle (\hat{e}_j-\hat{e}_i), D^{-1} A D^{-1}(\hat{e}_j-\hat{e}_i)\rangle \\
        &= \frac{a_{jj}}{d_j^2} - 2 \frac{a_{ij}}{d_id_j} + \frac{a_{ii}}{d_i^2}
    \end{align}

    Since the $u_r$ are eigenvectors of $\asym$ with eigenvalue $1-\mu_r$ and $\asym$ is symmetric, we also get
    \begin{align}
        &\sum_{r=2}^n \frac{1}{\mu_r} \langle \asym (b_j - b_i), u_r u_r^\top \asym (b_j -b_i)\rangle\\
        =& \sum_{r=2}^n \frac{1}{\mu_r} \langle b_j - b_i,  (\asym u_r)( \asym u_r)^\top (b_j -b_i))\rangle\\
        =& \sum_{r=2}^n \frac{1}{\mu_r} \langle b_j - b_i,  (1-\mu_r) u_r\big( (1-\mu_r) u_r\big)^\top (b_j -b_i)\rangle\\
        = & \sum_{r=2}^n  \left(\frac{1-\mu_r}{\sqrt{\mu_r}}u_r^\top D^{-\frac{1}{2}}(\hat{e}_j - \hat{e}_i)\right)^2\\
        =& \left\| 
        \frac{1}{\sqrt{d_j}} \left(\frac{1-\mu_r}{\sqrt{\mu_r}} u_{r,j}\right)_{r=2,\dots, n}  -
        \frac{1}{\sqrt{d_i}} \left(\frac{1-\mu_r}{\sqrt{\mu_r} }u_{r,i}\right)_{r=2,\dots, n}
        \right\|^2 \\
        =& \| e_j - e_i \|^2
    \end{align}
    Putting everything together yields the result
    \begin{align}
        d^\text{eff}_{ij} &= \tilde{d}^\text{eff}_{ij} - \frac{1}{d_i} - \frac{1}{d_j} + 2 \frac{a_{ij}}{d_id_j} - \frac{a_{ii}}{d_i^2} - \frac{a_{jj}}{d_j^2} \\
        &= \frac{a_{jj}}{d_j^2} - 2 \frac{a_{ij}}{d_id_j} + \frac{a_{ii}}{d_i^2} +  \| e_j - e_i \|^2 + 2 \frac{a_{ij}}{d_id_j} - \frac{a_{ii}}{d_i^2} - \frac{a_{jj}}{d_j^2}\\
        &= \| e^\textnormal{eff}_j - e^\textnormal{eff}_i \|^2.
    \end{align}
\end{proof}

The corrected version of effective resistance and diffusion distances are, in general, not proper metrics, unlike the naive effective resistance~\citep[Corollary 2.4.]{gobel1974random}. We show this by giving concrete examples of graphs where the metric axioms for these distances do not hold.

\begin{proposition}\label{prop:corr_eff_res_not_metric}
    Neither corrected effective resistance nor diffusion distances between distinct points are necessarily positive. Moreover, corrected effective resistance does not always satisfy the triangle inequality.
\end{proposition}
\begin{proof}
    Consider the unweighted chain graph with three nodes (Figure~\ref{fig:graphs}, left). The uncorrected effective resistance between the first and the last node is $2$. So the corrected effective resistance between these distinct nodes is $2- 1- 1= 0$. On the same graph, the random walker is necessarily at node $1$ after one step, independent whether it started at node $0$ or $2$, so the diffusion distance with $t=1$ between these two distinct nodes is zero.

    Consider now an unweighted graph with $5$ nodes, the first three of which form a triangle and the other two a chain connected to the triangle (Figure~\ref{fig:graphs}, right). Then
    \begin{align}
        \tilde{d}_{04}^\text{eff} =  \frac{8}{3}, \quad \tilde{d}_{02}^\text{eff} = \frac{2}{3}, \quad
        \tilde{d}_{24}^\text{eff}  = 2; \quad d_{04}^\text{eff}= \frac{7}{6}, \quad
        {d}_{02}^\text{eff}  = \frac{1}{6}, \quad d_{24}^\text{eff} =\frac{2}{3}.
    \end{align}
    We see that the triangle $0, 2, 4$ violates the triangle inequality
    $$d_{04}^\text{eff} =\frac{7}{6} > \frac{2}{3} + \frac{1}{6} = d_{02}^\text{eff}+ d_{24}^\text{eff}.$$
\end{proof}

The square root of corrected effective resistance does satisfy the triangle inequality and can also be used as input to persistent homology. This amounts to taking the square root of all birth and death times found with the corrected effective resistance as computing persistent homology commutes with strictly monotonic maps, and hence does not strongly affect the loop detection performance of corrected effective resistance~(Figure~\ref{fig:eff_res_comparison}). We therefore believe that it is not a problem for our persistent homology application that corrected effective resistance fails to satisfy the triangle inequality.

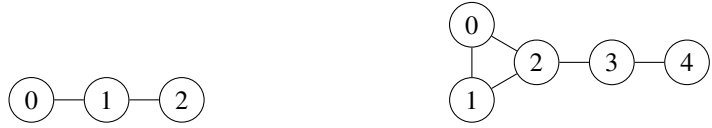
\begin{figure}
    \centering
    \begin{subfigure}{0.45\textwidth}
        \centering
        \begin{tikzpicture}[main_node/.style={circle,,draw,minimum size=1em,inner sep=3pt]}]
            \node[main_node] (0) at (0,0.0) {0};
            \node[main_node] (1) at (1, 0)  {1};
            \node[main_node] (2) at (2, 0) {2};
            \draw (0) -- (1) -- (2);
        \end{tikzpicture}
    \end{subfigure}%
    \begin{subfigure}{0.45\textwidth}
        \centering
        \begin{tikzpicture}[main_node/.style={circle,,draw,minimum size=1em,inner sep=3pt]}]
            \node[main_node] (0) at (0,0.5) {0};
            \node[main_node] (1) at (0, -0.5)  {1};
            \node[main_node] (2) at (0.86, 0) {2};
            \node[main_node] (3) at (1.86, 0) {3};
            \node[main_node] (4) at (2.86, 0) {4};
            \draw (0) -- (1) -- (2) -- (0);
            \draw (2) -- (3) -- (4);
        \end{tikzpicture}
    \end{subfigure}
    \caption{Counterexample graphs for Proposition~\ref{prop:corr_eff_res_not_metric}}
    \label{fig:graphs}
\end{figure}

\section{Effective resistance integrates diffusion distances}\label{app:diff_eff_res}

First, we describe the connection between diffusion distances and naive effective resistance communicated to us by~\citet{memoli}. To the best of our abilities, we could not find another reference. We use the same notation as in Section~\ref{sec:spectral} and Appendix~\ref{app:eff_res}.

Below, we speak of diffusion distances at half-integer time points, although the random walk analogy does not extend to half-steps. What we mean is inserting a half-integer value of $t$ in Eq.~\eqref{eq:spectral_diff}.

\begin{proposition}\label{prop:diff_naive_eff_res}
    The sum of the squared diffusion distances over all time points $t/2$ for \mbox{$t=0, 1, 2, \dots$} equals the naive effective resistance. Formally, if $G$ is a connected and non-bipartite graph, we have
    \begin{align}
        \frac{1}{\textnormal{vol}(G)}\sum_{t=0}^\infty d^{\textnormal{diff}}_{ij}(t/2)^2 = \tilde{d}^{\textnormal{eff}}_{ij}.\label{eq:odd_naive_eff_res}
    \end{align}
\end{proposition}

\begin{proof}
    This is an application of the geometric series. Since $G$ is not bipartite, the largest eigenvalue $\mu_n$ of its Laplacian $L^{\text{sym}}$ is smaller than $2$~\citep[ Lemma 1.7]{chung1997spectral}. Recall that by Eq.~\eqref{eq:spectral_diff} the diffusion distance for $t$ diffusion steps is given by 
    \begin{align}
    d_{ij}^{\text{diff}}(t) = \sqrt{\textnormal{vol}(G)}\| e_i^{\text{diff}}(t) - e_j^{\text{diff}}(t)\|, \text{ where } e_i^{\text{diff}}(t) = \frac{1}{\sqrt{d_i}} \left((1-\mu_2)^t u_{2,i}, \ldots, (1-\mu_n)^t u_{n,i}\right).
\end{align}
and that the naive effective resistance is given by

\begin{equation}
    \tilde{d}^{\text{eff}}_{ij} = \| \tilde{e}_i^{\text{eff}} - \tilde{e}_j^{\text{eff}}\|^2, \text{ where } \tilde{e}_i^{\text{eff}} = \frac{1}{\sqrt{d_i}}\left(\frac{1}{\sqrt{\mu_2}} u_{2,i},\ldots,\frac{1}{\sqrt{\mu_n}} u_{n,i}\right).
\end{equation}
    We compute
    \begin{align}
        \frac{1}{\text{vol}(G)}\sum_{t=0}^\infty d^{\textnormal{diff}}_{ij}(t/2)^2 
        &=\sum_{t=0}^\infty \sum_{l=2}^n (1-\mu_l)^{2\cdot t/2} \Big(\frac{u_{l, i}}{\sqrt{d_i}} - \frac{u_{l, j}}{\sqrt{d_j}}\Big)^2 \\
        &= \sum_{l=2}^n \Big(\frac{u_{l, i}}{\sqrt{d_i}} - \frac{u_{l, j}}{\sqrt{d_j}}\Big)^2 \sum_{t=0}^\infty (1-\mu_l)^{t} \\
        &= \sum_{l=2}^n \Big(\frac{u_{l, i}}{\sqrt{d_i}} - \frac{u_{l, j}}{\sqrt{d_j}}\Big)^2 \frac{1}{1- (1-\mu_l)} \\
        &= \sum_{l=2}^n \frac{1}{\mu_l}\Big(\frac{u_{l, i}}{\sqrt{d_i}} - \frac{u_{l, j}}{\sqrt{d_j}}\Big)^2 \\
        &= \| \tilde{e}^\text{eff}_i - \tilde{e}^\text{eff}_j \| ^2 \\
        &= \tilde{d}^{\text{eff}}_{ij}.
    \end{align}
    The application of the geometric series is justified as $\mu_l \in (0, 2)$ for $l=2, \dots, n$ by assumption on $G$, so that \mbox{$|1-\mu_l| < 1$.}
\end{proof}

Similarly, we show the corresponding result for the corrected version of effective resistance.

\begin{corollary}\label{cor:diff_corr_eff_res}
    The sum of the squared diffusion distances over all time points $t/2$ for $t=2, 3, \dots$ equals the corrected effective resistance. Formally, if $G$ is a connected non-bipartite graph, we have
    \begin{equation}
        \frac{1}{\textnormal{vol}(G)}\sum_{t=2}^\infty d^{\textnormal{diff}}_{ij}(t/2)^2 = d^{\textnormal{eff}}_{ij}\label{eq:diff_corr_eff_res}
    \end{equation}
   and
    \begin{equation}
        \frac{d^{\textnormal{diff}}_{ij}(0)^2}{\textnormal{vol}(G)} =  \frac{1}{d_i} + \frac{1}{d_j}, \qquad
        \frac{d^{\textnormal{diff}}_{ij}(1/2)^2}{\textnormal{vol}(G)} = \frac{a_{ii}}{d_i^2} + \frac{a_{jj}}{d_j^2}  - 2 \frac{a_{ij}}{d_id_j}
    \end{equation}
\end{corollary}
\begin{proof}
    For $x\in (-1, 1)$, starting a geometric series at the $m$-th term yields $\sum_{t=m}^\infty x^t= \frac{x^m}{1-x}$. By Proposition~\ref{prop:corr_eff_res}, we have 
    \begin{equation}
        d^\textnormal{eff}_{ij} = \sum_{l=2}^n \frac{(1-\mu_l)^2}{\mu_l}\left(\frac{u_{l, i}}{\sqrt{d_i}} - \frac{u_{l, j}}{\sqrt{d_j}}\right)^2  =  \sum_{l=2}^n \frac{(1-\mu_l)^2}{1-(1-\mu_l)}\left(\frac{u_{l, i}}{\sqrt{d_i}} - \frac{u_{l, j}}{\sqrt{d_j}}\right)^2.
    \end{equation}
    By the same argument as in the proof of Proposition~\ref{prop:diff_naive_eff_res}, this shows the first part of the Corollary. 

    The statement for diffusion distance after $t=0$ steps is an easy computation from the original definition, Eq.~\eqref{eq:diff}:
    \begin{align}
        d^\text{diff}_{ij}(0) &= \sqrt{\text{vol}(G)}\|(P^0_{i, :}- P^0_{j, :})D^{-0.5}\|\notag \\
        &= \sqrt{\text{vol}(G)}\|(\hat{e}_i^\top- \hat{e}_j^\top)D^{-0.5}\| \notag \\
        &=  \sqrt{\text{vol}(G) \left(\frac{1}{d_i} + \frac{1}{d_j}\right)},
    \end{align}
    where $\hat{e}_i$ is the $i$-th standard basis vector. The last part follows from the definition of the corrected effective resistance as 
    \begin{equation}
        d^\text{eff}_{ij} = \tilde{d}^\text{eff}_{ij} - 1/d_i - 1/d_j + 2 a_{ij}/(d_id_j) - a_{ii}/d_i^2 - a_{jj}/d_j^2.
    \end{equation}
    and the expression of the naive effective resistance as full geometric series of squared diffusion distances in Proposition~\ref{prop:diff_naive_eff_res}.
\end{proof}

\section{Diffusion pseudotime}
\label{app:dpt}

Diffusion pseudotime (DPT), $d^\text{dpt}$, \cite{haghverdi2016diffusion} also considers diffusion processes of arbitrary length. There are several variants of diffusion pseudotime~\cite{haghverdi2016diffusion, wolf2019paga}. They are all computed as  $ d^\textnormal{dpt}_{ij} = \| e_i^\textnormal{dpt} - e_j^\textnormal{dpt}\|$ for 
    \begin{align}
        e^\textnormal{dpt}_i = \left(\frac{1-\mu_2}{\mu_2} v_{2,i}, \dots,\frac{1-\mu_n}{\mu_n} v_{n,i} \right).
    \end{align}
The difference is the $v_1, \dots, v_n$'s. In the original publication~\citep{haghverdi2016diffusion}, they were the normalized eigenvectors of the random walker graph Laplacian $D^{-\frac{1}{2}}\lsym D^\frac{1}{2}$. These are given by $v_l = D^{-\frac{1}{2}}u_l  / \|D^{-\frac{1}{2}}u_l\|$. \citet{wolf2019paga} introduced a version using the eigenvectors of the symmetric graph Laplacian $\lsym$, so that $v_l= u_l$. Closest to the corrected resistance distance is the case where $v_l =D^{-\frac{1}{2}}u_l$ are non-normalized. We refer to these three versions as ``rw'', ``sym'' and ``symd''. 

The only difference between the ``symd'' version of DPT and corrected effective resistance is that the former decays eigenvalues with $(1-\mu_l)/\mu_l$, while the latter decays it with $(1-\mu_l)/\sqrt{\mu_l}$, so that DPT decays large eigenvalues more strongly than corrected effective resistance (Figure~\ref{fig:eigenvalues}).

Similar to effective resistance, one can also write diffusion pseudotime in terms of diffusion distances. But for diffusion pseudotime the diffusion distances corresponding to higher diffusion times contribute more.

\begin{proposition}\label{prop:dpt_diff}
    Let $G$ be a connected, non-bipartite graph. We can write the ``symd'' version of DPT as
    \begin{equation}
        d_{ij}^\textnormal{dpt} = \sqrt{\frac{1}{\textnormal{vol}(G)}\sum_{t=1}^\infty (t-1) d_{ij}^\textnormal{diff}(t/2)^2}.
    \end{equation}
\end{proposition}
\begin{proof}
    We will use that for a real number $x$ with $|x|<1$ the derivative of the geometric series is
    $$
    \sum_{t=1}^\infty tx^{t-1} = \frac{1}{(1-x)^2}.
    $$
    Multiplying with $x^2$, we get 
    \begin{align}
        \frac{x^2}{(1-x)^2} =  \sum_{t=1}^\infty tx^{t+1} = \sum_{t=2}^\infty (t-1)x^t.
    \end{align}
    By assumption on $G$, for all eigenvalues $\mu_l$ we have $|1-\mu_l|<1$ for $l=2, \dots, n$, so that the above equation holds for all $x=1-\mu_l$ with $l=2, \dots, n$. Together with the ``symd'' expression for diffusion pseudotime, we compute
    \begin{align}
        \big(d_{ij}^\text{dpt}\big)^2 &= \sum_{l=2}^{n} \frac{(1-\mu_l)^2}{\mu_l^2} \Big(\frac{u_{l, i}}{\sqrt{d_i}} - \frac{u_{l, j}}{\sqrt{d_j}}\Big)^2 \notag \\
        &= \sum_{l=2}^{n} \sum_{t=2}^\infty (t-1) (1-\mu_l)^t  \Big(\frac{u_{l, i}}{\sqrt{d_i}} - \frac{u_{l, j}}{\sqrt{d_j}}\Big)^2 \notag \\
        &= \sum_{t=2}^\infty (t-1) \sum_{l=2}^n (1-\mu_l)^{2\cdot(t/2)} \Big(\frac{u_{l, i}}{\sqrt{d_i}} - \frac{u_{l, j}}{\sqrt{d_j}}\Big)^2 \notag \\
        &=  \sum_{t=2}^\infty (t-1) d_{ij}^\text{diff}(t/2)^2.
    \end{align}
\end{proof}

\subsection{Performance of diffusion pseudotime and potential distance}

Here we compare the spectral methods of our main benchmark to the three versions of DPT and potential distance. 

For DPT, we used all three versions defined above and used $k\in\{15, 100\}$ nearest neighbors.

The potential distance underlies the visualization method PHATE~\citep{moon2019visualizing} and is closely related to the diffusion distance. It is defined by 
$$d_t^{\text{pot}}(i, j) = \| \log(P^t_{i, :}) - \log(P^t_{j,:})\|,$$
where the logarithm is applied element-wise. We used $k\in\{15, 100\}$ and $t\in\{8, 64\}$.

Overall, all spectral methods performed very similarly on the 1D datasets circle, linked circles, and eyeglasses (Figures~\ref{fig:comparison_spectrals}, \ref{fig:best_spectrals}), as well as on the single-cell datasets (Figure~\ref{fig:scRNAseq_all_spectral}).

\begin{figure}
    \centering
    \includegraphics[width=\linewidth]{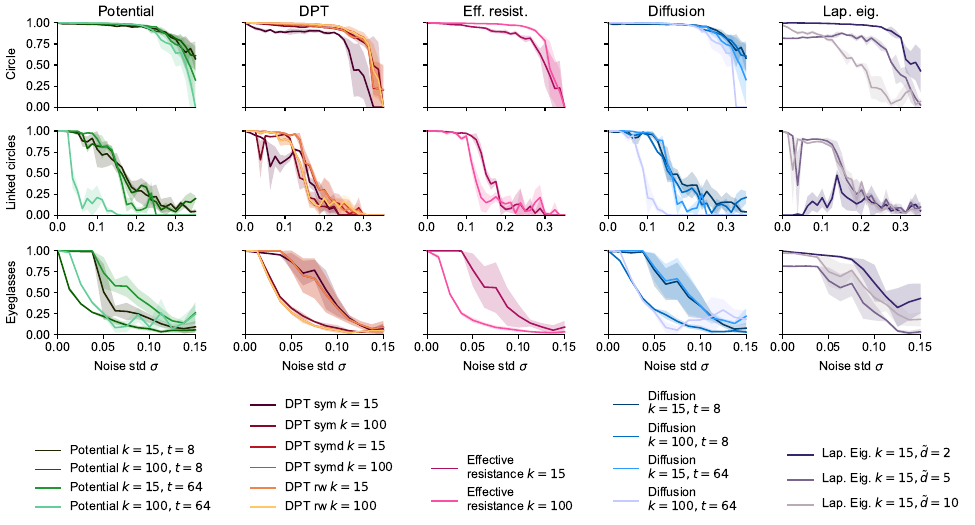}
    \caption{Comparison of all spectral methods on the noised versions of the circle, the linked circles, and the eyeglassed dataset in $\mathbb{R}^{50}$.}
    \label{fig:comparison_spectrals}
\end{figure}

\begin{figure}
    \centering
    \includegraphics[width=0.8\linewidth]{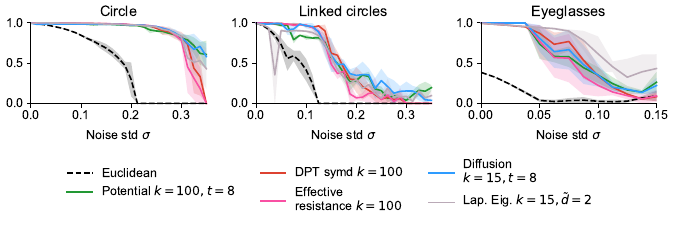}
    \caption{Best hyperparameter choices for the methods in Figure~\ref{fig:comparison_spectrals}. All spectral methods reach very similar performance.}
    \label{fig:best_spectrals}
\end{figure}

\begin{figure}
    \centering
    \includegraphics[width=\linewidth]{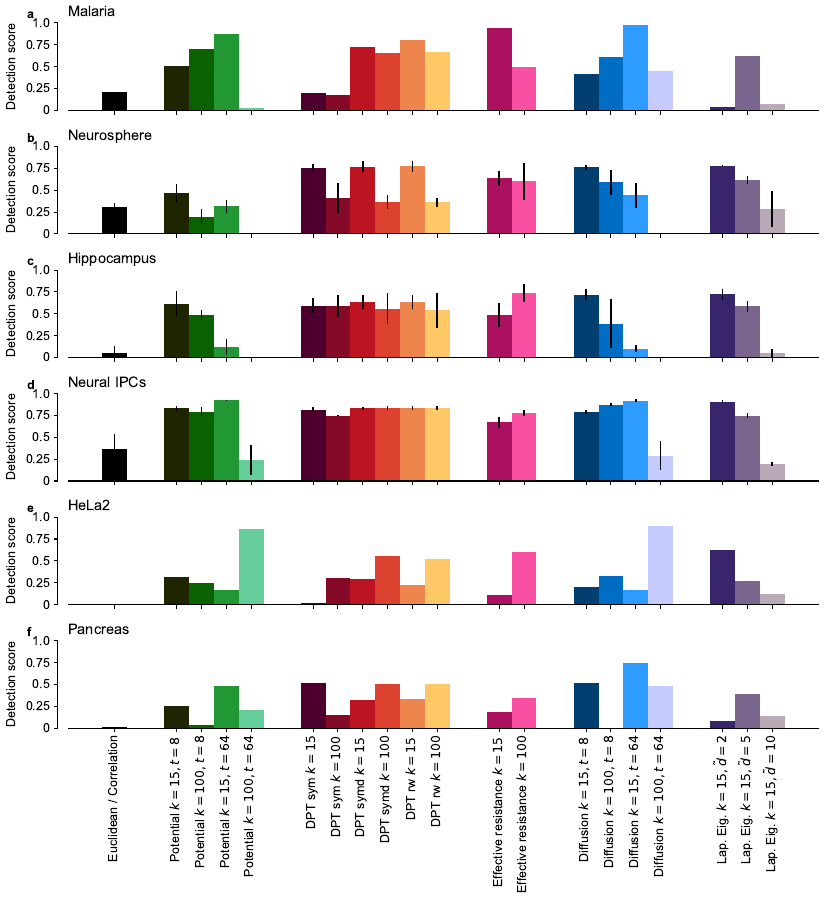}
    \caption{Comparison of all spectral methods on the single cell datasets. They all achieve similar performance.}
    \label{fig:scRNAseq_all_spectral}
\end{figure}

\section{Details on the distances used in our benchmark}
\label{app:dist}

Let $x_1,\dots, x_n \in \mathbb{R}^d$. We denote pairwise Euclidean distances by $d_{ij}= \|x_i - x_j\|$, the $k$ nearest neighbors of $x_i$ in increasing distance by $x_{i_1}, \dots, x_{i_k}$, and the set containing them by $N_i$. Many distances rely on the symmetric $k$-nearest-neighbor (\sknn) graph. This graph contains edge $ij$ if $x_i$ is among the $k$ nearest neighbors of $x_j$ or vice versa.

\paragraph{Fermat distances}
For $p\geq 1$, the Fermat distance is defined as 
\begin{equation}
    d^{F}_p(ij) = \inf_\pi\bigg(\sum_{(uv)\in \pi} d_{uv}^p\bigg),
\end{equation}
where the infimum is taken over all finite paths from $x_i$ to $x_j$ in the complete graph with edge weights $d_{ij}^p$. As a speed-up, \citet{fernandez2023intrinsic} suggested to compute the shortest paths only on the $k$NN graph, but for our sample sizes we could perform the calculation on the complete graph. For $p=1$ this reduces to normal Euclidean distances due to the triangle inequality. We used $p\in\{2,3,5,7\}$.

\paragraph{DTM distances}
The DTM distances depend on three hyperparameters: the number of nearest neighbors $k$, one hyperparameter controlling the distance to measure $p$, and finally a hyperparameter $\xi$ controlling the combination of DTM and Euclidean distance. The DTM value for each point is given by
\begin{equation}
    \text{dtm}_i = \begin{cases}
        \sqrt[\uproot{5} p]{\sum_{\kappa=1}^k  \| x_i - x_{i_\kappa}\|^{p}/k} &\text{ if } p < \infty \\
        \| x_i - x_{i_k}\| &\text{ else}.
    \end{cases}
\end{equation}
These values are combined with pairwise Euclidean distances to give pairwise DTM distances:
\begin{equation}
    d^{\text{DTM}}_{k,p,\xi}(ij) = \begin{cases}
        \max(\text{dtm}_i, \text{dtm}_j) &\text{ if } \| x_i - x_j\|\leq \sqrt[\uproot{5} \xi]{|\text{dtm}_i^{\xi} - \text{dtm}_j^{\xi}|}\\
        \theta &\text{ else},
    \end{cases}
\end{equation}
where $\theta$ is the only positive root of $\sqrt[\uproot{5}\xi]{\theta^{\xi} - \text{dtm}_i^{\xi}} + \sqrt[\uproot{5}\xi]{\theta^{\xi} - \text{dtm}_j^{\xi}} = d_{ij}$. We only considered the values $\xi \in \{1, 2, \infty\}$, for which the there are closed-form solutions:
\begin{equation}
    \theta = \begin{cases}
        (\text{dtm}_i + \text{dtm}_j + d_{ij})/2 &\text{ if } \xi = 1 \\
        \left. \sqrt{\big((\text{dtm}_i+\text{dtm}_j)^2 + d_{ij}^2\big)\cdot\big((\text{dtm}_i-\text{dtm}_j)^2 + d_{ij}^2\big)} \middle/ (2d_{ij})\right. &\text{ if } \xi = 2 \\
        \max(\text{dtm}_i, \text{dtm}_j, d_{ij}/2) &\text{ if } \xi = \infty.
    \end{cases}
\end{equation}
We used $k\in \{4, 15, 100\}$, $p \in \{2, \infty\}$, and $\xi\in\{1, 2,\infty\}$.

The original exposition of DTM-based filtrations~\cite{anai2020dtm} only considered the setting $p=2$, while DTM has been defined for arbitrary $p\geq 1$~\cite{chazal2021introduction}. We explore an additional value, $p=\infty$, in order to possibly strengthen DTM. Indeed, in several experiments it outperformed the $p=2$ setting.

Moreover, \citet{anai2020dtm} actually used a small variant of the Vietoris-Rips complex on the above distance $d^{\text{DTM}}_{ij}(k,p,\xi)$: They only included point $x_i$ in the filtered complex once the filtration value exceeds $\text{dtm}_i$. This, however, only affects the $0$-th homology, which we do not consider in our experiments. 

\paragraph{Core distance} 
The core distance is similar to the DTM distance with $\xi=\infty$ and $p=\infty$ and is given by
\begin{equation}
    d^{\text{core}}_k(ij) = \max(d_{ij}, \| x_i - x_{i_k}\|, \| x_j - x_{j_k}\|).
\end{equation}
We used $k\in\{15, 100\}$.

\paragraph{$t$-SNE graph affinities}
The $t$-SNE affinities are given by
\begin{equation}
    p_{ij}= \frac{p_{i|j}+ p_{j|i}}{2n}, \quad
    p_{j|i} = \frac{\nu_{j|i}}{\sum_{k\neq i} \nu_{k|i}}, \quad
    \nu_{j|i} = \begin{cases}
    \exp\left(\| x_i - x_j \|^2/(2\sigma_i^2)\right)  &\text{ if } x_j \in N_i  \\
        0  & \text{ else,}   \end{cases}
\end{equation}
where $\sigma_i$ is selected such that the distribution $p_{j|i}$ has pre-specified perplexity $\rho$. Standard implementations of $t$-SNE use $k=3\rho$. We transformed $t$-SNE affinities into pairwise distances by taking the negative logarithm. Pairs $x_i$ and $x_j$ with $p_{ij} = 0$ (i.e. not in the $k$NN graph) get distance $\infty$. We used $\rho\in\{30, 200, 333\}$.

\paragraph{UMAP graph affinities}
The UMAP affinities are given by 
\begin{equation}
    \mu_{ij} = \mu_{i|j} + \mu_{j| i} - \mu_{i| j}\mu_{j| i}, \quad
    \mu_{j|i} = 
    \begin{cases}
        \exp\big(-(d_{ij} - \mu_i)/\sigma_i\big) &\text{ for } j \in \{i_1, \dots, i_k\}\\
        0 &\text{ else,}
    \end{cases}
\end{equation}
where $\mu_i = \|x_i- x_{i_1}\|$ is the distance between $x_i$ and its nearest non-identical neighbor. The scale parameter $\sigma_i$ is selected such that 
\begin{equation}
    \sum_{\kappa=1}^k \exp\Big(-\big(d(x_i, x_{i_\kappa}) - \mu_i\big) / \sigma_i\Big) = \log_2(k).
\end{equation}
As above, to convert these affinities into distances, we take the negative logarithm and handle zero similarities as for the $t$-SNE case. We used $k\in\{100, 999\}$; $k=15$ resulted in memory overflow on one of the void-containing datasets.

\citet{gardner2022toroidal} and \citet{hermansen2024uncovering} first used these distances, but omitted $\mu_i$, which we included to completely reproduce UMAP's affinities. 

Note that distances derived from UMAP and $t$-SNE affinities are not guaranteed to obey the triangle inequality.

\paragraph{Geodesic distances}
We computed the shortest path distances between all pairs of nodes in the \sknn{} graph with edges weighted by their Euclidean distances. We used the Python function \texttt{scipy.sparse.csgraph.shortest\_path}. We used $k=\{15, 100\}$.

\paragraph{UMAP embedding}
We computed the UMAP embeddings in $2$ embedding dimensions using $750$ optimization epochs, \texttt{min\_dist} of $0.1$, exactly computed $k$ nearest neighbors, and PCA initialization. Then we used Euclidean distances between the embedding points. We used UMAP commit \texttt{a7606f2}. We used $k\in\{15, 100, 999\}$.

\paragraph{$t$-SNE embedding}
We computed the $t$-SNE embeddings in $2$ embedding dimensions using openTSNE~\citep{polivcar2024opentsne} with default parameters, but providing manually computed affinities. For that we used standard Gaussian affinities on the \sknn{} graph with $k=3\rho$. Then we used the Euclidean distances between the embedding points. We used perplexity $\rho\in\{8, 30, 333\}$.

For UMAP and $t$-SNE affinities as well as for UMAP and $t$-SNE embeddings we computed the \sknn{} graph with PyKeOps~\citep{pykeops} instead of using the default approximate methods. The UMAP and $t$-SNE affinities (without negative logarithm) were used by the corresponding embedding methods.

\paragraph{Effective resistance}
We computed the effective resistance on the \sknn{} graph. Following the analogy with resistances in an electric circuit, if the \sknn{} graph is disconnected, we computed the effective resistance separately in each connected component and set resistances between components to $\infty$. The uncorrected resistances were computed via the pseudoinverse of the unnormalized graph Laplacian~\citep{fouss2007random}
\begin{equation}
    \tilde{d}^\text{eff}_{ij} = l_{ii}^\dag - 2l_{ij}^\dag + l_{jj}^\dag,
\end{equation}
where $l_{ij}^\dag$ is the $ij$-th entry of the pseudoinverse $L^\dag$ of the unnormalized \sknn{} graph Laplacian $L = D -A$. The pseudoinverse inverts all non-zero eigenvalues. Denote the eigenvalue decomposition of $L$ by $L = V \Lambda V^T$, where $V=(v_1, \dots, v_n)^T$ is the matrix of eigenvectors of $L$ and $\Lambda$ is the diagonal matrix of their eigenvalues $\lambda_1, \dots, \lambda_n$. Then $L^\dag = V \Lambda^\dag V^T$, where $\Lambda^\dag$ is the diagonal matrix with entry $1/\lambda_i$ if $\lambda_i>0$ and $0$ otherwise. We can use this to derive a similar coordinate expression as in Eq.~\ref{eq:eff_res_naive}, but based on the unnormalized graph Laplacian (\citep[4.B]{fouss2007random}). Let $e_i$ be the $i$-th standard basis vector
\begin{align}
    \tilde{d}^\text{eff}_{ij} &= l_{ii}^\dag - 2l_{ij}^\dag + l_{jj}^\dag\nonumber\\
    &=(e_i - e_j)^T L^\dag (e_i- e_j) \nonumber \\
    &= (e_i - e_j)^T V \Lambda^\dag V^T (e_i- e_j) \nonumber \\
     &= \left(\sqrt{\Lambda^\dag} V^T e_i - \sqrt{\Lambda^\dag} V^Te_j\right)^T \left(\sqrt{\Lambda^\dag} V^T e_i - \sqrt{\Lambda^\dag} V^Te_j\right)\nonumber \\
    &= \| \hat{e}_i^{\text{eff}} - \hat{e}_j^{\text{eff}}\|^2 \text{ where}\nonumber \\
    \hat{e}_i^{\text{eff}}&= \left(\frac{v_{2,i}}{\sqrt{\lambda_i}}, \dots, \frac{v_{n,i}}{\sqrt{\lambda_n}}\right).
\end{align}

For the corrected version, we used
\begin{equation}
    d^\text{eff}_{ij} = \tilde{d}^\text{eff}_{ij} - \frac{1}{d_i} - \frac{1}{d_j} + 2 \frac{a_{ij}}{d_id_j} - \frac{a_{ii}}{d_i^2} - \frac{a_{jj}}{d_j^2}.
\end{equation}
For the weighted version of effective resistance, each edge in the \sknn{} graph was weighted by the inverse of the Euclidean distance. We experimented with the weighted and unweighted versions, but only reported the unweighted version in the paper as the difference was always minor. We also experimented with the unweighted and uncorrected version and saw that correcting is crucial for high noise levels (Figure~\ref{fig:eff_res_comparison}). We used $k~\in\{15, 100\}$. 

Both forms of the effective resistance can be written as squared distances between certain embedding points \eqref{eq:eff_res_naive}, \eqref{eq:eff_res_corr}. Nevertheless, the uncorrected effective resistance is a proper metric~\citep[Corollary 2.4.]{gobel1974random}. The corrected version in general is not a proper metric (Proposition~\ref{prop:corr_eff_res_not_metric}).

\paragraph{Diffusion distance}
We computed the diffusion distances on the unweighted \sknn{} graph directly by equation~\eqref{eq:diff}, i.e., 
\begin{equation}\label{eq:diff_app}
    d_t^{\text{diff}}(i,j) = \sqrt{\text{vol}(G)}\| (P^t_{i, :} - P^t_{j,:})D^{-\frac{1}{2}}\|.
\end{equation}
Note that our \sknn{} graphs do not contain self-loops. We used $k\in\{15, 100\}$ and $t\in\{8, 64\}$.

It is clear from the above definition of the diffusion distance that it satisfies the triangle inequality, but it can fail to be positive on distinct points (Proposition~\ref{prop:corr_eff_res_not_metric}).

\begin{figure}[tb]
    \centering
    \includegraphics[width=\linewidth]{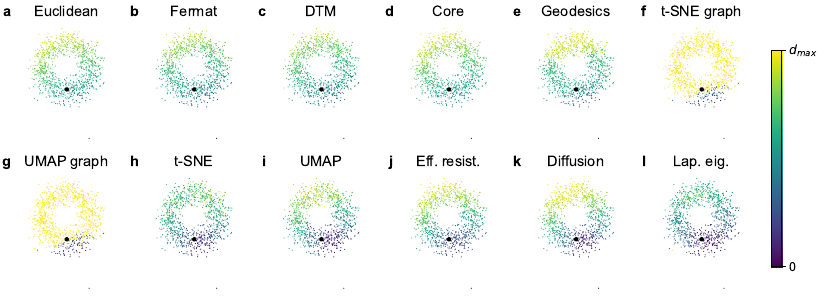}
    \caption{Visualization of all distances on the noisy circle in $\mathbb R^{50}$ with $\sigma=0.25$. All scatter plots are the 2D PCA of the 50D dataset. The colors indicate the distance to the highlighted point.}
    \label{fig:dists_pca}
\end{figure}

\begin{figure}[tb]
    \centering
    \includegraphics[width=\linewidth]{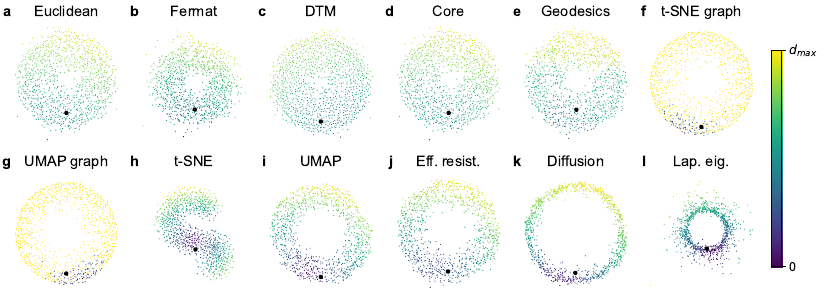}
    \caption{Visualization of all distances on the noisy circle in $\mathbb R^{50}$ with $\sigma=0.25$. The scatter plots are 2D multidimensional scaling embeddings using the respective distances. The colors indicate the distance to the highlighted point. For this random seed the $t$-SNE embedding tore the circle apart.}
    \label{fig:dists_mds}
\end{figure}

\paragraph{Laplacian Eigenmaps}
For an \sknn{} graph with $K$ connected components, we computed the $K+\tilde{d}$ eigenvectors $u_{1}, \dots, u_{K+\tilde{d}}$ of the normalized graph Laplacian $\lsym$ of the \sknn{} graph and discarded the first $K$ eigenvectors $u_{1},\dots, u_{K}$, which are coding for the connected components. Then we computed the Euclidean distances between the embedding vectors \mbox{$e^{\text{LE}}_i = (u_{K+1,i},\ldots,u_{(K+\tilde{d}), i})$.} We used $k=15$ and embedding dimensions $\tilde{d}\in\{2, 5, 10\}$. 

Alternatively, one can compute Laplacian Eigenmaps using the un-normalized graph Laplacian $L$. We tried this normalization for $\tilde{d}=2$ but obtained very similar embeddings.

\paragraph{Diffusion pseudotime} Diffusion pseudotime~\citep{haghverdi2016diffusion} also integrates the different time scales of the diffusion distance, but on the level of the transition matrices, rather than on the level of the distances themselves (Proposition~\ref{cor:diff_corr_eff_res}). There are multiple variants of diffusion pseudotime. They are all computed as  $ d^\textnormal{dpt}_{ij} = \| e_i^\textnormal{dpt} - e_j^\textnormal{dpt}\|$ for 
    \begin{align}
        e^\textnormal{dpt}_i = \left(\frac{1-\mu_2}{\mu_2} v_{2,i}, \dots,\frac{1-\mu_n}{\mu_n} v_{n,i} \right).
    \end{align}
The difference is the $v_1, \dots, v_n$'s. In the original publication~\citep{haghverdi2016diffusion}, they were the normalized eigenvectors of the random walker graph Laplacian $D^{-\frac{1}{2}}\lsym D^\frac{1}{2}$. These are given by $v_i = D^{-\frac{1}{2}}u_i  / \|D^{-\frac{1}{2}}u_i\|$. \citet{wolf2019paga} introduced a version using the eigenvectors of the symmetric graph Laplacian $\lsym$, so that $v_i= u_i$. Closest to the corrected resistance distance is the case where $v_i =D^{-\frac{1}{2}}u_i$ are non-normalized. We tested all three versions, to which we refer as ``rw'', ``sym'' and ``symd'' version. We used $k\in\{15, 100\}$ nearest neighbors.

\paragraph{Potential distance} The potential distance underlies the visualization method PHATE~\citep{moon2019visualizing} and is closely related to the diffusion distance. It is defined by 
$$d_t^{\text{pot}}(i, j) = \| \log(P^t_{i, :}) - \log(P^t_{j,:})\|,$$
where the logarithm is applied element-wise. We used $k\in\{15, 100\}$ and $t\in\{8, 64\}$.

For all methods, we replaced infinite distances with twice the maximal finite distance to be able to compute our hole detection scores.

We illustrate the distances used in the main benchmark in Figures~\ref{fig:dists_pca} and~\ref{fig:dists_mds}.

\section{Datasets}
\label{app:datasets}
\subsection{Synthetic datasets}
The synthetic, noiseless datasets with $n=1\,000$ points each are depicted in Figure~\ref{fig:toy_data}. Noised versions of the circle for ambient dimensions $d=2, 50$ are depicted in Figure~\ref{fig:circle_with_noise}.

\begin{figure}[t]
    \centering
    \includegraphics[width=\linewidth]{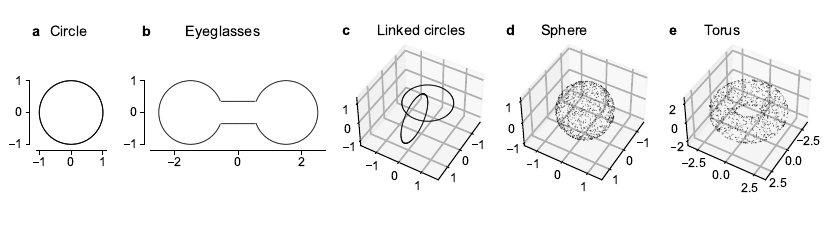}
    \caption{Synthetic, noiseless datasets with $n=1\,000$ points each.}
    \label{fig:toy_data}
\end{figure}

\begin{figure}[tb]
    \centering
    \includegraphics[width=\linewidth]{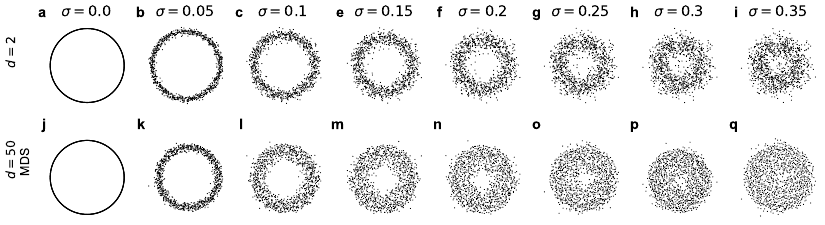}
    \caption{Circle with Gaussian noise of different standard deviation $\sigma$. \textbf{a\,--\,i.} Original data in ambient dimension $d=2$. \textbf{j\,--\,q.} Multidimensional scaling of the Euclidean distance of the data in ambient dimension $d=50$.}
    \label{fig:circle_with_noise}
\end{figure}

\paragraph{Circle}
The circle dataset consists of $n$ points equidistantly spaced along a circle of radius $r=1$.

\paragraph{Linked circles}
The linked circles dataset consists of two circle datasets of $n/2$ points each, arranged such that each circle perpendicularly intersects the plane spanned by the other and goes through the other's center.

\paragraph{Eyeglasses}
The eyeglasses dataset consists of four parts: Two circle segments of arclength $\pi + 2.4$ and radius $r=1$, centered $3$ units apart with the gaps facing each other. The third and fourth part are two straight line segments of length $1.06$, separated by $0.7$ units linking up the two circle segments. The circle segments consist of $0.425 n$ equidistantly distributed points each and the line segments consist of $0.075n$ equispaced points each. As the length scale of this dataset is dominated by the bottleneck between the two line segments, we only considered noise levels $\sigma \in [0, 0.15]$ for this dataset, as at this point the bottleneck essentially merges in $\mathbb R^2$.

\paragraph{Sphere}
The sphere dataset consists of $n$ points sampled uniformly from a sphere $S^2$ with radius $r=1$.

\paragraph{Torus}
The torus dataset consists of $n$ points sampled uniformly from a torus. The radius of the torus' tube was $r=1$ and the radius of the center of the tube was $R=2$. Note that we do not sample the points to have uniform angle distribution along the tube's and the tube center's circle, but uniform on the surface of the torus.

\paragraph{High-dimensional noise}
We mapped each dataset to $\mathbb{R}^d$ for $d\in[2, 50]$ using a random matrix $\mathbf V$ of size $d\times 2$ or $d\times 3$ with orthonormal columns, and then added isotropic Gaussian noise sampled from $\mathcal{N}(\mathbf 0, \sigma^2 \mathbf I_d)$ for $\sigma \in [0, 0.35]$. 

The orthogonal embedding in $\mathbb R^d$ does not change the shape of the data. The procedure is equivalent to adding $d-2$ or $d-3$ zero dimensions and then randomly rotating the resulting dataset in $\mathbb R^d$.

\subsection{Single-cell datasets}
We depict 2D embeddings of all single-cell datasets in Figure~\ref{fig:sc_embd}.

\paragraph{Malaria}
The Malaria dataset~\citep{howick2019malaria} consists of gene expression measurement of $5\,156$ genes obtained with the modified SmartSeq2 approach of~\citet{reid2018single} in \mbox{$n=1\,787$} cells from the entire life cycle of \textit{Plasmodium berghei}. The resulting transcripts were pre-processed with the trimmed mean of M-values method~\citep{robinson2010scaling}. We obtained the pre-processed data from \url{https://github.com/vhowick/MalariaCellAtlas/raw/v1.0/Expression_Matrices/Smartseq2/SS2_tmmlogcounts.csv.zip}. The data is licensed under the GNU GPLv3 licence.

The UMAP embedding shown in Figure~\ref{fig:malaria} follows the authors' setup and uses correlation distance as input metric, $k=10$ nearest neighbors, and a \texttt{min\_dist} of $1$ and \texttt{spread} of $2$. Note that when computing persistent homology with UMAP-related distances, we used our normal UMAP hyperparameters and never changed \texttt{min\_dist} or \texttt{spread}.

\paragraph{Neural IPCs}
The Neural IPC dataset~\citep{braun2023comprehensive} consists of gene expressions of \mbox{$n=26\,625$} neural IPCs from the developing human cortex. scVI~\citep{lopez2018deep} was used to integrate cells with different ages and donors based on the $700$ most highly variable genes, resulting in a $d=10$ dimensional embedding. \citet{braun2023comprehensive} shared this representation with us for a superset of $297\,927$ telencephalic exitatory cells and allowed us to share it with this paper (MIT License). We limited our analysis to the neural IPCs because they formed a particularly prominent cell cycle.

\paragraph{Neurosphere}
The Neurosphere dataset~\citep{zheng2022universal} consists of gene expressions for $n=12\,805$ cells from the mouse neurosphere. After quality control, the data was library size normalized and $\log_2$ transformed. Seurat was used to integrate different samples based on the first $30$ PCs of the top $2\,000$ highly variable genes, resulting in a $12\,805 \times 2\,000$ matrix of $\log_2$ transformed expressions. These were subsetted to the genes in the gene ontology (GO) term cell cycle (GO:0007049). The $500$ most highly variable genes were selected and a PCA was computed to $d=20$. The GO PCA representation was downloaded from~\url{https://zenodo.org/record/5519841/files/neurosphere.qs}. It is licensed under CC BY 4.0.

\paragraph{Hippocampus}
The Hippocampus dataset~\citep{zheng2022universal} consists of gene expressions for $n=9\,188$ mouse hippocampal NPCs. The pre-processing was the same as for the Neurosphere dataset. The GO PCA representation was downloaded from~\url{https://zenodo.org/record/5519841/files/hipp.qs}. It is licensed under CC BY 4.0.

\begin{figure}[t]
    \centering
    \includegraphics[width=\linewidth]{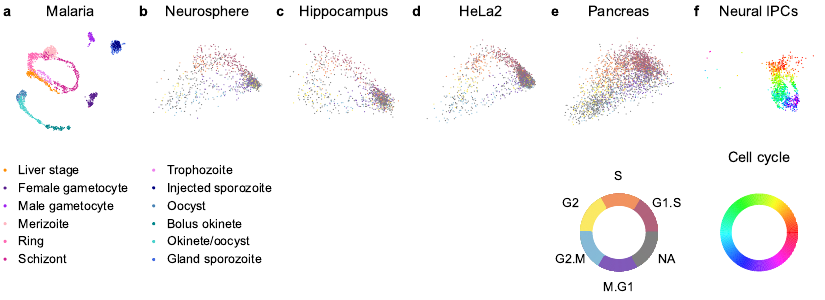}
    \caption{2D embeddings of all six single-cell datasets. \textbf{a, f.} UMAP embeddings of the Malaria~\citep{howick2019malaria} and the Neural IPC datasets~\citep{braun2023comprehensive}. We recomputed the embedding for the Malaria dataset using UMAP hyperparameters provided in the original publication, and subsetted an author-provided UMAP of a superset of telencephalic exitatory cells to the Neural IPC. The text legend refers to Malaria cell types. \textbf{b\,--\,e.} 2D linear projection constructed to bring out the cell cycle (`tricycle embedding')~\citep{zheng2022universal} of the Neurosphere, Hippocampus, HeLa2, and Pancreas datasets. We used the projection coordinates provided by~\citet{zheng2022universal}.}
    \label{fig:sc_embd}
\end{figure}
\paragraph{HeLa2}
The HeLa2 dataset~\citep{schwabe2020transcriptome, zheng2022universal} consists of gene expressions for $2\,463$ cells from a human cell line derived from cervical cancer. After quality control, the data was library size normalized and $\log_2$ transformed. From here the GO PCA computation was the same as for the neurosphere dataset. The GO PCA representation was downloaded from~\url{https://zenodo.org/record/5519841/files/HeLa2.qs}. It is licensed under CC BY 4.0.

\paragraph{Pancreas}
The Pancreas dataset~\citep{bastidas2019comprehensive, zheng2022universal} consists of gene expressions for $3\,559$ cells from the mouse endocrine pancreas. After quality control, the data was library size normalized and $\log_2$ transformed. From here the GO PCA computation was the same as for the neurosphere dataset. The GO PCA representation was downloaded from~\url{https://zenodo.org/record/5519841/files/endo.qs}. It is licensed under CC BY 4.0.

\section{Hyperparameter selection}
\label{sec:hyperparameters}
For each of the datasets and hole dimensions, we showed the result with the best hyperparameter setting. For the synthetic experiments, this meant the highest area under the hole detection score curve, while for the single-cell datasets it meant the highest loop detection score. Here, we give details of the selected hyperparamters.

For Figure~\ref{fig:1} we used effective resistance with $k=100$ as in Figure~\ref{fig:circle-benchmark}.

For Figure~\ref{fig:circle-benchmark} we specified the selected hyperparameters directly in the figure. For the density-based methods, they were $p=3$ for Fermat distances, $k=4, p=2, \xi=\infty$ for DTM, and $k=15$ for the core distance. For the graph-based methods, they were $k=100$ for the geodesics, $k=100$ for the UMAP graph affinities, and $\rho=30$ for $t$-SNE graph affinities. The embedding-based methods used $k=15$ for UMAP and $\rho=30$ for $t$-SNE. Finally, as spectral methods, we selected effective resistance with $k=100$, diffusion distance with $k=100, t=8$ and Laplacian Eigenmaps with $k=15, d=2$. 

The hyperparameters for Figure~\ref{fig:datasets-benchmark} are given in Table~\ref{tab:hyperparam_datasets}.

In Figure~\ref{fig:circle-dims} we specified the hyperparameters used. They were the same as for Figure~\ref{fig:circle-benchmark} save for the diffusion distance, for which we used $k=100, t=8$. This setting had better performance for $d=2$ and only marginally lower performance than $k=15, t=8$ in higher dimensionalities.

For Figure~\ref{fig:malaria}, we selected DTM with $k=15, p=\infty, \xi=\infty$, effective resistance with $k=15$ and diffusion distance with $k=15, t=64$. They are the same for the Malaria dataset in Figure~\ref{fig:scrnaseq}.

The selected hyperparameters for Figure~\ref{fig:scrnaseq} can be found in Table~\ref{tab:hyperparameters_scrnaseq}.

The hyperparameters for Figures~\ref{fig:dists_pca}
and~\ref{fig:dists_mds} are the same as those used in Figure~\ref{fig:circle-benchmark}. The hyperparameters for Figure~\ref{fig:widest_gap} are given in Table~\ref{tab:hyperparam_widest} and those for Figure~\ref{fig:outliers} in Table~\ref{tab:hyperparam_outliers}. All other supplementary figures either do not depend on hyperparameters or detail them directly in the figure or the caption.
\begin{table}[t]
    \centering
    \caption{The optimal hyperparameters that were  selected in Figure~\ref{fig:datasets-benchmark}. For torus and sphere, we consider the case of loop detection ($H_1$) and void detection ($H_2$) separately.}
    \begin{tabular}{lllll}
        \toprule
        Dataset & Fermat & DTM & Eff. res. & Diffusion\\
        \midrule
        Circle         & $p=3$ & $k=4, p=2, \xi=1$ & $k=100$ & $k=15, t=8$ \\
        Eyeglasses     & $p=7$ & $k=100, p=2, \xi=1$ & $k=15$  & $k=15, t=64$ \\
        Linked circles & $p=7$ & $k=15, p=\infty, \xi=1$ & $k=15$  & $k=15, t=8$ \\
        Torus $H_1$       & $p=2$ & $k=4, p=2, \xi=\infty$ & $k=100$ & $k=15, t=8$ \\
        Sphere $H_1$      & $p=2$ & $k=100, p=2,\xi=1$ & $k=15$  & $k=100, t=64$\\
        Torus $H_2$       & $p=2$ & $k=4, p=2, \xi=\infty$ & $k=100$ & $k=15, t=8$\\
        Sphere $H_2$      & $p=2$ & $k=4, p=2, \xi=1$ & $k=100$ & $k=100, t=8$ \\
        \bottomrule
    \end{tabular}
    \label{tab:hyperparam_datasets}
\end{table}

\begin{table}[t]
    \centering
    \caption{The optimal hyperparameters that were selected in Figure~\ref{fig:scrnaseq}. For DTM we report the best setting without thresholding (because none of the DTM runs passed our birth/death thresholding, so all $s_m$ scores for all parameter combinations are zero).}
    \begin{tabular}{lllllllll}
    \toprule
    Dataset & Fermat & DTM & $t$-SNE & UMAP & Eff. res. & Diffusion & Lap. Eig.\\ 
    \midrule
    Malaria     & $p=2$ & $k=15$       & $\rho=8$  & $k=15$  & $k=15$  & $k=15$  & $\tilde{d}=5$ \\
                &       & $p=\infty$   &           &         &         & $t=64$  &        \\
                &       & $\xi=\infty$ &           &         &         &         &        \\
    Neurosphere & $p=2$ & $k=100$      & $\rho=30$ & $k=999$ & $k=15$  & $k=15$  & $\tilde{d}=2$ \\
                &       & $p=2$        &           &         &         & $t=8$   &        \\
                &       & $\xi=2$      &           &         &         &         &        \\
    Hippocampus & $p=7$ & $k=100$      & $\rho=8$  & $k=15$  & $k=100$ & $k=15$  & $\tilde{d}=2$\\
                &       & $p=2$        &           &         &         & $t=8$   &        \\
                &       & $\xi=2$      &           &         &         &         &        \\
    Neural IPC  & $p=2$ & $k=4$        & $\rho=30$ & $k=15$  & $k=100$ & $k=15$  & $\tilde{d}=2$ \\
                &       & $p=\infty$   &           &         &         & $t=64$   &        \\
                &       & $\xi=\infty$ &           &         &         &         &        \\
    HeLa2       & $p=3$ & $k=4$        & $\rho=30$ & $k=100$ & $k=100$ & $k=100$ & $\tilde{d}=2$ \\
                &       & $p=2$        &           &         &         & $t=64$   &        \\
                &       & $\xi=\infty$ &           &         &         &         &        \\
    Pancreas    & $p=7$ & $k=100$      & $\rho=8$  & $k=100$ & $k=4$   & $k=15$  & $\tilde{d}=5$\\
                &       & $p=2$        &           &         &         & $t=64$  &        \\
                &       & $\xi=\infty$ &           &         &         &         &        \\
    \bottomrule
    \end{tabular}
    \label{tab:hyperparameters_scrnaseq}
\end{table}

\begin{table}[t]
    \centering
    \caption{The optimal hyperparameters that were  selected in Figure~\ref{fig:widest_gap}.}
    \begin{tabular}{lllll}
        \toprule
        Dataset & Fermat & DTM & Eff. res. & Diffusion\\
        \midrule
        Circle         & $p=2$ & $k=4, p=2, \xi=1$ & $k=100$ & $k=15, t=8$ \\
        Eyeglasses     & $p=7$ & $k=4, p=2, \xi=1$ & $k=15$  & $k=15, t=8$ \\
        Linked circles & $p=5$ & $k=4, p=2, \xi=1$ & $k=100$  & $k=15, t=8$ \\
        \bottomrule
    \end{tabular}
    \label{tab:hyperparam_widest}
\end{table}

\begin{table}[t]
    \centering
    \caption{The optimal hyperparameters that were  selected in Figure~\ref{fig:outliers}.}
    \begin{tabular}{llllll}
        \toprule
        Ambient & Number of & Fermat & DTM & Effective & Diffusion\\
        dimension & outliers && &resistance &\\
        \midrule
        2 & 0        & $p=2$ & $k=100, p=\infty, \xi=2$ & $k=100$ & $k=100, t=64$ \\
        2 & 50       & $p=7$ & $k=100, p=\infty, \xi=2$ & $k=100$ & $k=100, t=64$ \\
        2 & 100      & $p=7$ & $k=100, p=\infty, \xi=2$ & $k=100$ & $k=100, t=64$ \\
        2 & 200      & $p=7$ & $k=100, p=\infty, \xi=2$ & $k=100$ & $k=100, t=64$ \\
        50 & 0       & $p=3$ & $k=4, p=2,\xi=1$ & $k=100$ & $k=100, t=8$\\
        50& 50       & $p=2$ & $k=4, p=2,\xi=1$ & $k=100$ & $k=15, t=8$\\
        50 & 100     & $p=2$ & $k=4, p=2,\xi=1$ & $k=100$ & $k=15, t=8$ \\
        50 & 200     & $p=2$ & $k=4, p=2,\xi=1$ & $k=100$ & $k=15, t=64$ \\
        \bottomrule
    \end{tabular}
    \label{tab:hyperparam_outliers}
\end{table}

\section{Hyperparameter sensitivity}\label{sec:hyperparam_sensitivity}

While all distances other than the Euclidean distance have hyperparameters and we show results for the best hyperparameter setting in most figures, hyperparameter selection does not pose a serious problem for diffusion distance and effective resistance. We only tuned hyperparameters very mildly for these methods (4 settings for diffusion distances, only 2 for effective resistance, as opposed to 24 for the competing method DTM). Moreover, we conduced a more fine-grained sensitivity analysis and found the performance of diffusion distances and effective resistances to be robust to the exact value of their hyperparameters. For this sensitivity analysis we computed the area under the noise-level / detection score curves for the circle, the interlinked circles and the eyeglasses dataset in $50$ ambient dimensions with $k=4, 15,30,45,60,75,90,105,120,135,150$ nearest neighbors and $t=2,4,8,16,32,64,128,256$ diffusion steps. For all but the most extreme hyperparameter values both effective resistance and diffusion distances strongly outperformed the Euclidean distance  (Figure~\ref{fig:hyperparam_sensitivity}). We observed a trade-off between the optimal number of neighbors $k$ and diffusion steps $t$ for the diffusion distance, since both control how fast the diffusion can progress across the dataset.

\begin{figure}
    \centering
    \includegraphics[width=\linewidth]{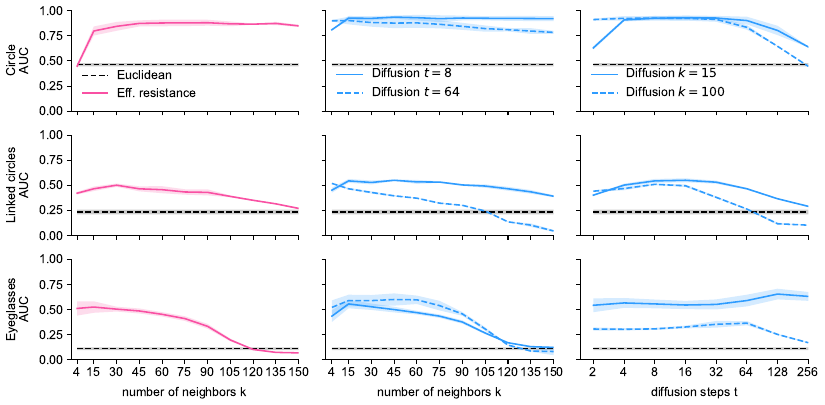}
    \caption{Diffusion distances and effective resistance are robust to their hyperparameters and outperform the Euclidean distance for most choices. We depict the area under the noise-level / detection score curve for the three 1D datasets circle, linked circles and eyeglasses in $\mathbb{R}^{50}$.}
    \label{fig:hyperparam_sensitivity}
\end{figure}

\section{Implementation details}
\label{app:implementation}

We computed persistent homology using the \texttt{ripser}~\citep{bauer2021ripser} project's \texttt{representative-cycles} branch at commit \texttt{140670f} to compute persistent homologies and representative cycles. We used coefficients in $\mathbb{Z} / 2\mathbb{Z}$. To compute $k$NN graphs, we used the PyKeops package~\citep{pykeops}. The rest of our implementation is in Python. Our code is available at \url{https://github.com/berenslab/eff-ph/tree/neurips2024}.

Our experiments were run on a machine with an Intel(R) Xeon(R) Gold 6226R CPU @ 2.90GHz with 64 kernels, 377GB memory, and an NVIDIA RTX A6000 GPU. The persistent homology computations only ever used a single kernel.

Our benchmark consisted of many individual experiments. We explored $47$ hyperparameter settings across all distances, computed results for $3$ random seeds and $29$ noise levels $\sigma$. In the synthetic benchmark, we computed only 1D persistent homology for $3$ datasets and both 1D and 2D persistent homology of $2$ more datasets. So the synthetic benchmark with ambient dimension $d=50$ alone consisted of $12\,267$ computations of 1D persistent homology and $8\, 178$ computations of both 1D and 2D persistent homology. 

The run time of persistent homology vastly dominated the time taken by the distance computation. The persistent homology run time depended most strongly on the sample size $n$, the dataset, and on the highest dimensionality of holes. The difference between distances was usually small. However, we observed that there were some outliers, depending on the noise level and the random seed, that had much longer run time. Overall, we found that methods that produce many pairwise distances of the same value (e.g., because of infinite distance in the graph affinities or maximum operations like for DTM with $p=\infty, \xi=\infty$) often had a much longer run time than other settings. We presume this was because equal distances led to many simplices being added to the complex at the same time. We give exemplary run times in Table~\ref{tab:run_times}. 

As a rough estimate for the total run time, we extrapolated the run times for the circle to all 1D persistent homology experiments for ambient dimension $d=50$ and the times for the sphere to all 2D experiments. In both cases we took the mean between the noiseless ($\sigma=0$) and highest noise ($\sigma=0.35$) setting in Table~\ref{tab:run_times}. This way, we estimated a total sequential run time of about $60$ days, but we parallelized the runs.

\begin{table}[t]
    \centering
    \caption{Exemplary run times in seconds.}
    \begin{tabular}{llllccc}
        \toprule
         Dataset& n&$\sigma$& Distance & Feature dim& Time distance [s] & Time PH [s]\\
         \midrule
         Circle &$1\,000$ & $0.0$ & Euclidean & $1$ & $0.013\pm 0.002$ & $12.3 \pm 0.4$\\
         Circle &$1\,000$& $0.0$ & Eff. res $k=100$ & $1$ & $0.17\pm 0.04$ & $12.0\pm0.2$\\
         Circle &$2\,000$& $0.0$  & Euclidean & $1$ & $0.09\pm0.04$ & $117\pm 9$\\
         Sphere &$1\,000$& $0.0$ & Euclidean & $1$ & $0.012\pm0.001$ & $1.31\pm0.06$\\
         Sphere &$1\,000$& $0.0$ & Euclidean & $2$ & $0.017\pm0.002$ & $4687\pm2501$\\
         Circle &$1\,000$ & $0.35$ & Euclidean & $1$ & $0.016\pm 0.001$ & $5 \pm 2$\\
         Sphere &$1\,000$ & $0.35$ & Euclidean & $2$ & $0.03\pm 0.02$ & $258 \pm 18$\\
         \bottomrule
    \end{tabular}

    \label{tab:run_times}
\end{table}

\section{Effect of outliers}
\label{sec:outliers}
Persistent homology with the Euclidean distance is known to be sensitive to outliers. Methods such as DTM were introduced to handle this issue. Here we show that spectral methods can also handle outliers well. Moreover, in high ambient dimensionality outliers are distributed over the large volume and hence are very sparse, making them less of a problem.

We experimented with the noisy circle with $n=1\,000$ points in ambient $\mathbb R^d$ for $d=2, 50$ and added $50$, $100$, or $200$ outlier points. These were sampled uniformly from axis-aligned cubes around the data in ambient space. The size of the cube was set just large enough that it contained the data even with the strongest added Gaussian noise. 

In low dimensionality, Euclidean distance suffered in the low Gaussian noise setting already when only $50$ outliers were added.
Adding $100$ or $200$ outliers severely lowered the detection score for Euclidean distance across the entire range of Gaussian noise strength. Fermat distances suffered from high random seed variability when adding outliers in low ambient dimension. Diffusion distance and effective resistance were much more outlier-resistant than the Euclidean distance and were only affected by $200$ outliers. Even then they performed better than the Euclidean distance without any outliers. DTM excelled in this setting, being completely insensitive to outliers and achieving top score for all noise levels (Figure~\ref{fig:outliers}a\,--\,d). 

The volume of the bounding box in $d=50$ ambient dimensions is much larger and thus the same number of outlier points are distributed much more sparsely. In particular, it is much less likely that an outlier happens to fall into the middle of the circle. As a result, even Euclidean and Fermat distances were very outlier-robust in $d=50$ ambient dimensions (Figure~\ref{fig:outliers}e\,--\,g). Similarly, DTM's performance did not change at all in the face of outliers. However, all three methods suffered strongly from the high-dimensional Gaussian noise. 

As diffusion distance and effective resistance in our implementation rely on the unweighted $k$NN graph, they were somewhat more susceptible to outliers. Performance of diffusion distance decreased steadily with the number of outliers in high ambient dimension. When outliers were present, it performed worse than the Euclidean distance in the low Gaussian noise setting, but much better in the high Gaussian noise setting, even for $200$ outliers. Effective resistance performed best overall, deteriorating only slightly in the low Gaussian noise setting when outliers were added. Both spectral methods clearly outperformed other methods in the high-Gaussian noise regime even in the presence of numerous outliers.

To sum up, effective resistance (and to a lesser extent diffusion distance) can handle both outliers and high-dimensional Gaussian noise, while other methods can handle at most one type of noise.

\begin{figure}
    \centering
    \includegraphics[width=\linewidth]{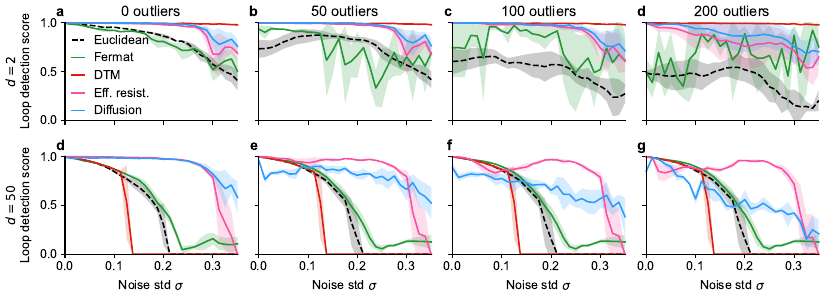}
    \caption{Loop detection performance of various methods on the noisy circle in the presence of outliers in low- and high-dimensional ambient space. Outliers were sampled uniformly from an axis-aligned cube around the data. \textbf{a\,--\,d.}~In low ambient dimension ($d=2$) adding outliers hurt the performance of the Euclidean and Fermat distances, but barely affected the performance of the spectral methods and not at all DTM's excellent performance. \textbf{e\,--\,g.}~In high ambient dimensionality ($d=50$) outliers did not further decrease the weak performance of non-spectral methods. Diffusion distance was somewhat outlier-sensitive, but could still detect the loop structure in the high Gaussian noise setting. Effective resistance performed best overall and was not very outlier-sensitive.
    }
    \label{fig:outliers}
\end{figure}

\section{UMAP with higher embedding dimension}
\label{sec:umap_high_embd_dim}
Following the original publication~\citep{mcinnes2018umap}, UMAP is typically used to embed data into two dimensions. This is an obvious issue for datasets sampled from manifolds which are not embeddable into two dimensions, such as the sphere or the torus. Therefore, we also experimented with the less common approach of higher embedding dimensionality ($3$, $5$, $10$) for UMAP. Save for the successful detection of the sphere's void with at least three embedding dimensions, we observed few consistent effects of the embedding dimension either on the toy (Figure~\ref{fig:highD_umap}) or on the real data (Figure~\ref{fig:scRNAseq_UMAPs}). Nevertheless, on the linked circles dataset, which is not embeddable into two dimensions, we saw a small improvement for UMAP when going beyond two embedding dimensions for the low noise setting.

Independent of the embedding dimension, UMAP struggled in the low-noise setting on the eyeglasses dataset for $k=15$. Moreover, we observed very poor performance both for loop and void detection on the torus. We believe the reasons may be UMAP's tendency to over-fragment the manifold~\citep{damrich2021umap} and UMAP's use of a heavy-tailed kernel. We visualized three-dimensional UMAP embeddings of the noiseless eyeglasses, sphere, and torus for $k=15$ in Figure~\ref{fig:umap_vis_dim_3}. UMAP left a gap in the embedding of the eyeglasses dataset, so that the true loop only had the second highest persistence. Short-cutting at the bottleneck yielded the most persistent loop. For the torus, the surface of the embedding was very fragmented, preventing the detection of any void. While the main loop of the torus was detected well, the second most persistent detected loop was already in the noise cloud as the fragmented surface of the embedding allowed for many fairly persistent loops. In a similar way, the surface of the sphere's embedding got fragmented, leading to many loops. Our score does not penalize this, because the detection score for $m=1$ loop is low as there is nearly no gap in persistence between the most and the second-most persistent loop. Higher levels of noise and also higher $k$ could overcome UMAP's over-fragmentation tendency for the eyeglasses dataset. This may be due to high-dimensional noise better matching UMAP's heavy-tailed kernel. However, it did not resolve the over-fragmentation  for the torus. 

The runtime of UMAP scales linearly with the embedding dimension, but typical $t$-SNE implementations scale exponentially. In fact, the implementation we used here, openTSNE~\citep{polivcar2024opentsne}, does not implement embedding dimensions higher than two, which is why we explored higher-dimensional embeddings only for UMAP.

\begin{figure}
    \centering
    \includegraphics[width=\linewidth]{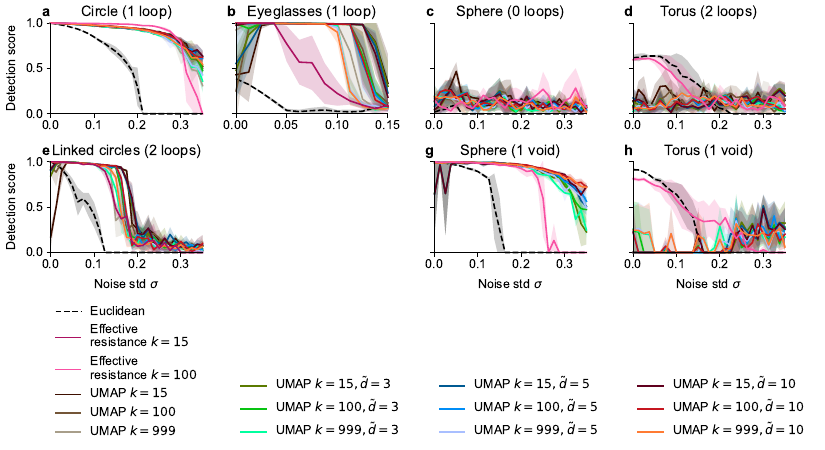}
    \caption{Hole detection scores in UMAP embeddings of the toy datasets in different embedding dimensions. Changing the embedding dimension only has a small effect.}
    \label{fig:highD_umap}
\end{figure}

\begin{figure}
    \centering
    \includegraphics[width=\linewidth]{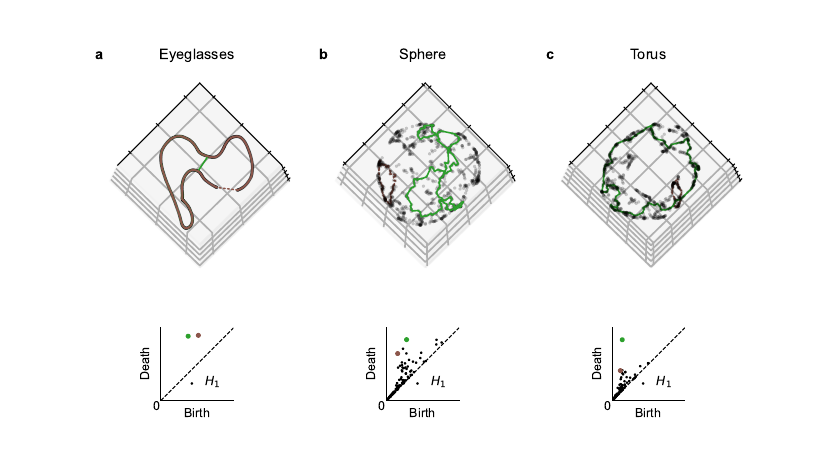}
    \caption{Exemplary UMAP embeddings with $k=15$ of the noiseless eyeglasses dataset, the sphere, and the torus into three dimensions. We show the persistence diagrams for loops and highlight the two most persistent loops and superimposed them on the embedding. We see strong over-fragmentation of the surfaces that challenges the loop detection and in the case of the torus the void detection (persistence diagram for the voids not shown).}
    \label{fig:umap_vis_dim_3}
\end{figure}

\begin{figure}
    \centering
    \includegraphics[width=0.97\linewidth]{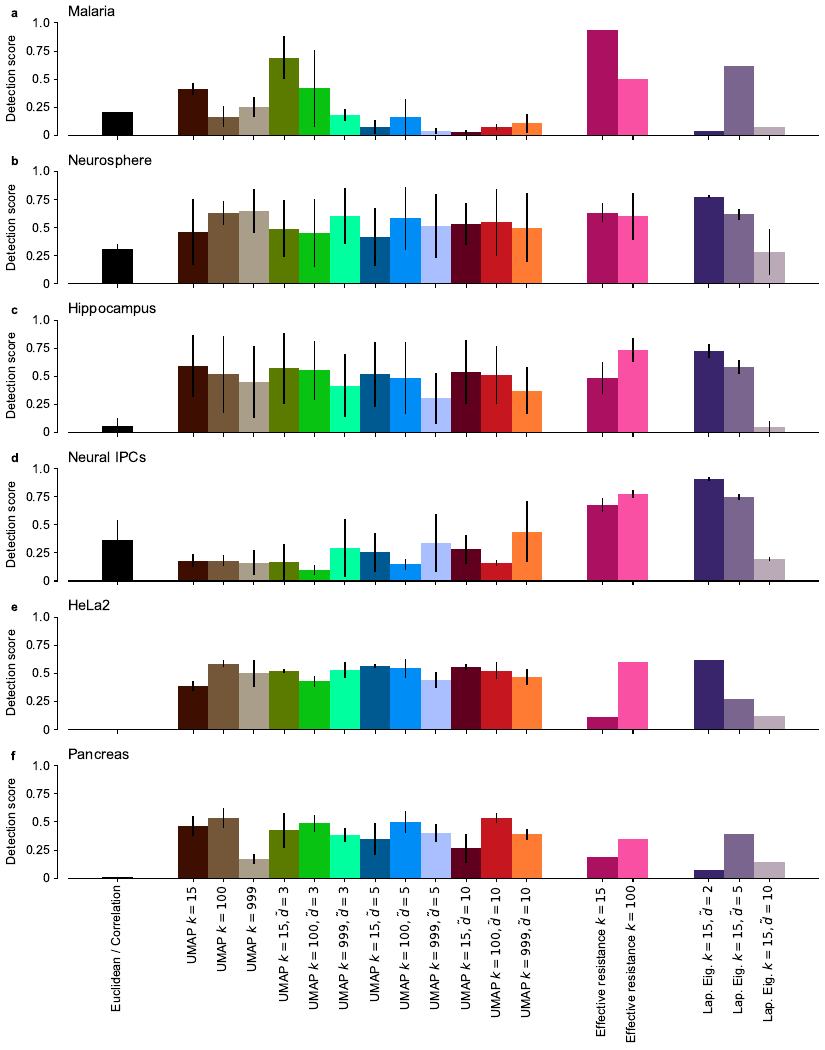}
    \caption{Detection scores for UMAP with different embedding dimensions on the single-cell datasets with some other methods for reference. Changing the embedding dimension did not have a consistent effect on the detection scores, while higher embedding dimension usually hurt for Laplacian Eigenmaps.}
    \label{fig:scRNAseq_UMAPs}
\end{figure}

\clearpage
\section{Additional figures}

\begin{figure}[h!]
    \centering
    \includegraphics[width=\linewidth]{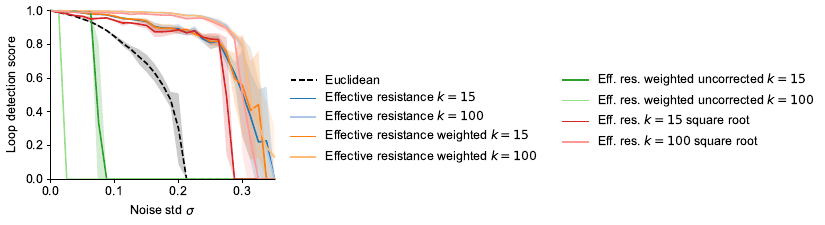}
    \caption{Loop detection score on noisy $S^1\subset \mathbb R^{50}$ for various versions of effective resistance. There was little difference between using the weighted $k$NN graph, unweighted $k$NN graph, and using the square root of effective resistance based on the unweighted $k$NN graph. The latter got filtered out for high noise levels. Using $k=100$ instead of $k=15$ helped only marginally in this dataset. The uncorrected (naive) version of effective resistance collapsed already at very small noise levels.}
    \label{fig:eff_res_comparison}
\end{figure}

\begin{figure}[h!]
    \centering
    \includegraphics[width=4in]{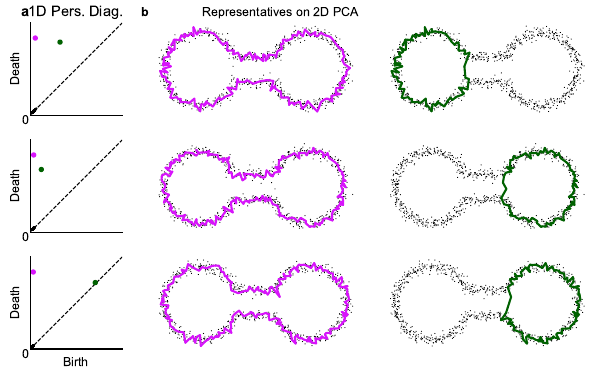}
    \caption{Illustration for the random seed variability of effective resistance with $k=15$ on the noisy eyeglasses dataset in $\mathbb R^{50}$ with $\sigma=0.075$. This refers to Figure~\ref{fig:datasets-benchmark}b. \textbf{a.} One-dimensional persistence diagrams for three random seeds. \textbf{b.} Representatives of the most two most persistent features superimposed on a 2D PCA of the dataset. These always corresponded to the full shape and one of the two circle segments. For the first two random seeds, some points are distorted in such a way that they form a bridge in the 2D PCA, while in the third there is not such bridge and the second most persistent feature is much less persistent. Note that this is just a 2D PCA, in particular, much of the noise in 50D is not visible. A similar explanation applies for the diffusion distance in  Figure~\ref{fig:datasets-benchmark}b.}
    \label{fig:eyeglasses_stability}
\end{figure}

\begin{figure}
    \centering
    \includegraphics[width=3.5in]{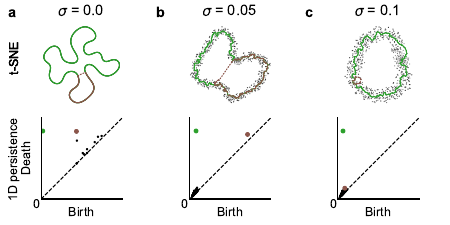}
    \caption{$t$-SNE embeddings with perplexity $\rho = 8$ and 1D persistence diagrams of the embedding for a circle in ambient $\mathbb R^{50}$ with Gaussian noise of low standard deviation $\sigma$. Perplexity $\rho=8$ is rather small, such that each embedding point only feels attraction to very few other points. In the noiseless setting this very sparse attraction is only among immediate neighbors along the circle. This makes the embedding to have spurious curves. For higher noise, the sparse attraction pattern is less regular and less local such that the spurious curves disappear. The more spurious curves the embedding has, the more high persistent features, given by bottlenecks in the curvy embedding, exist. This explains the dip for the $t$-SNE $\rho=8$ curve in  Figure~\ref{fig:circle_many_methods_d_50}g.}
    \label{fig:circle_tsne_embd}
\end{figure}

\begin{figure}
    \centering
    \includegraphics[width=\linewidth]{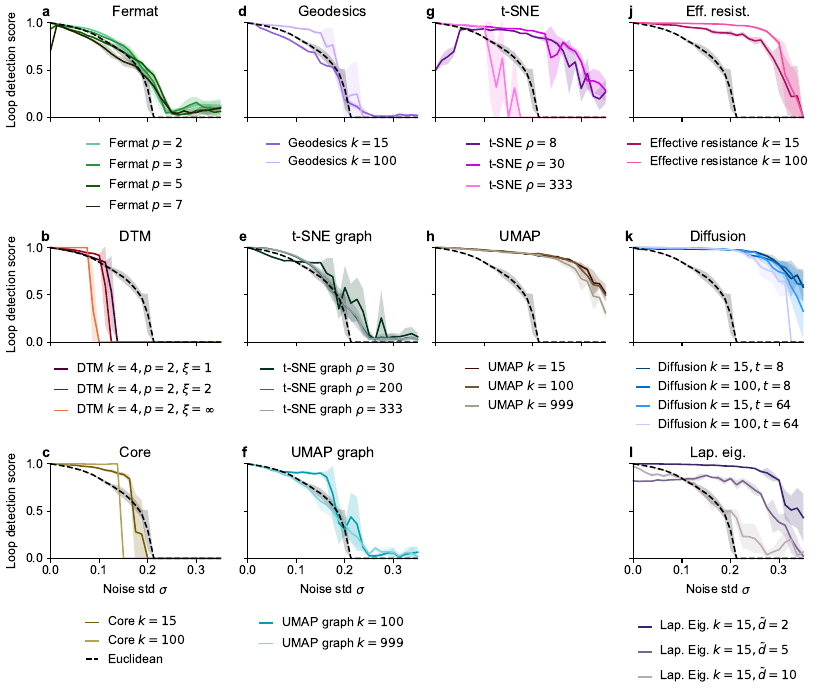}
    \caption{Loop detection score for persistent homology with various distances on a noisy circle in ambient $\mathbb R^{50}$. Extension of Figure~\ref{fig:circle-benchmark}. Spectral and embedding methods performed best. The reason for the dip for the low-perplexity $t$-SNE embedding is depicted in Figure~\ref{fig:circle_tsne_embd}.}
    \label{fig:circle_many_methods_d_50}
\end{figure}

\begin{figure}
    \centering
    \includegraphics[width=\linewidth]{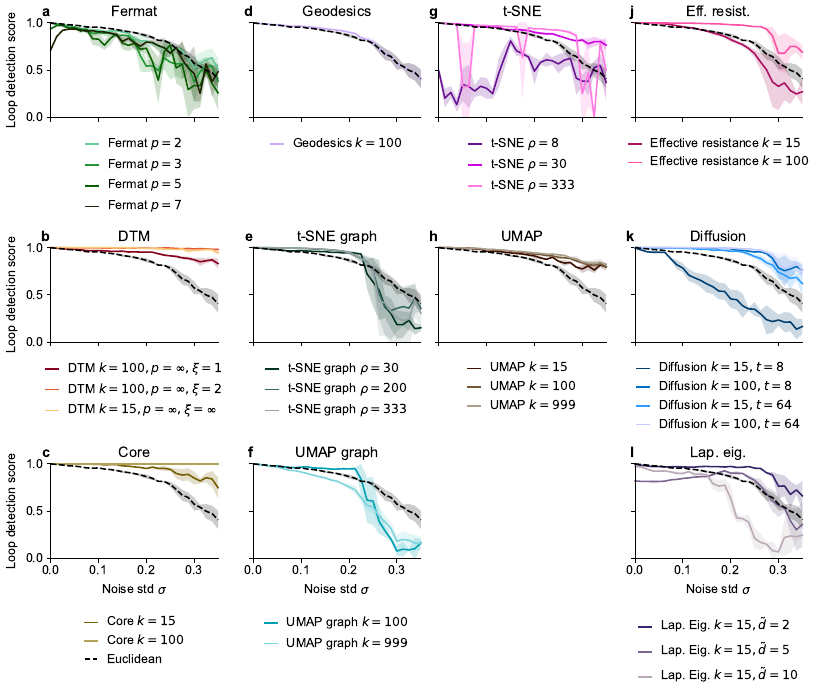}
    \caption{Loop detection score for persistent homology with various distances on a noisy circle in ambient $\mathbb R^{2}$. Our code for finding the geodesics for $k=15$ did not terminate. Nearly all methods performed near perfectly for most noise levels. Note the striking difference to the 50D setting in Figure~\ref{fig:circle_many_methods_d_50}.}
    \label{fig:circle_many_methods_d_2}
\end{figure}

\begin{figure}
    \centering
    \includegraphics[width=\linewidth]{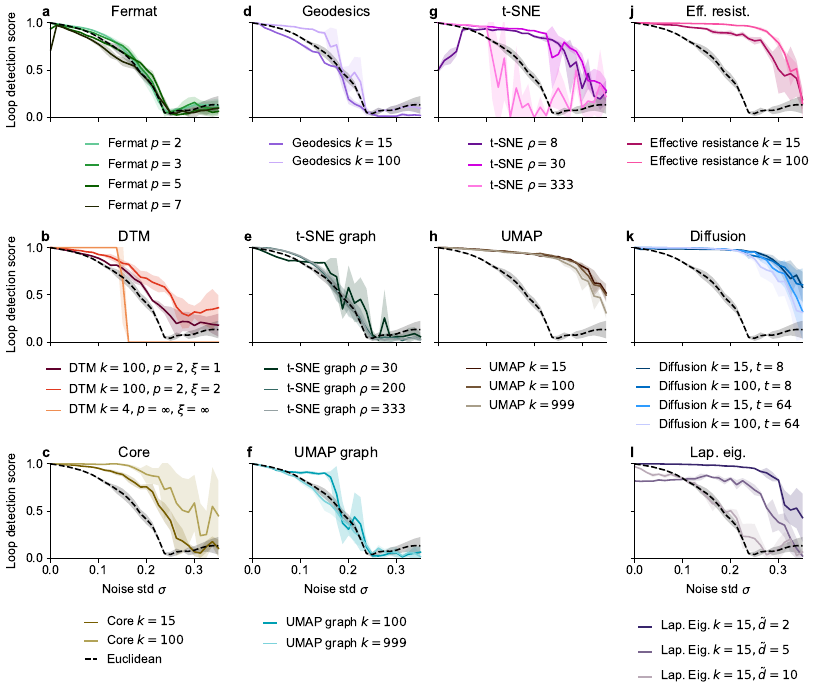}
    \caption{Loop detection score for persistent homology with various distances on a noisy circle in ambient $\mathbb R^{50}$. No thresholding was used for this figure, in contrast to Figure~\ref{fig:circle_many_methods_d_50}. Without thresholding, DTM had better performance, but not much beyond the level of Euclidean distance. Several issues such as high random seed variability for Core $k=100$, $t$-SNE $\rho=333$ and artifactually increasing performance for several methods at very high noise levels can be visible here; this is why we used the thresholding procedure in the main text.}
    \label{fig:circle_many_methods_d_50_no_filter}
\end{figure}

\begin{figure}
    \centering
    \includegraphics[width=\linewidth]{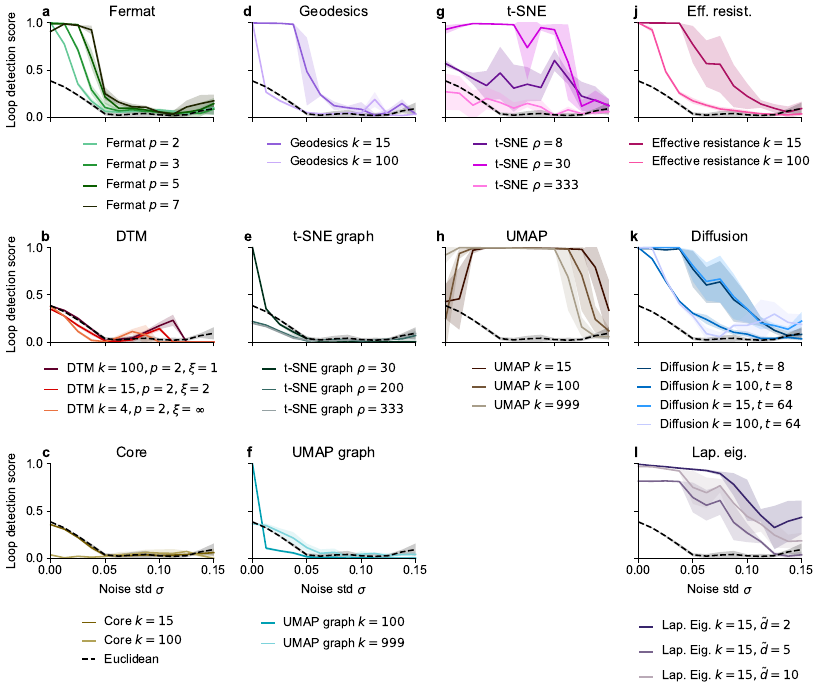}
    \caption{Loop detection score for persistent homology with various distances on the noisy eyeglasses dataset in ambient $\mathbb R^{50}$. Only Fermat distance, geodesics, $t$-SNE and UMAP, and spectral methods outperformed the Euclidean distance, but UMAP struggled in the low noise setting. The reason for the high random seed variability for effective resistance with $k=15$ is depicted in Figure~\ref{fig:eyeglasses_stability}.}
    \label{fig:eyeglasses_many_methods_d_50}
\end{figure}

\begin{figure}
    \centering
    \includegraphics[width=\linewidth]{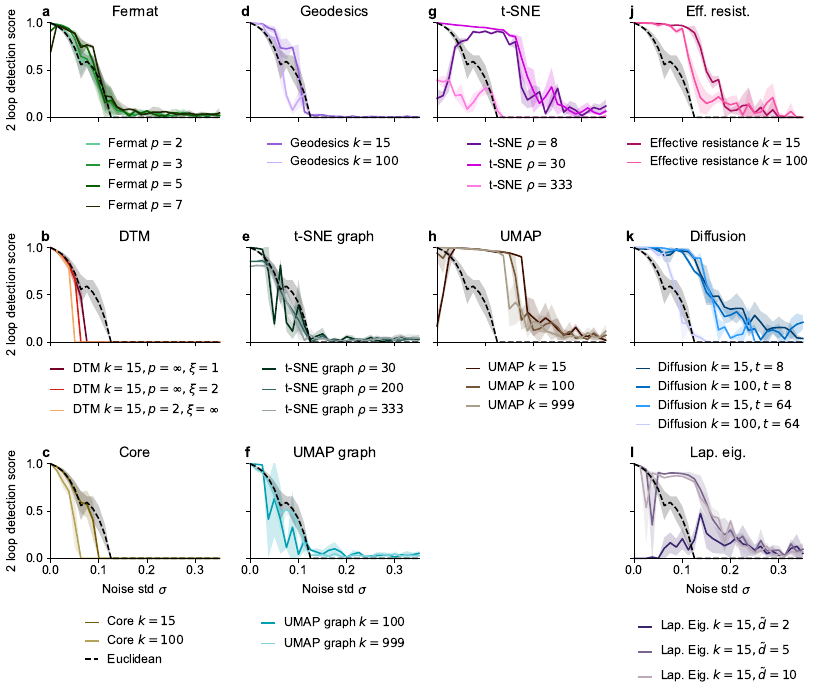}
    \caption{2-loop detection score for persistent homology with various distances on two interlinked circles in ambient $\mathbb R^{50}$. Spectral and embedding methods performed best, but the latter sometimes had issues in the low noise setting.}
    \label{fig:inter_circles_many_methods_d_50}
\end{figure}

\begin{figure}
    \centering
    \includegraphics[width=\linewidth]{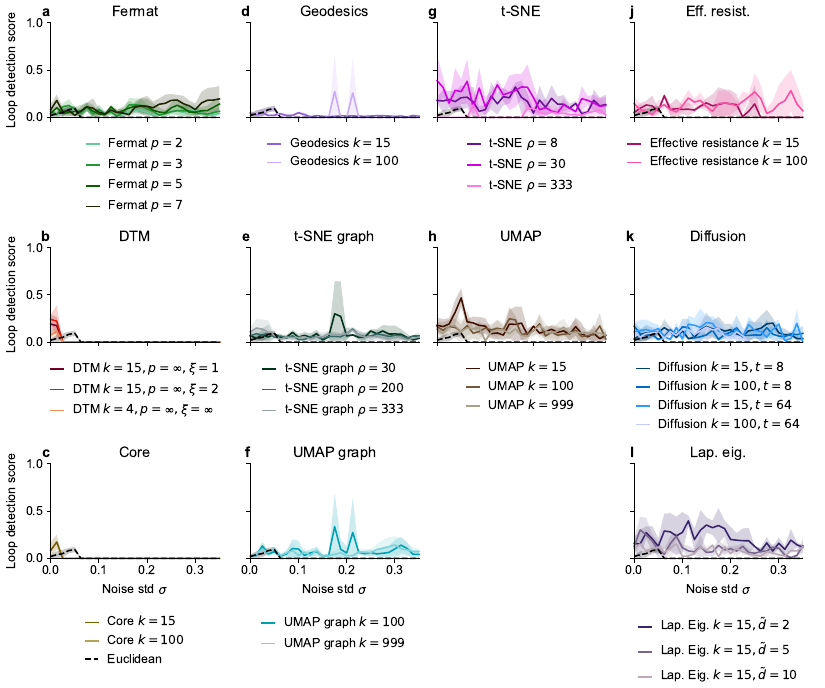}
    \caption{Loop detection score for persistent homology with various distances on a noisy sphere in ambient $\mathbb R^{50}$. Most methods passed this negative control.}
    \label{fig:sphere_many_methods_d_50_1D}
\end{figure}

\begin{figure}
    \centering
    \includegraphics[width=\linewidth]{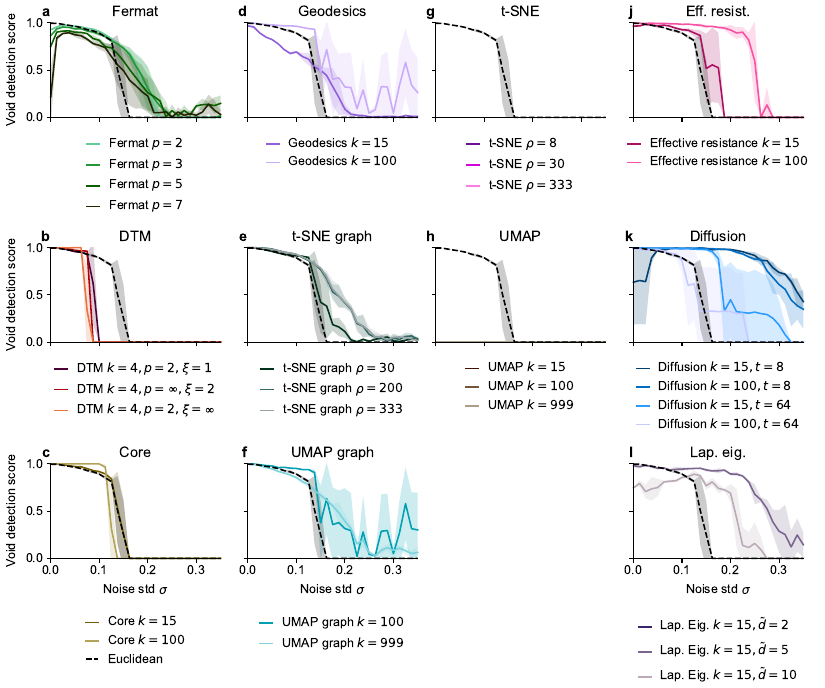}
    \caption{Void detection score for persistent homology with various distances on a noisy sphere in ambient $\mathbb R^{50}$. Methods relying on 2D embeddings did not find the loop for any noise level. Spectral methods performed best.}
    \label{fig:sphere_many_methods_d_50_2D}
\end{figure}

\begin{figure}[tb]
    \centering
    \includegraphics[width=\linewidth]{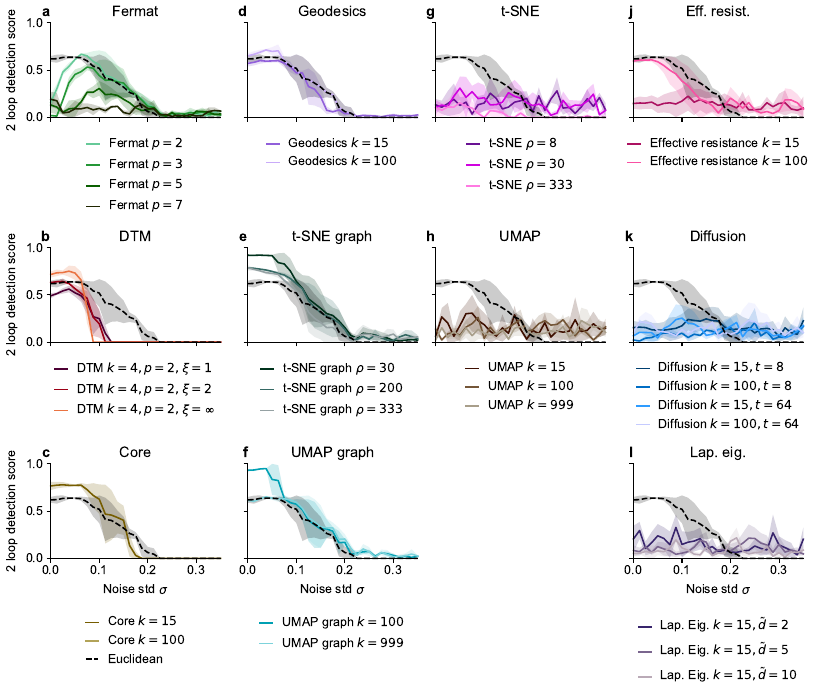}
    \caption{2-loop detection score for persistent homology with various distances on a noisy torus in ambient $\mathbb R^{50}$. All methods struggled here, and only DTM, core, $t$-SNE graph and UMAP graph improved noticeably over the Euclidean distance. On a denser sampled torus effective resistance and diffusion distance outperformed other methods (Figure~\ref{fig:torus_vary_n}). Using fewer diffusion steps improved the performance on the torus (Figure~\ref{fig:torus_diff}).}
    \label{fig:torus_many_methods_d_50_1D}
    \end{figure}

\begin{figure}[tb]
    \centering
    \includegraphics[width=\linewidth]{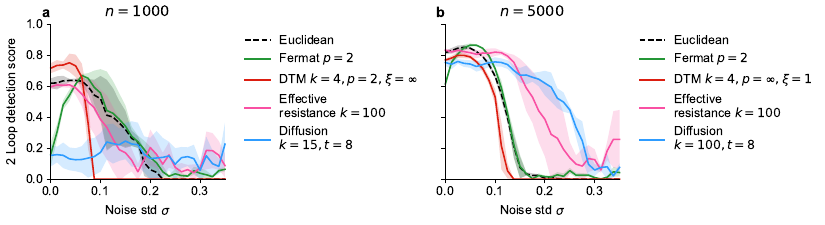}
    \caption{2-loop detection score for persistent homology with various distances on a noisy torus with different sample size $n$. For more points, all methods performed better as the shape of the torus gets sampled more densely. The difference in performance is particularly striking for the spectral methods which outperformed the others for $n=5\,000$ points, but did not for $n=1\,000$.}
    \label{fig:torus_vary_n}
\end{figure}

\begin{figure}[tb]
    \centering
    \includegraphics[width=\linewidth]{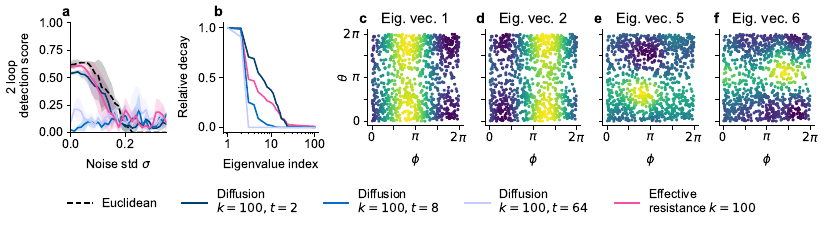}
    \caption{Diffusion distances failed on the torus with $n=1000$ because its eigenvalue decay suppressed the relevant eigenvectors. \textbf{a.} 2-loop detection score for the torus in $d=50$ ambient dimensions. Diffusion distances with $t=2$ diffusion steps were on par with effective resistance and Euclidean distance. \textbf{b.} Decay of eigenvalues in various spectral distances on the noiseless torus. Diffusion distances with $t=8,64$ only had contribution below $0.1$ for the fifth and sixth eigenvectors, while effective resistance and diffusion distance with $t=2$ had substantial contributions from the first ${\sim}10$ eigenvectors. \textbf{c\,--\,f.} Eigenvectors of the symmetric graph Laplacian of a symmetric $100$-nearest-neighbor graph of the noiseless torus. Coordinates are the angles of each point along ($\phi$) and around ($\theta$) the tube of the torus. The loop along the tube is encoded in the first two eigenvectors, the loop around the tube in the fifth and sixth eigenvectors.}
    \label{fig:torus_diff}
\end{figure}

\begin{figure}
    \centering
    \includegraphics[width=\linewidth]{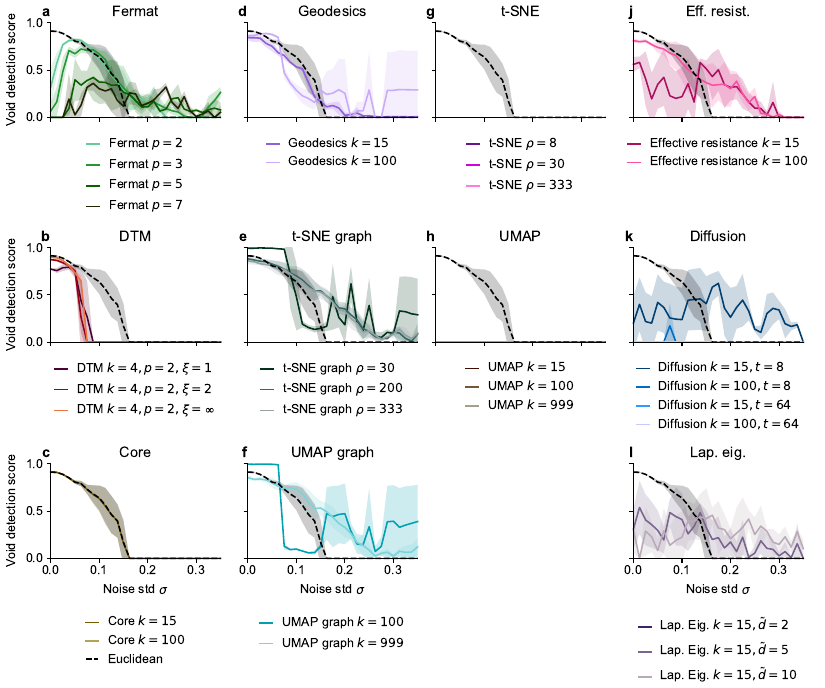}
    \caption{Void detection score for persistent homology with various distances on a noisy torus in ambient $\mathbb R^{50}$. Methods relying on 2D embeddings did not find the void for any noise level. Only $t$-SNE graph and UMAP graph could reliably improve above the Euclidean distance and only for low noise levels. However, they had unstable behavior for higher noise levels, resulting in high uncertainties. We suspect that a higher sampling density would benefit effective resistance and diffusion distance (as we saw for loop detection in Figure~\ref{fig:torus_vary_n}), but the computational complexity of persistent homology makes such experiments difficult.}
    \label{fig:torus_many_methods_d_50_2D}
\end{figure}

\begin{figure}
    \centering
    \includegraphics[width=0.925\linewidth]{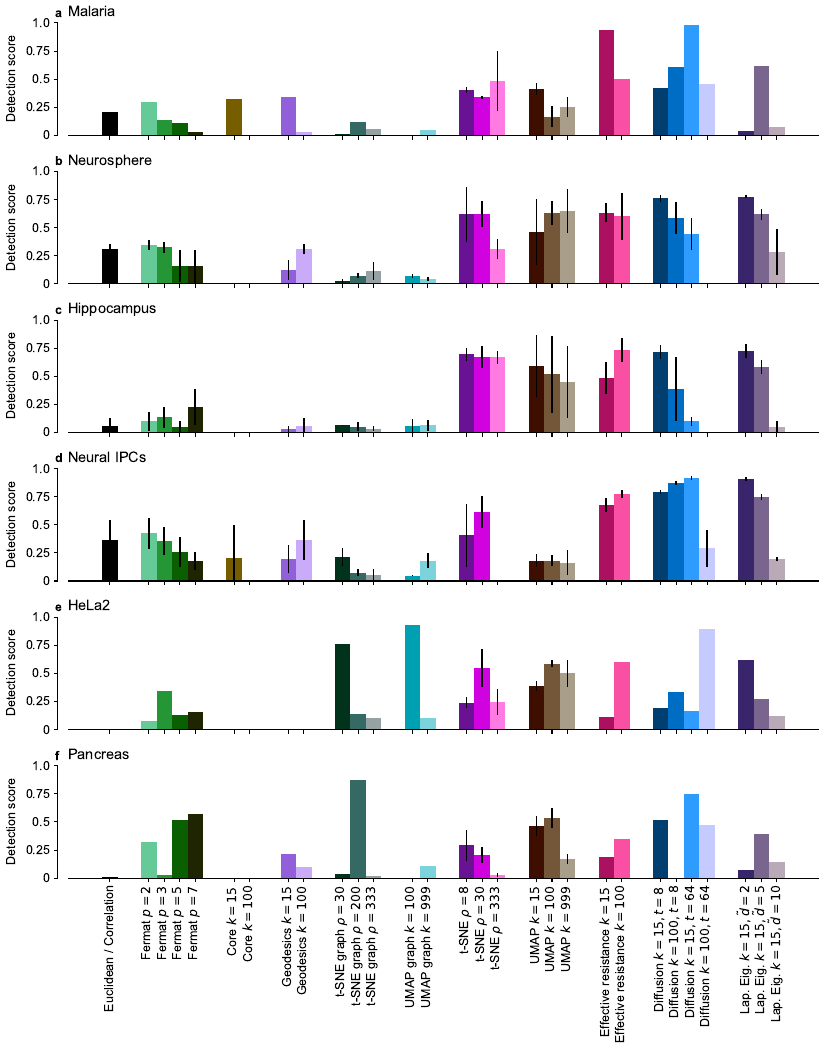}
    \caption{Detection scores for all hyperparameter settings for all six single-cell datasets. We omitted DTM as no setting passed the thresholding on any dataset. The black bar refers to correlation distance on the Malaria dataset and to Euclidean distance on the others. Extension of Figure~\ref{fig:scrnaseq}. $t$-SNE graph and UMAP graph could perform very well, but were very hyperparameter-dependent. Their embedding variants often performed well, but collapsed on some datasets. The spectral methods behaved similarly, but on average performed better.}
    \label{fig:scrnaseq_all}
\end{figure}

\end{document}